\documentclass{article}

\usepackage{graphicx}
\usepackage{subfigure}
\usepackage{booktabs} % for professional tables

\usepackage{hyperref}

 \usepackage[accepted]{icml2023}

\usepackage{amsmath}
\usepackage{amssymb}
\usepackage{mathtools}
\usepackage{amsthm}

\usepackage[capitalize,noabbrev]{cleveref}

\theoremstyle{plain}
\newtheorem{theorem}{Theorem}[section]
\newtheorem{proposition}[theorem]{Proposition}
\newtheorem{lemma}[theorem]{Lemma}

\theoremstyle{definition}
\newtheorem{definition}[theorem]{Definition}

\newtheorem{remark}[theorem]{Remark}

\usepackage[textsize=tiny]{todonotes}

\newcommand{\alglinelabel}{%
  \addtocounter{ALC@line}{-1}% Reduce line counter by 1
  \refstepcounter{ALC@line}% Increment line counter with reference capability
  \label% Regular \label
}

\usepackage[compact]{titlesec}
\titlespacing*{\section}
{0pt}{1.1ex}{1.1ex}
\titlespacing*{\subsection}
{0pt}{1.1ex}{1.1ex}
\titlespacing*{\subsubsection}
{0pt}{1.1ex}{1.1ex}
\usepackage{smile}
\newcommand{\la}{\langle}
\newcommand{\ra}{\rangle}
\newcommand{\qvalue}{Q}
\newcommand{\vvalue}{V}

\newcommand{\reward}{r}

\newcommand{\regret}{\mathbf{Regret}}

\allowdisplaybreaks
\usepackage{colortbl}
\definecolor{LightCyan}{rgb}{0.8, 1, 0.9}

\def \alg {\mathtt{Alg}}
\def \algname {\text{LSVI-UCB++}}

\usepackage{enumitem}

\usepackage[compact]{titlesec}
\usepackage{sidecap}
\titlespacing*{\section}{0pt}{*0.1}{*0.1}
\titlespacing*{\subsection}{0pt}{*0.1}{*0.1}
\titlespacing*{\subsubsection}{0pt}{*0.1}{*0.1}

\allowdisplaybreaks

\icmltitlerunning{Nearly Minimax Optimal Reinforcement Learning 
for Linear Markov Decision Processes}

\begin{document}

\twocolumn[
\icmltitle{Nearly Minimax Optimal Reinforcement Learning \\ for Linear Markov Decision Processes}



\begin{icmlauthorlist}
\icmlauthor{Jiafan He}{UCLA}
\icmlauthor{Heyang Zhao}{UCLA}
\icmlauthor{Dongruo Zhou}{UCLA}
\icmlauthor{Quanquan Gu}{UCLA}
\end{icmlauthorlist}

\icmlaffiliation{UCLA}{Department of Computer Science, University of California, Los Angeles, CA 90095, USA}

\icmlcorrespondingauthor{Quanquan Gu}{qgu@cs.ucla.edu}

\icmlkeywords{Machine Learning, ICML}

\vskip 0.3in
]



\printAffiliationsAndNotice{}  % leave blank if no need to mention equal contribution

\begin{abstract}
We study reinforcement learning (RL) with linear function approximation. For episodic time-inhomogeneous linear Markov decision processes (linear MDPs) whose transition probability can be parameterized as a linear function of a given feature mapping, we propose the first computationally efficient algorithm that achieves the nearly minimax optimal regret $\tilde O(d\sqrt{H^3K})$, where $d$ is the dimension of the feature mapping, $H$ is the planning horizon, and $K$ is the number of episodes. Our algorithm is based on a weighted linear regression scheme with a carefully designed weight, which depends on a new variance estimator that (1) directly estimates the variance of the \emph{optimal} value function, (2) monotonically decreases with respect to the number of episodes to ensure a better estimation accuracy, and (3) uses a rare-switching policy to update the value function estimator to control the complexity of the estimated value function class. Our work provides a complete answer to optimal RL with linear MDPs, and the developed algorithm and theoretical tools may be of independent interest.
\end{abstract}

\section{Introduction} 

How to make reinforcement learning (RL) efficient with large state and action spaces has been a central research problem in the RL community. A widely used approach is \emph{function approximation}, which approximates the value function in RL with a predefined function class for efficient exploration and exploitation. Although the intuition is simple, some basic questions about the function approximation approach still remain open. For instance, what is the optimal sample complexity (or regret) for RL algorithms with function approximation to find the optimal policy? Such optimal sample complexity results have been widely-studied and established for tabular RL methods (e.g., \citealt{azar2017minimax,zhang2019regret,zhang2020almost}), but are still understudied for RL with function approximation. 

Some recent works have studied the optimal regret results for a special class of MDPs called \emph{linear mixture Markov decision processes (linear mixture MDPs)} \citep{jia2020model, ayoub2020model,zhou2021nearly, zhou2022computationally}, which assume that the transition probability of the MDP is a \emph{linear} combination of several base models. More specifically, \citet{zhou2021nearly} proposed the near-optimal algorithm for time-inhomogeneous linear mixture MDPs. \citet{zhou2022computationally} further proposed near-optimal horizon-free algorithm for time-homogeneous linear mixture MDPs under the assumption that the total reward is bounded by $1$. However, the computational efficiency of their algorithms highly depends on the \emph{value-targeted regression} procedure \citep{jia2020model, ayoub2020model}, which relies on an integration or sampling oracle of the individual base model. %needs to do planning over the existing base models. 
Such an integration or sampling oracle exists for some special linear mixture MDPs but can be computationally expensive or even intractable in the general case.

Another line of works studies the \emph{linear Markov decision processes} (linear MDPs) \citep{yang2019sample,jin2020provably}, which assumes that the transition probability and the reward of the environment enjoys a compact low-rank  representation. The most appealing feature of linear MDPs is that they can induce a linear structure of the value function for any policy, which makes sample-efficient RL possible. Meanwhile, the algorithms for linear MDPs directly approximate the value function itself, which is computationally more efficient than the algorithms for linear mixture MDPs. In particular, \citet{yang2019sample} first proposed a near-optimal RL algorithm with the access to a generative model, which can generate any number of samples for any given state-action pairs. Without accessing the generative model, \citet{jin2020provably} proposed an LSVI-UCB algorithm based on the principle of optimism in the face of uncertainty and achieved $\tilde O(\sqrt{d^3H^4K})$ regret, where $d$ is the dimension of a linear MDP, $H$ is the planning horizon and $K$ is the number of episodes. Nevertheless, their algorithm is not optimal since there exists an $O(\sqrt{dH})$ gap between their regret upper bound, and the lower bound $O(d\sqrt{H^3K})$ proved in \citet{zhou2021nearly}. \citet{zanette2020learning} studied a more general MDP class called \emph{low Bellman error class}, which contains linear MDPs as a special case, and they proposed a computationally inefficient algorithm with a near-optimal regret.

Therefore, a natural question arises\footnote{We are aware of a recently published work \citep{hu2022nearly}, which claims to achieve the nearly minimax optimal regret for linear MDPs. However, a closer examination of their proof can find a technical error, which makes their result invalid. 
%In detail, \citet{hu2022nearly} constructed an over-optimistic value function $\dot{\widehat{V}}_{k,h}(s)$ and uses it to estimate the transition variance $\sigma_{k,h}$ in the weighted ridge regression. Unfortunately, the monotonicity of the over-optimistic value function in \citet{hu2022nearly} does not hold as claimed, and the corresponding estimated variance is thus incorrect. 
We will discuss it in more detail and show why our algorithm and proof can get around the issue in Appendix \ref{section:wrong}. Using the techniques proposed by our paper, \citet{hu2022nearly} recently fixed the technical flaw by using the ``rare-switching'' update strategy and also abandoning the over-optimistic estimator. This is acknowledged in the updated arXiv version of \citet{hu2022nearly}.}:

\begin{center}
    Can we design a computationally efficient algorithm that achieves the minimax optimality for linear MDPs?
\end{center}

We give an affirmative answer to the above question in this work. Our contributions are listed as follows.
\begin{itemize}[leftmargin = *]
    \item We propose an algorithm $\algname$ which attains a near-optimal regret $\tilde O(d\sqrt{H^3K})$ when $K$ is large, which matches the lower bound \citep{zhou2021nearly} up to logarithmic factors. To the best of our knowledge, this is the first computationally efficient RL algorithm that is nearly minimax-optimal for linear MDPs.  
    
    \item The first key component of our algorithm is a variance-aware weighted ridge regression scheme, which is firstly introduced to acheive nearly minimax optimal regret for linear mixture MDPs in \citet{zhou2021nearly} and later improved in \citet{zhou2022computationally} to achieve horizon-free regret. Such a component reduces the variance of the estimators in our algorithm, which leads to a $\sqrt{H}$ improvement in the regret over \citet{jin2020provably}. 
    \item To improve the $d$ dependence, inspired by previous works for tabular RL \citep{azar2017minimax}, our algorithm utilizes a new strategy to estimate the variance of the estimated value function. Unlike the previous approach for linear mixture MDPs \citep{zhou2021nearly}, our new estimator directly estimates the variance of the \emph{true} value function and computes the difference between the variances of the true value function and the estimated one. Such a strategy allows the variance estimator to focus on a simpler function class that only includes the true value function, and therefore gives a tighter confidence set than that in \citet{jin2020provably}.
    \item To obtain a uniform variance upper bound, we construct our value function estimator as a monotonically decreasing estimator with a ``rare-switching'' update strategy, which makes the estimated value function decrease with respect to the episodes and being updated rarely. %Intuitively speaking, the monotonic decreasing estimator allows a more accurate estimator of the value function, and the ``rare-switching'' update strategy greatly reduces the size of the estimation function class the algorithm will face during the whole algorithm process. 
    Together with our new variance estimator, we can remove the additional $\sqrt{d}$ dependency from the previous regret, which makes our algorithm nearly minimax optimal. Notably, our algorithm only needs to update the policy $O(\log K)$ times instead of $K$ times, and therefore enjoys a low-switching cost.
\end{itemize}

For the ease of comparison, we summarize the regret bounds of our algorithm and previous algorithms for linear MDPs in Table \ref{table:11}. 

Recently, an independent concurrent work \citep{agarwal2022vo} proposed a different algorithm that can also achieve near-optimal regret for linear MDPs. Their algorithm follows the algorithm design in \citet{hu2022nearly}, which introduces an additional over-optimistic value function to construct a monotonic variance estimator, and a non-Markovian policy to fix the technical flaw in \citet{hu2022nearly}. In contrast, our algorithm takes a neat approach and constructs the monotonic variance estimator with a simple ``rare-switching'' update strategy, which enjoys low-switching cost. \citet{agarwal2022vo} also studied RL with nonlinear function approximation, which is beyond the scope of this work.

%\CC{add notation paragraph}

\noindent\textbf{Notation} In this work, we use lowercase letters to denote scalars and use lower and uppercase boldface letters to denote vectors and matrices respectively.  For a vector $\xb\in \RR^d$ and matrix $\bSigma\in \RR^{d\times d}$, we denote by $\|\xb\|_2$ the Euclidean norm and $\|\xb\|_{\bSigma}=\sqrt{\xb^\top\bSigma\xb}$. For two sequences $\{a_n\}$ and $\{b_n\}$, we write $a_n=O(b_n)$ if there exists an absolute constant $C$ such that $a_n\leq Cb_n$, and we write $a_n=\Omega(b_n)$ if there exists an absolute constant $C$ such that $a_n\geq Cb_n$. We use $\tilde O(\cdot)$ and $\tilde \Omega(\cdot)$ to further hide the logarithmic factors. For any $a \leq b \in \RR$, $x \in \RR$, let $[x]_{[a,b]}$ denote the truncate function $a\cdot \ind(x \leq a) + x \cdot \ind (a \leq x \leq b) + b \cdot \ind (b \leq x)$, where $\ind(\cdot)$ is the indicator function. For a positive integer $n$, we use $[n]=\{1,2,..,n\}$ to denote the set of integers from $1$ to $n$.

\newcolumntype{g}{>{\columncolor{LightCyan}}c}
\begin{table}[t!]
\centering
\begin{tabular}{cggg}
\toprule
\rowcolor{white}
Model & Algorithm & Regret
 \\
%\midrule
%\rowcolor{white}
%  &UCRL-VTR &    \\
%\rowcolor{white}
% &\small{\citep{ayoub2020model}} & \multirow{-2}{*}{$\tilde  O\Big(\sqrt{d^2H^4K}\Big)$} \\
% \rowcolor{white}
% &UCRL-VTR+  &     \\
% \rowcolor{white}
% &\small{\citep{zhou2021nearly}} & \multirow{-2}{*}{$\tilde  O\Big(\sqrt{d^2H^3K+dH^4K}\Big)$}\\
% \rowcolor{white}
%  &VARLin &    \\
%\rowcolor{white}
% &\small{\citep{zhang2021improved}}\footnotemark & \multirow{-2}{*}{$\tilde  O\Big(\sqrt{d^9H^3K}\Big)$} \\
%  \rowcolor{white}
%\multirow{-2}{*}{Linear mixture MDP} &VARLin2  &     \\
% \rowcolor{white}
%&\small{\citep{kim2021improved}}\footnotemark & \multirow{-2}{*}{$\tilde  O\Big(\sqrt{d^2H^3K}\Big)$}\\

%\rowcolor{white}
% & HF-UCRL-VTR+  &    \\
% \rowcolor{white}
%&\small{\citep{zhou2022computationally}} & \multirow{-2}{*}{$\tilde  O\Big(\sqrt{d^2H^3K}\Big)$}\\
%\rowcolor{white}
 \midrule

 %\rowcolor{white}

 \rowcolor{white}
&LSVI-UCB &  \\
\rowcolor{white}
 &\citep{jin2020provably} & \multirow{-2}{*}{$\tilde  O\big(\sqrt{d^3H^4K}\big)$}  \\

% \rowcolor{white}
 %  &{ LSVI-UCB+ } &    \\
 %  \rowcolor{white}
% &\small{\citep{hu2022nearly}} &\multirow{-2}{*}{ $\tilde  O\Big(d\sqrt{H^3K}\Big)$} \\
 
 &$\algname$ &    \\

 \multirow{-4}{*}{Linear MDP}&(\textbf{Our work}) &\multirow{-2}{*}{ $\tilde  O\big(d\sqrt{H^3K}\big)$} \\

 \midrule
 \rowcolor{white}
Lower bound & \citet{zhou2021nearly} & $\Omega\big(d\sqrt{H^3K}\big)$\\
\bottomrule
\end{tabular}
\caption{Comparison of RL with linear function approximation in terms of regret guarantee. }\label{table:11}
%\footnotesize{
%\noindent 1. 2. \citet{zhang2021improved, kim2021improved} and \citet{zhou2022computationally} studied the time-homogeneous MDP, where the transition probability $\PP_h(\cdot|\cdot,\cdot)$ is the same across different stages $h\in[H]$, and assume the total reward in each episode is up to $1$. In this table, we translate their results to the time-inhomogeneous setting.
%}
\end{table}

\section{Related Work}
\noindent\textbf{Near-optimal tabular reinforcement learning } 
There is a voluminous amount of works developing nearly minimax optimal algorithms for tabular MDPs under different settings \citep{azar2017minimax,  zanette2019tighter, zhang2019regret, simchowitz2019non, zhang2020almost, zhang2021reinforcement, he2021nearly}. A key idea behind these works is to exploit the $O(H^2)$ total variance of the value functions for each episode \citep{azar2017minimax,jin2018q}. %Adaptive to the upper bound on the sum of variance, 
\citet{azar2017minimax} first proposed this idea to design Bernstein-type bonuses in tabular MDPs and provided an $\tilde{O}(\sqrt{H^2SAK})$ regret upper bound, which matches the lower bound for the tabular setting. In their analysis, \citet{azar2017minimax} also introduced a new value function decomposition scheme which mainly focuses on the variances of the optimal value function rather than the estimated value function. \citet{zhang2021reinforcement} further improved the dependence on $H$ for the constant terms and achieved a nearly minimax optimal horizon-free regret (nearly independent of $H$) under the assumption that the total reward is bounded by $1$. Our algorithm extends the idea of Bernstein-type bonuses and value function decomposition in \citet{azar2017minimax} to RL with linear function approximation.

\noindent\textbf{Reinforcement Learning with Linear Function Approximation. } There exists a large body of literature on RL with linear function approximation \citep{jiang2017contextual, dann2018oracle, yang2019sample, jin2020provably, wang2020optimism, du2019good, sun2019model, zanette2020frequentist, yang2020reinforcement, modi2020sample, ayoub2020model, zhou2021nearly, he2021logarithmic, zhou2022computationally}. All these works assume certain linear structures of the underlying MDP. The most related work to ours is initiated by \citet{yang2019sample}, which assumes that the reward function and the transition probability are linear in the feature mapping $\phi(s, a)$ for each state-action pair $(s, a)$. \citet{jin2020provably} further considered \emph{Linear MDPs} and proposed LSVI-UCB which achieves an $\tilde{O}( \sqrt{d^3H^4K})$ regret bound. %There is also a \CC{concurrent work} by 
\citet{zanette2020frequentist} proposed a Thompson sampling based algorithm for linear MDPs, which attains a regret upper bound of order $\tilde{O}(\sqrt{d^4H^5K})$.
Another popular MDP model for RL with linear function approximation is \emph{linear mixture Markov Decision Processes} \citep{modi2020sample, yang2020reinforcement, jia2020model,ayoub2020model}, or \emph{Linear Kernel MDPs} \citep{zhou2021provably}, where the transition probability is a linear combination of several base models. 
%$\bphi$ on the triplet $(s, a, s')$, i.e., $\PP(s' |s, a) = \la \btheta, \bphi(s, a, s') \ra$ for some $\btheta \in \RR^d$. 
For linear mixture MDPs, \citet{zhou2021nearly} is the first to achieve a nearly minimax optimal regret bound. There are also works achieving horizon-free regret bounds for time-homogeneous linear mixture MDPs \citep{zhang2021improved, zhou2022computationally}. Compared with \citet{zhou2021nearly}, our algorithm is the first to achieve the near-optimality for linear MDPs. 

% \paragraph*{Bernstein-type Bonuses for RL. }
% There is a voluminous amount of work proposing algorithms with nearly minimax optimal regret for MDPs under different settings \citep{azar2017minimax,  zanette2019tighter, zhang2019regret, simchowitz2019non, zhang2020almost, he2021nearly,zhou2021nearly}. A general idea underlying these works is to exploit the $O(H^2)$ total variance of the value functions for each episode \citep{jin2018q}. Adaptive to the upper bound on the sum of variance, \cite{azar2017minimax} first proposed this idea to design Bernstein-type bonuses in tabular MDPs and provided an $\tilde{O}(\sqrt{H^2SAK})$ regret upper bound, which matches the lower bound for the tabular setting. \cite{zhou2021nearly} proposed a Bernstein-type inequality for vector-valued martingales, making it possible to adopt Berstein-type bonuses for RL with linear function approximation. According to their novel concentration inequality, it was shown that their proposed algorithm, UCRL-VTR+ achieved minimax optimal regret bound for linear mixture MDPs. 

\section{Preliminaries}
In this work, we consider the episodic Markov Decision Processes (MDP), where the MDP can be denoted by a tuple of $M(\cS, \cA, H, \{\reward_h\}_{h=1}^H, \{\PP_h\}_{h=1}^H)$. Here, $\cS$ is the state space, $\cA$ is the finite action space, $H$ is the length of each episode (i.e., planning horizon), $\reward_h: \cS \times \cA \rightarrow [0,1]$\footnote{In this work, we study the deterministic and known reward functions for simplicity, and it is not difficult to generalize our results to stochastic and unknown linear reward functions in \citep{jin2019provably}, where $\reward_h(s,a)=\big\la \bphi(s,a),\bmu_h\big\ra.$} is the reward function at stage $h$ and $\PP_h(s'|s,a) $ is the transition probability function at stage $h$ which denotes the probability for state $s$ to transfer to next state $s'$ with current action $a$. Following \citet{jin2020provably}, %\CC{we cite Jin 2020 and Jin 2019}, 
we assume that $\cS$ is a measurable space with possibly infinite number of states and $\cA$ is a finite set. A policy $\pi: \cS \times [H] \rightarrow \cA$ is a function that maps a state $s$ and the stage number $h$ to an action $a$. 
For any stage $h\in [H]$ and policy $\pi$, we define the value function $\vvalue_h^{\pi}(s)$ and the action-value function $\qvalue_h^{\pi}(s,a)$ as follows:
\begin{align}
\qvalue^{\pi}_h(s,a) &=\reward_h(s,a)\notag\\
&+ \EE\bigg[\sum_{h'=h+1}^H \reward_{h'}\big(s_{h'},a_{h'} \big)\big| s_h=s,a_h=a\bigg],\notag\\
\vvalue_h^{\pi}(s) &= \qvalue_h^{\pi}\big(s, \pi(s,h)\big),\notag
\end{align}
where $s_{h'+1}\sim \PP_h(\cdot|s_{h'},a_{h'})$ denotes the state at stage $h'+1$ and $a_{h'}=\pi(s_{h'},h')$ denotes the action at stage $h'$
. Furthermore, 
we can define the optimal value function $V_h^*$ and the optimal action-value function $\qvalue_h^*$ as $V_h^*(s) = \max_{\pi}\vvalue_h^{\pi}(s)$ and $\qvalue_h^*(s,a) = \max_{\pi}\qvalue_h^{\pi}(s,a)$. By this definition, the value function $\vvalue_h^{\pi}(s)$ and action-value function $\qvalue_h^{\pi}(s,a)$ are bounded in $[0,H]$.
For any function $\vvalue: \cS \rightarrow \RR$, we denote $[\PP_h \vvalue](s,a)=\EE_{s' \sim \PP_h(\cdot|s,a)}\vvalue(s')$ and $[\VV_h \vvalue](s,a)=[\PP_h V^2](s,a)-\big([\PP_h V](s,a)\big)^2$ for simplicity. Thus, for every stage $h\in[H]$ and policy $\pi$, we have the following Bellman equation for value functions $Q_{h}^{\pi}(s,a)$ and $V_{h}^{\pi}(s)$, as well as the Bellman optimality equation for optimal value functions $Q_{h}^{*}(s,a)$ and $V_{h}^{*}(s)$:
\begin{align}
    \qvalue_h^{\pi}(s,a) &= \reward_h(s,a) + [\PP_h\vvalue_{h+1}^{\pi}](s,a),\notag\\
    \qvalue_h^{*}(s,a) &= \reward_h(s,a) + [\PP_h\vvalue_{h+1}^{*}](s,a),\notag
\end{align}
where $\vvalue^{\pi}_{H+1}(s)=\vvalue^{*}_{H+1}(s)=0$. At the beginning of each episode $k\in[K]$, the agent selects a policy $\pi_k$ to be followed in this episode. At each stage $h\in[H]$, the agent first observes the current state $s_h^k$, chooses an action following the policy $\pi_k$ and then observes the next state with $s_{h+1}^k \sim \PP_h(\cdot|s_h^k,a_h^k)$. Based on these definitions, we further define the regret in the first $K$ episodes as follows:
\begin{definition}\label{def:Regret} 
For any algorithm $\alg$, we define its regret on learning an MDP $M(\cS, \cA, H, \reward, \PP)$ in the first $K$ episodes as the sum of the sub-optimality gaps for episode $k = 1,\ldots, K$, i.e.,
%\CC{suboptimality or sub-optimality}
\begin{align}
    \text{Regret}(K) = \sum_{k=1}^K \vvalue_1^*(s_1^k) - \vvalue_1^{\pi_k}(s_1^k),\notag
\end{align}
where $\pi_k$ is the agent's policy in the $k$-th episode.
\end{definition}

\noindent\textbf{Linear Markov Decision Process}
In this work, we focus on the linear Markov decision Process \citep{jin2020provably,yang2019sample}, which is formally defined as follows:
\begin{definition}\label{assumption:linear-MDP}
An MDP $\cM(\cS, \cA, H, \{\reward_h\}_{h=1}^H, \{\PP_h\}_{h=1}^H)$ is a linear MDP if for any stage $h\in[H]$, there exists 
%a known vector $\bmu_h$, 
an unknown measure $\btheta_h(\cdot): \cS\rightarrow \RR^d$ and a known feature mapping $\bphi: \cS \times \cA \rightarrow \RR^d$, such that for each state-action pair $(s,a) \in \cS \times \cA$ and state $s' \in \cS$, we have
%such that
\begin{align}
    \PP_h(s'|s,a) = \big\la \bphi(s,a), \btheta_h(s')\big\ra.
    %,\reward_h(s,a)=\big\la \bphi(s,a),\bmu_h\big\ra.\notag
\end{align}
\end{definition}
For simplicity, we assume that the norms of 
%$\bmu_h$, 
$\btheta_n(\cdot)$ and $\bphi(\cdot,\cdot)$ are upper bounded by $\|\bphi(s,a)\|_2 \leq 1$
%$\|\bmu_h\|_2 \leq \sqrt{d}$ 
and $\big\|\btheta_h(\cS)\big\|_2\leq \sqrt{d}$. For linear MDPs, we have the following property:
\begin{proposition}[Proposition 2.3, \citealt{jin2020provably}]\label{prop:linearq}
For any policy $\pi$, there exist weights $\{\wb_h^\pi\}_{h=1}^H$ such that for any state-action pair $(s,a)\in \cS \times \cA$ and stage $h\in[H]$, we have $[\PP V_{h+1}^\pi](s,a) = \la \bphi(s,a), \wb_h^\pi\ra$. 
\end{proposition}
%\CC{citealt}

\section{The Proposed Algorithm}

%\CC{Line 4 and 5 can be moved into the if statement?}
%\CC{Require more input parameters, how to initialize $k_{last}$?}
\begin{algorithm}[t]
    \caption{$\algname$}
    \begin{algorithmic}[1]\label{algorithm1}
   \REQUIRE Regularization parameter $\lambda>0$, confidence radius $\beta, \bar\beta,\tilde{\beta}$
   \STATE Initialize $k_\text{last}=0$ and for each stage $h\in[H]$ set $\bSigma_{0,h},\bSigma_{1,h}\leftarrow \lambda \Ib$ 
   \STATE For each stage $h\in[H]$ and state-action $(s,a)\in \cS \times \cA$, set $ Q_{0,h}(s,a)\leftarrow H, \check{Q}_{0,h}(s,a)\leftarrow 0$
\FOR{episodes $k=1,\ldots,K$}
\STATE Received the initial state $s_1^k$. 
    \FOR{stage $h=H,\ldots,1$}
        
        \STATE 
    $\hat{\wb}_{k,h}=\bSigma_{k,h}^{-1}\sum_{i=1}^{k-1}\bar\sigma_{i,h}^{-2}\bphi(s_h^i,a_h^i)\vvalue_{k,h+1}(s_{h+1}^{i})$ \alglinelabel{algorithm:line0}
        \STATE $\check{\wb}_{k,h}=\bSigma_{k,h}^{-1}\sum_{i=1}^{k-1}\bar\sigma_{i,h}^{-2}\bphi(s_h^i,a_h^i)\check{\vvalue}_{k,h+1}(s_{h+1}^{i})$\alglinelabel{algorithm:line1}
        \IF {there exists a stage $h'\in[H]$ such that $\det(\bSigma_{k,h'})\ge 2\det(\bSigma_{k_{\text{last}},h'})$}\alglinelabel{algorithm:det}
        \STATE $\qvalue_{k,h}(s,a)=\min\Big\{\reward_h(s,a)+\hat{\wb}_{k,h}^{\top}\bphi(s,a)+\beta\sqrt{\bphi(s,a)^{\top}\bSigma_{k,h}^{-1}\bphi(s,a)} ,{\qvalue}_{k-1,h}(s,a),H\Big\}$ \alglinelabel{algorithm1:min}
         \STATE $\check{\qvalue}_{k,h}(s,a)=\max\Big\{\reward_h(s,a)+\check{\wb}_{k,h}^{\top}\bphi(s,a)-\bar{\beta}\sqrt{\bphi(s,a)^{\top}\bSigma_{k,h}^{-1}\bphi(s,a)} ,\check{\qvalue}_{k-1,h}(s,a),0\Big\}$
        \STATE Set the last updating episode $k_{\text{last}}=k$
        \ELSE 
        \STATE $\qvalue_{k,h}(s,a)=\qvalue_{k-1,h}(s,a)$
        \STATE $\check{\qvalue}_{k,h}(s,a)=\check{\qvalue}_{k-1,h}(s,a)$
        \ENDIF  
        \STATE $\vvalue_{k,h}(s)=\max_{a}\qvalue_{k,h}(s,a)$
        \STATE $\check{\vvalue}_{k,h}(s)=\max_{a}\check{\qvalue}_{k,h}(s,a)$
        \alglinelabel{algorithm:line3}
    \ENDFOR
    \FOR{stage $h=1,\ldots,H$}
    \STATE Take action $a_h^k\leftarrow \argmax_{a} \qvalue_{k,h}(s_h^k,a)$
    \STATE Set the estimated variance $\sigma_{k,h}$ as in \eqref{eq:variance} 
     \STATE $\bar\sigma_{k,h}\leftarrow \max\big\{\sigma_{k,h}, H,2d^3H^2\|\bphi(s_h^k,a_h^k)\|_{\bSigma_{k,h}^{-1}}^{1/2}\big\}$\alglinelabel{algorithm:def-variance}
    \STATE $\bSigma_{k+1,h}=\bSigma_{k,h}+\bar\sigma_{k,h}^{-2}\bphi(s_h^k,a_h^k)\bphi(s_h^k,a_h^k)^{\top}$ 
    \STATE Receive next state $s_{h+1}^k$\alglinelabel{algorithm:line4}
\alglinelabel{algorithm:line2}
    \ENDFOR
\ENDFOR
    \end{algorithmic}
\end{algorithm}
%\CC{change the bonus in the algorithm to be norm}
In this section, we propose a new algorithm $\algname$ to learn the linear MDPs (See Definition \ref{assumption:linear-MDP}). The main algorithm is illustrated in Algorithm \ref{algorithm1}. In the sequel, we introduce the key ideas of the proposed algorithm one by one.

\subsection{Weighted Ridge Regression}
The basic framework of our algorithm follows the LSVI-UCB algorithm proposed by \citet{jin2020provably}. Based on Proposition \ref{prop:linearq} that the expected value function $[\PP_h V_{h+1}^\pi](s,a) = \la \bphi(s,a), \wb_h^\pi\ra$, Algorithm~\ref{algorithm1} reduces the learning of the optimal action-value function into a series of linear regression problems. In order to have a good estimation for the vector $\wb_h^\pi$ and achieve the minimax-optimal regret result, Algorithm \ref{algorithm1} adapts the weighted ridge regression method \citep{henderson1975best}, which was used in heteroscedastic linear bandits \citep{lattimore2015linear,kirschner2018information} and more recently RL with linear function approximation \citep{zhou2021nearly} for linear mixture MDPs.
In detail, for each stage $h\in[H]$ and episode $k\in[K]$,
we construct the estimator $\hat{\wb}_{k,h}$ by solving the following weighted ridge regression
\begin{align}
    \hat{\wb}_{k,h}&\leftarrow \argmin_{\wb\in \RR^d}\lambda\|\wb\|_2^2\notag\\
    &\qquad +\textstyle{\sum_{i=1}^{k-1}}\bar{\sigma}_{i,h}^{-2}\big(\wb^{\top}\bphi(s_h^i,a_h^i)-V_{k,h+1}(s_{h+1}^i)\big)^2.\notag
\end{align} 
Here, we take the inverse of the estimated variances $\sigma_{k,h}^2$ as the weights for the regression problem and set $\sigma_{k,h}$ as
 \begin{align*}
     \bar\sigma_{k,h}= \max\big\{\sigma_{k,h}, H,2d^3H^2\|\bphi(s_h^k,a_h^k)\|_{\bSigma_{k,h}^{-1}}^{1/2}\big\}
 \end{align*}
 in Line \ref{algorithm:def-variance} of Algorithm \ref{algorithm1}, which depends on the uncertainty term $\|\bphi(s_h^k,a_h^k)\|_{\bSigma_{k,h}^{-1}}$. Note that the uncertainty-dependent weight has also been used in \citet{he2022nearly} to defend  the adversarial corruption in the linear bandits problem. The reason why we want to use an uncertainty-dependent weight can be explained by the following lemma. 
    \begin{lemma}[Theorem 4.3, \citealt{zhou2022computationally}]\label{lemma:concentration_variance} 
        Let $\{\cG_k\}_{k=1}^\infty$ be a filtration, and $\{\xb_k,\eta_k\}_{k\ge 1}$ be a stochastic process such that
        $\xb_k \in \RR^d$ is $\cG_k$-measurable and $\eta_k \in \RR$ is $\cG_{k+1}$-measurable.
        Let $L,\sigma>0$, $\bmu^*\in \RR^d$. 
        For $k\ge 1$, 
        let $y_k = \la \bmu^*, \xb_k\ra + \eta_k$ and
        suppose that $\eta_k, \xb_k$ also satisfy 
        \begin{align}
            \EE[\eta_k|\cG_k] = 0,\ \EE [\eta_k^2|\cG_k] \leq \sigma^2,\  |\eta_k| \leq R,\,\|\xb_k\|_2 \leq L.\notag
        \end{align}
        For $k\ge 1$, let $\Zb_k = \lambda\Ib + \sum_{i=1}^{k} \xb_i\xb_i^\top$, $\bbb_k = \sum_{i=1}^{k}y_i\xb_i$, $\bmu_k = \Zb_k^{-1}\bbb_k$, and $\beta_k=\tilde O\big(\sigma\sqrt{d}+\max_{1 \leq i \leq k} |\eta_i|\min\{1, \|\xb_i\|_{\Zb_{i-1}^{-1}}\}\big)$.
        %\begin{small}
        %\begin{align}
         %   \beta_k &= 12\sqrt{\sigma^2d\log(1+kL^2/(d\lambda))\log(32(\log(R/\epsilon)+1)k^2/\delta)} \notag \\
          %  & + 24\log(32(\log(R/\epsilon)+1)k^2/\delta)\max_{1 \leq i \leq k} |\eta_i|\min\{1, \|\xb_i\|_{\Zb_{i-1}^{-1}}\} \notag\\
          %  &+ 6\log(32(\log(R/\epsilon)+1)k^2/\delta)\epsilon.\notag
        %\end{align}
        %\end{small}
        Then, for any $0 <\delta<1$,  with probability at least $1-\delta$, for all $k\in[K]$, we have 
        \begin{align}
            \big\|\textstyle{\sum}_{i=1}^{k} \xb_i \eta_i\big\|_{\Zb_k^{-1}} \leq \beta_k,\ \|\bmu_k - \bmu^*\|_{\Zb_k} \leq \beta_k + \sqrt{\lambda}\|\bmu^*\|_2. \notag
        \end{align}
    \end{lemma}
By Lemma \ref{lemma:concentration_variance}, one can easily verify that $|\la \hat\wb_{k,h}, \bphi(s,a)\ra - \PP_hV_{k,h+1}(s,a)| = O\Big(\beta \big\|\bSigma_{k,h}^{-1/2} \bphi(s,a)\big\|_2\Big)$, %\CC{change it to norm}, 
where $\beta=\tilde O(\sqrt{d})$. Such an $\tilde O(\sqrt{d})$ dependence is similar to that in \citet{zhou2022computationally}, which allows our algorithm to use a tighter confidence set than \citet{jin2020provably}. Therefore, we can construct the optimistic value function $Q_{k,h}$ with the linear function and an additional exploration bonus term (Line 7 in Algorithm \ref{algorithm1}), i.e.,
 \begin{align*}
     \qvalue_{k,h}(s,a) \approx \reward_h(s,a)+\hat{\wb}_{k,h}^{\top}\bphi(s,a)+\beta\big\|\bSigma_{k,h}^{-1/2} \bphi(s,a)\big\|_2.
 \end{align*}
 With the help of the exploration bonus, we can show that the optimistic value function $Q_{k,h}(s,a)$ is an upper bound of the  optimal value function $Q_{h}^*(s,a)$ and the summation of the sub-optimality gaps can be upper bounded by the summation
of exploration bonus $\sum_{h=1}^H \sum_{k=1}^K\beta\sqrt{\bphi(s,a)^{\top}\bSigma_{k,h}^{-1}\bphi(s,a)}$. 
By adapting the weighted ridge regression, \citet{zhou2022computationally} proposed HF-UCRL-VTR+, which is able to achieve a nearly minimax optimal regret for linear mixture MDPs. However, their algorithm and approach cannot be directly applied to linear MDPs, and we need to construct a pessimistic value function $\check{V}_{k,h}$ for the optimal value function $Q_{h}^*(s,a)$ to estimate the gap between $V_{k,h}(s)$ and $V_{h}^*(s)$, where we have $V_{k,h}(s)-V_{h}^*(s)\leq V_{k,h}(s)-\check{V}_{k,h}(s)$.
Similar to the optimistic value function, we construct the vector 
$\check{\wb}_{k,h}$ by solving the following weighted ridge regression,
\begin{align}
    \check{\wb}_{k,h}&\leftarrow \argmin_{\wb\in \RR^d}\lambda\|\wb\|_2^2\notag\\
    &\qquad +\textstyle{\sum_{i=1}^{k-1}}\bar{\sigma}_{i,h}^{-2}\big(\wb^{\top}\bphi(s_h^i,a_h^i)-\check{V}_{k,h+1}(s_{h+1}^i)\big)^2,\notag
\end{align} 
and compute the pessimistic value function $\check{Q}_{k,h}$ as:
 \begin{align*}
     \check{\qvalue}_{k,h}(s,a) &\approx \reward_h(s,a)\notag\\
     &\qquad +\check{\wb}_{k,h}^{\top}\bphi(s,a)-\bar{\beta}\sqrt{\bphi(s,a)^{\top}\bSigma_{k,h}^{-1}\bphi(s,a)},
 \end{align*}
 where $\bar{\beta}=\tilde O\Big(\sqrt{d^3H^2}\Big)$. We can show that the pessimistic value function $\check{V}_{k,h}(s)$ is a lower bound for the optimal value function $V_{h}^*(s)$.
% Next, we will show several key differences between Algorithm~\ref{algorithm1} and UCRL-VTR+. \CC{why compare with HF-UCRL-VTR+?}
\subsection{Variance Estimator}\label{sec:varest}
We compare our variance estimator and its counterparts in \citet{zhou2021nearly}. 
    % \begin{remark}
    %     The Bernstein-type self-normalized martingale inequality is first proposed by \citet{zhou2021nearly}. Later, \citet{hu2022nearly} and \citet{zhou2022computationally} improved the martingale inequality by introducing the uncertainty-dependent term.
    % \end{remark}
\citet{zhou2021nearly} first introduced variance estimators into RL with linear function approximation. They studied linear mixture MDPs, and their algorithm estimates the variance of the \emph{optimistic} value function $V_{k,h+1}(s)$ directly. In comparison, for linear MDPs, estimating the variance of the \emph{optimistic} value function $V_{k,h+1}(s)$ will encounter the dependence issue, which is discussed in \citet{jin2020provably} and will introduce an additional $\sqrt{d}$ factor in the regret due to the covering-based decoupling argument.
Inspired by the previous works \citep{azar2017minimax,hu2022nearly}, 
we decompose the \emph{optimistic} value function $V_{k,h+1}(s)$ into the \emph{optimal} value function $V_{h+1}^*(s)$ and the sub-optimality gap $V_{k,h+1}(s)-V_{h+1}^*(s)$, then estimate their variances $[\VV_h V_{h+1}^*](s,a)$ and  $\big[\VV_h (V_{k,h+1}-V_{h+1}^*)\big](s,a)$ separately.

For the variance of \emph{optimal} value function $[\VV_h V_{h+1}^*](s,a)$, since neither the variance operator $\VV_h$ nor the optimal value function $V_{h+1}^*$ is observable, Algorithm \ref{algorithm1} takes several steps to estimate these two quantities. In detail, Algorithm \ref{algorithm1} uses the optimistic value function $V_{k,h+1}$ to estimate the  optimal value function $V_{h+1}^*$ and introduce an error term $D_{k,h}$ to bound the difference between $\VV_h V_{k,h+1}$ and $\VV_h V_{h+1}^*$.
For the variance operator, Algorithm \ref{algorithm1} follows \citet{zhou2021nearly} to write the variance as the difference between the second-order moment and the square of the first-order moment of $V_{k,h}$, which is upper bounded by the bonus term $E_{k,h}$. 
%The rest part of the variance estimator shares the same structure as the variance estimator in \citet{zhou2021nearly}. 
More specifically, the variance of function $V_{k,h}$ can be denoted by
\begin{align*}
    [\VV_h \vvalue_{k,h}](s,a)=[\PP_h V_{k,h}^2](s,a)-\big([\PP_h V_{k,h}](s,a)\big)^2.
\end{align*}
According to the Proposition \ref{prop:linearq}, the expectation $\PP_h V_{k,h}(s,a)$ and $\PP_h V^2_{k,h}(s,a)$ are linear in the feature mapping $\bphi(s,a)$ and can be approximated as follows,
\begin{align}
    &[\VV_h \vvalue_{k,h}](s,a)%&=[\PP_h V_{k,h}^2](s,a)-\big([\PP_h V_{k,h}](s,a)\big)^2\notag\\
    \approx \bar\VV_h\vvalue_{k,h+1}(s_h^k,a_h^k)\notag\\
    &:= \big[\tilde{\wb}^{\top}_{k,h}\bphi(s_h^k,a_h^k)\big]_{[0,H^2]}-\big[\hat{\wb}^{\top}_{k,h}\bphi(s_h^k,a_h^k)\big]_{[0,H]}^2, \notag
\end{align}
where %$\tilde{\wb}_{k,h}$ is computed by
\begin{align*}
\tilde{\wb}_{k,h}&:=\argmin_{\wb\in \RR^d}\lambda\|\wb\|_2^2\notag\\
&\qquad +\textstyle{\sum_{i=1}^{k-1}}\bar{\sigma}_{i,h}^{-2}\big(\wb^{\top}\bphi(s_h^i,a_h^i)-V^2_{k,h+1}(s_{h+1}^i)\big)^2 \\
%&=\bSigma_{k,h}^{-1}\sum_{i=1}^{k-1}\bar\sigma_{i,h}^{-2}\bphi(s_h^i,a_h^i){\vvalue}^2_{k,h+1}(s_{h+1}^{i})
\end{align*}
is %the estimation vector for the value function square $V_{k,h+1}^2$ from 
the solution to the weighted ridge regression problem for the squared value function. To summarize, $\algname$ constructs the estimated variance $\sigma_{k,h}$ as follows:
\begin{align}
\sigma_{k,h}=\sqrt{[\bar{\VV}_{k,h}\vvalue_{k,h+1}](s_h^k,a_h^k)+E_{k,h}+D_{k,h}+H},\label{eq:variance}
\end{align}
%%
%\CC{merge the above equation with the earlier one?}
%\CC{$\tilde{\wb}$ is not defined anywhere} 
%is the empirical variance of $V_{k,h+1}$ at state-action pair $(s_h^k, a_h^k)$ and 
where $E_{k,h}$ and $D_{k,h}$ are defined as follows
\begin{align}
    E_{k,h}&=\min \Big\{\tilde{\beta}\big\|\bSigma_{k,h}^{-1/2}\bphi(s_h^k,a_h^k)\big\|_2,H^2\Big\}\notag\\
    &\qquad +\min \Big\{2H\bar{\beta}\big\|\bSigma_{k,h}^{-1/2}\bphi(s_h^k,a_h^k)\big\|_2,H^2\Big\},\notag\\
    D_{k,h}&=\min\bigg\{4d^3H^2\Big(\hat{\wb}_{k,h}^{\top}\bphi(s_h^k,a_h^k)-\check{\wb}_{k,h}^{\top}\bphi(s_h^k,a_h^k)\notag\\
&+2\bar{\beta}\sqrt{\bphi(s_h^k,a_h^k)^{\top}\bSigma_{k,h}^{-1}\bphi(s_h^k,a_h^k)}\Big),d^3H^3\bigg\}.\notag
\end{align}
%\CC{$(s,a)$ missing subscripts? }
Here $E_{k,h}$ is the error between the estimated variance and the true variance of $V_{k,h+1}$, and $D_{k,h}$ is the error between the variance of $V_{k,h+1}$ and the variance of the optimal value function $V_h^*$. For term $D_{k,h}$, %\CC{we introduce an additional estimation sequence $\check V_{k,h}$, which serves as a \emph{pessimistic} estimation of $V_h^*$, and 
we use the difference between the optimistic value function $V_{k,h}$ and the pessimistic value function $\check V_{k,h}$ to bound the difference between $V_{k,h}$ and $V_h^*$. %In addition, the variance $\big[\VV_h (V_{k,h+1}-V_{h+1}^*)\big](s,a)$ is also bounded by the term $D_{k,h}$. 
More discussions on the decomposition and the variance estimator can be found in Section \ref{section:key-tech}.
%The necessarity of why Algorithm \ref{algorithm1} needs to directly do regression corresponding with the optimal value function is going to be explained in the Section \ref{section:key-tech}. 

\section{Main Results}
In this section, we provide the regret bound for our algorithm $\algname$.
\begin{theorem}\label{theorem-1}
    For any linear MDP $M$, if we set the parameters
    $\lambda =1/H^2$ and confidence radius $\beta, \bar\beta, \tilde\beta$ as
    \begin{align}
    &\beta=O\Big(H\sqrt{d\lambda}+\sqrt{d \log^2\big(1+dKH/(\delta\lambda)\big)}\Big),\notag \\
    &\bar{\beta}=O\Big(H\sqrt{d\lambda} +\sqrt{d^3H^2\log^2\big(dHK/(\delta\lambda)\big)}\Big),\notag \\
    &\tilde{\beta}= O \Big(H^2\sqrt{d\lambda} +\sqrt{d^3H^4\log^2\big(dHK/(\delta\lambda)\big)}\Big),\notag
    \end{align}
    then with high probability of at least $1-7\delta$, the regret of $\algname$ is upper bounded as follows:
    \begin{align}
       \text{Regret}(K)\leq \tilde{O}\Big(d\sqrt{H^3K}+d^7H^8\Big).\notag
    \end{align}
    In addition, the number of updates for $Q_{k,h}, \check Q_{k,h}$ is upper bounded by $O(dH \log(1+K/\lambda))$. 
\end{theorem}

%\CC{number of policy switches?}
\begin{remark}
When the number of episodic $K$ satisfies that $K $ is large, the regret can be simplified as $\tilde{O}(d\sqrt{H^3K})$. Compared with the lower bound $\Omega(d\sqrt{H^3K})$ proved in \citet{zhou2021nearly}, our regret bound matches the lower bound up to logarithmic factors, which suggests that $\algname$ is near-optimal for linear MDPs.
\end{remark}
\noindent\textbf{Computational Complexity} As shown in \citet{jin2020provably}, the computational complexity of the original LSVI-UCB is $O(d^2|\cA|HK^2)$, where $\cA$ is a finite action space and $|\cA|$ is the size of the action set.
Compared with the LSVI-UCB algorithm, Algorithm \ref{algorithm1} uses the ``rare-switching'' technique, where the algorithm only updates the estimated value functions if the determinant of the covariance matrix $\det(\bSigma_{h,k})$ doubles (Line \ref{algorithm:det}).
According to Lemma \ref{lemma:update-number}, the number of episodes that triggers the updating criterion is at most $dH\log (1+K/\lambda)$ and the action-value function $Q_{k,h}(s, a)$ can be represented as a minimum over $dH\log (1+K/\lambda)$ quadratic functions. 
Therefore, given all previous optimistic weight vectors $\wb_{i,h}$ and covariance matrices $\bSigma_{i,h}$, computing the optimistic value function $Q_{k,h}(s,a)$ needs $\tilde O(d^3H)$ computational complexity. Thus, for each episode $k\in[K]$, calculating the value function $Q_{k,h}(s_h^k,a)$, choosing the action $a_h^k\leftarrow \arg\max_{a}Q_{k,h}(s_h^k,a)$ and 
estimating the variance $\bar{\sigma}_{k,h}$ will only lead to $\tilde O(d^3H^2|\cA|)$ computational complexity. %\CC{how to get it?}.

For computing the linear regression weight vectors (Line \ref{algorithm:line0} to Line \ref{algorithm:line1}), if the updating criterion is not triggered in episode $k$, then  $\algname$ only needs to update the weight vectors $\hat\wb_{k,h}$ and $\check\wb_{k,h}$. Since the value functions $V_{k,h+1}$ and $\check{V}_{k,h+1}$ remain unchanged, we only need to compute the new terms $\sigma_{k-1,h}^{-2}\bphi(s_h^{k-1},a_h^{k-1})V_{k,h+1}(s_{h+1}^{k-1})$ and $\sigma_{k-1,h}^{-2}\bphi(s_h^{k-1},a_h^{k-1})\check{V}_{k,h+1}(s_{h+1}^{k-1})$,
%$\hat{\wb}_{k,h}=\hat{\wb}_{k-1,h}+\sigma_{k-1,h}^{-2}\bphi(s_h^{k-1},a_h^{k-1})V_{k,h+1}$ and $\check{\wb}_{k,h}=\check{\wb}_{k-1,h}+\sigma_{k-1,h}^{-2}\bphi(s_h^{k-1},a_h^{k-1})\check{V}_{k,h+1}$ 
%\CC{is the formulate correct? why?}, 
which has an $O(d^3H|\cA|)$ computational complexity. 
On the other hand, if the updating criterion is triggered in episode $k$, then $\algname$ needs to update the value function and recalculate the weight vectors $\hat{\wb}_{k,h}, \check{\wb}_{k,h}$, which incurs an $\tilde O(d^4H^2|\cA|K)$ computational complexity. Combining the computational complexity for all episodes and noticing that the number of episodes that trigger the updating criterion is at most $\tilde O(dH)$, the total computational complexity of $\algname$ is $\tilde O(d^4H^3|\cA|K)$, which improves the original LSVI-UCB algorithm by a factor of $K$. %\CC{why $d$ dependence is much worse?}.

\section{Overview of Key Techniques}\label{section:key-tech}
In this section, we provide an overview of the key techniques in our algorithm design and analysis.

\subsection{ Decompose $V_{k,h+1}$ to $V^*_{h+1}$ and $V_{k,h+1}-V^*_{h+1}$} \label{section:keyidea-1}
We start with estimating the Bellman backup $[\PP_hV_{k,h+1}](s,a)$, which is the main difficulty in almost all existing analyses of algorithms for linear MDPs. According to Proposition \ref{prop:linearq}, for any value function $V$ and state-action pair $(s,a)$, the Bellman backup $[\PP_h V](s,a)$ is always a linear function of the feature mapping $\phi(s,a)$ %and the expectation $[\PP_h V](s,a)$ 
can be approximated as follows
\begin{align*}
    &[\hat{\PP}_{k,h}  V_{k,h+1}](s,a)\notag\\
    &\approx \bphi(s,a)^{\top}\bSigma_{k,h}^{-1}\sum_{i=1}^{k-1}\sigma_{i,h}^{-2}\bphi(s_h^i,a_h^i)V_{k,h+1}(s_{h+1}^i),
\end{align*}
which utilizes all the past observations $\phi(s_h^i, a_h^i)$ and the associated values $V(s_{h+1}^k)$.
In addition, the estimation error for this estimator can be measured by
\begin{align}
    &[\hat{\PP}_{k,h} V_{k,h+1}](s,a)-[\PP_{h}  V_{k,h+1}](s,a)\notag\\
 %   &\approx \bphi(s,a)^{\top}\bSigma_{k,h}^{-1}\sum_{i=1}^{k-1}\sigma_{i,h}^{-2}\bphi(s_h^i,a_h^i)V_{k,h+1}(s_{h+1}^i)\notag\\
 %   &\qquad -\bphi(s,a)^{\top}\bSigma_{k,h}^{-1}\sum_{i=1}^{k-1}\sigma_{i,h}^{-2}\bphi(s_h^i,a_h^i)[\PP_hV_{k,h+1}](s_{h+1}^i)\notag\\
&\approx\bphi(s,a)^{\top}\bSigma_{k,h}^{-1}\sum_{i=1}^{k-1}\sigma_{i,h}^{-2}\bphi(s_h^i,a_h^i)\eta_{i,h}(V_{k,h+1}),\label{eq:0-1}
\end{align}
where $\eta_{i,h}(V)=V(s_{h+1}^i)-[\PP_hV](s_{h+1}^i)$ denotes the stochastic transition noise at episode $i$ with value function $V$.
According to the Bernstein-type self-normalized martingale inequality (Lemma \ref{lemma:concentration_variance}), the summation of stochastic noise can be bounded by a small value (e.g., $\big|[\hat{\PP}_{k,h} V](s,a)-[\PP_{h} V](s,a)\big|\leq \beta\sqrt{\bphi(s,a)^{\top}\bSigma_{k,h}^{-1}\bphi(s,a)}$). 
However, \citet{jin2020provably} noticed that the estimation of the optimistic value function $V_{k,h+1}$ \emph{depends} on the past observations $(s_h^k, a_h^k, s_{h+1}^k)$, which violates the conditional independence required by the martingale concentration inequality, i.e., $\EE\big[\eta_{i,h}(V_{k,h+1})\big]\ne 0$.

%and the Bernstein-type self-normalized martingale inequality cannot be applied directly. 
To deal with this problem, \citet{jin2020provably} applied the uniform convergence argument based on covering number for all possible value functions and introduce a fixed covering set to replace $V_{k,h}$ in their analysis. In detail, the function class considered in \citet{jin2020provably} is denoted by
\begin{align}
    \mathcal{V}&=\bigg\{V\bigg|V(\cdot)=\max_{a}\min\bigg(H,\wb_{i}^\top\bphi(\cdot,a)\notag\\
    &\qquad +\beta\sqrt{\bphi(\cdot,a)^\top \bSigma_{i}^{-1}\bphi(\cdot,a)}\bigg),\|\wb_i\|\leq L,\bSigma_i \succeq \lambda \Ib\bigg\}.\notag
\end{align}
We can cover $\mathcal{V}$ by an $\epsilon$-net denoted by $\cN_\epsilon$, and its covering entropy $\log \mathcal{N}_\epsilon$ satisfies $\log \mathcal{N}_\epsilon=\tilde O(d^2)$.
Such an approach, although fixing the dependency issue in \eqref{eq:0-1}, %introduces an additional cost named \emph{covering number}, which depends on the function class $V_{k,h}$ belongs to.
%That 
introduces an additional $\sqrt{d}$ factor to their final regret since each $V$ belongs to a quadratic function class by their optimistic construction, which prevents them from achieving the optimal $d$ dependency in the regret. 

In comparison, our approach gets around the covering number issue by decomposing the value function $V_{k,h}$ into the optimal value function $V^*_{h+1}$ and the sub-optimality gap $V_{k,h+1}-V^*_{h+1}$. Such an analysis approach has been firstly considered in the tabular MDPs \citet{azar2017minimax, zhang2021reinforcement} and later in the linear MDP by \citet{hu2022nearly}. More specifically, we have
\begin{align}
    &[\hat{\PP}_{k,h} V_{k,h+1}](s,a)-[\PP_{h}  V_{k,h+1}](s,a)\notag\\
    &\approx\bphi(s,a)^{\top}\bSigma_{k,h}^{-1}\sum_{i=1}^{k-1}\sigma_{i,h}^{-2}\bphi(s_h^i,a_h^i)\eta_{i,h}(V_{k,h+1})\notag\\
    &=\underbrace{\bphi(s,a)^{\top}\bSigma_{k,h}^{-1}\sum_{i=1}^{k-1}\sigma_{i,h}^{-2}\bphi(s_h^i,a_h^i)\eta_{i,h}(V_{h+1}^*)}_{I_1}\notag\\
    & + \underbrace{\bphi(s,a)^{\top}\bSigma_{k,h}^{-1}\sum_{i=1}^{k-1}\sigma_{i,h}^{-2}\bphi(s_h^i,a_h^i)\eta_{i,h}(\Delta V_{k,h+1})}_{I_2},\label{eq:0-2}
\end{align}
where $\Delta V_{k,h+1}=V_{k,h+1}-V^*_{h+1}$ denotes the estimation error for value function $V_{k,h+1}$. For the first term $I_1$, as discussed in Section \ref{sec:varest}, we can use $V_{k,h+1}$ to approximate the optimal value function $V_{h+1}^*$ and the estimation error for the variance can be bounded by:
    \begin{align}
    \big|[\bar\VV_h\vvalue_{k,h+1}](s_h^k,a_h^k)         -[\VV_h\vvalue_{h+1}^*](s_h^k,a_h^k)\big|\leq E_{k,h}+D_{k,h}.\notag
\end{align}
Since the optimal value function $V_{h+1}^*$ is fixed across all episodes $k\in[K]$ and does not depend on the past observations, such an approach can prevent the covering number argument and save a $\sqrt{d}$ factor in the regret compared with \citet{jin2020provably}. For the second term $I_2$, the sub-optimality gap $\Delta V_{k,h+1}=V_{k,h+1}-V^*_{h+1}$ depends on the past observations and we still need to use the covering number argument. However, the magnitude of the sub-optimality gap $\Delta V_{k,h+1}$ is small provided that $V_{k,h+1}$ is an accurate estimate for $V_{h+1}^*$. In this case, term $I_2$ will be dominated by term $I_1$ even with the extra factors from the covering number argument. %which will not affect the estimate severely.
With the help of the decomposition, we have the following 
Bernstein-type error bound between the estimated $\hat{\PP}_{k,h} V_{k,h+1}$ and its true value:
 \begin{align}
    \big|\hat{\wb}_{k,h}^{\top}\bphi(s,a)-\PP_h\vvalue_{k,h+1}(s,a)\big| \leq \beta \|\bphi(s,a)\|_{\bSigma_{k,h}^{-1}}.\notag
\end{align}
 This analysis also explains why Algorithm \ref{algorithm1} needs to estimate the variance of $V_{h+1}^*$ and $\Delta V_{k,h+1}$ instead of $V_{k,h+1}$. %in Section \ref{sec:varest} that 

\subsection{ Monotonic Variance Estimator}\label{sec-montinoc}
Here we provide more details about the 
variance estimator $\sigma_{k,h}$. According to our previous discussion, we decompose the value function $V_{k,h}$, and only need to control the estimation errors $I_1$ and $I_2$ in~\eqref{eq:0-2} separately. In order to derive a Bernstein-type
 error bound, we use Lemma \ref{lemma:concentration_variance} for both the optimal value function $V_{h+1}^*$ and $\Delta V_{k,h+1}$, which require an estimation for the variance $\VV_h V_{h+1}^*$ and $\VV_h [V_{k,h+1} - V_{h+1}^*]$. For the variance $\VV_h V_{h+1}^*$, as we discussed in Section \ref{sec:varest}, we approximate it with the following empirical variance:
\begin{align}
&\bar\VV_h\vvalue_{k,h+1}(s_h^k,a_h^k) - [\VV_h V_{h+1}^*](s_h^k,a_h^k) \notag \\
&= \underbrace{\bar\VV_h\vvalue_{k,h+1}(s_h^k,a_h^k) - [\VV_h V_{k,h+1}](s_h^k,a_h^k)}_{J_1}\notag \\
&\quad + \underbrace{[\VV_h V_{k,h+1}](s_h^k,a_h^k) - [\VV_h V_{h+1}^*](s_h^k,a_h^k)}_{J_2},\notag
\end{align}
%\CC{$I_1$ and $I_2$ are repeatedly used}
% \begin{align}
%     \bar\VV_h\vvalue_{k,h+1}(s_h^k,a_h^k)
%     &= \big[\tilde{\wb}^{\top}_{k,h}\bphi(s_h^k,a_h^k)\big]_{[0,H^2]}-\big[\hat{\wb}^{\top}_{k,h}\bphi(s_h^k,a_h^k)\big]_{[0,H]}^2\notag\\
%     &\approx \big[\PP_h V^2_{k,h+1}\big](s,a)-\big[\PP_h V_{k,h+1}\big]^2(s,a)\notag\\
%     &= [\VV_h V_{k,h+1}](s,a)\notag\\
%     &\approx [\VV_h V_{h+1}^*](s,a),\notag
% \end{align}
 where the estimation error $J_1$ can be controlled by a Hoeffding-type bound (term $E_{k,h}$) and the estimation error $J_2$ can be upper bounded by 
\begin{align}
     &\big|[\VV_h\vvalue_{k,h+1}](s_h^k,a_h^k)         -[\VV_h\vvalue_{h+1}^*](s_h^k,a_h^k)\big|\notag\\
     &\leq 4H [\PP_h(\vvalue_{k,h+1}-\vvalue_{h+1}^*)](s_h^k,a_h^k).\notag
\end{align}
%\CC{dont use $\times $}
 For this error bound, the optimal value function $V^*_{h+1}$ is not observable and we replace it by the \emph{pessimistic} value function $\check V_{k,h}$, which gives an upper bound
\begin{align*}
    &4H [\PP_h(\vvalue_{k,h+1}-\vvalue_{h+1}^*)](s_h^k,a_h^k)\notag\\
    &\leq 4H [\PP_h(\vvalue_{k,h+1}-\check{\vvalue}_{k,h+1})](s_h^k,a_h^k).
\end{align*}
The above term is further dominated by the term $D_{k,h}$.
%\CC{why do we need the following?} 
With a similar approach, for the variance $\VV_h [V_{k,h+1} - V_{h+1}^*]$, it can be upper bound by
 \begin{align}
        &[\VV_h(\vvalue_{k,h+1}-\vvalue_{h+1}^*)](s_h^k,a_h^k)\notag\\
        &\leq 2H [\PP_h(\vvalue_{k,h+1}-\vvalue_{h+1}^*)](s_h^k,a_h^k)\notag\\
        &\leq 2H [\PP_h(\vvalue_{k,h+1}-\check{\vvalue}_{k,h+1})](s_h^k,a_h^k) \approx D_{k,h}/(d^3H),\label{eq:0-3}
    \end{align}
where we approximate the optimal value function $V^*_{h+1}$ by the \emph{pessimistic} value function $\check V_{k,h}$ and introduce an extra $d^3H$-factor in the $D_{k,h}$ to offset the error caused by the covering number argument.

 However, there exists another difficulty in our algorithm and analysis when we extend the result in \eqref{eq:0-3} to future episode $i>k$. In detail, while the value function $V_{i,h+1}(s_{h+1}^k)$ and corresponding variance $[\VV_h(\vvalue_{i,h+1}-\vvalue_{h+1}^*)](s_h^k,a_h^k)$ 
 will change across different episodes, the estimated variance $\sigma_{k,h}$ is chosen at episode $k$ and cannot be changed in the subsequent episode. Therefore, $\sigma_{k,h}$ should be a uniform variance upper bound for all subsequent episodes. To achieve such a uniform variance upper bound, it suffices to have the 
 sub-optimality gap ${\vvalue}_{k,h+1}-V_{h+1}^*$ to be monotonically decreasing. Our solution is to set $V_{k,h+1}(s)$ to be a monotonically decreasing sequence in $k$ given any state $s$, by setting it as the minimum between its current estimate and its predecessor $V_{k-1, h+1}$ (Line \ref{algorithm1:min} of Algorithm \ref{algorithm1}). A similar approach is also applied to $\check V_{k,h+1}$ to guarantee the estimate sequence is monotonically increasing in $k$. Then, the following property shows that the estimated variance of the sub-optimality at episode $k$ $D_{k,h}$ holds for all the subsequent episodes.
 \begin{align*}
         [\VV_h(\vvalue_{i,h+1}-\vvalue_{h+1}^*)](s_h^k,a_h^k)\leq D_{k,h}/(d^3H), \forall i>k.
    \end{align*}
 This idea was firstly introduced by \citet{azar2017minimax} for tabular MDPs. \citet{hu2022nearly} adopted a similar idea to guarantee the monotonicity for linear MDPs, while their approach is to construct another sequence of  ``over-optimistic'' value functions, which turns out to be flawed as we will discuss in Appendix \ref{section:wrong}.

% The next lemma summarizes our idea in this section. 

% \CC{Original informal lemma}
% On some event, for any episode $k$ and $i\ge k$, we have
%     \begin{align*}
%         [\VV_h(\vvalue_{i,h+1}-\vvalue_{h+1}^*)](s_h^k,a_h^k)\leq D_{k,h}/(d^2H). 
%     \end{align*}

\subsection{Rare-Switching Value Function Update}
As we discussed in Section \ref{sec-montinoc}, we ensure the monotonicity and construct the variance estimation, by taking minimization with its predecessor $V_{k-1, h+1}$.
However, this approach will introduce an extra issue for the augmented value function class. In detail, the optimistic value function $V_{k,h}$ can be denoted by the minimum over several quadratic functions and belongs to the following function class,
\begin{small}
\begin{align}
    \mathcal{V}_h&=\Bigg\{V\bigg|V(\cdot)=\max_{a}\min_{1\leq i\leq l}\min\bigg(H,r_h(\cdot,a)+\wb_{i}^\top\bphi(\cdot,a)\notag\\
    &\qquad +\beta\sqrt{\bphi(\cdot,a)^\top \bSigma_{i}^{-1}\bphi(\cdot,a)}\bigg),\|\wb_i\|\leq L,\bSigma_i \succeq \lambda \Ib\Bigg\},\notag
\end{align}
\end{small}
where $l$ is the number of quadratic functions and equals to the number of policy updates in Algorithm \ref{algorithm1}. Here, we denote the covering number of that function class by $\cN$, and the covering number of a quadratic function class by $\cN_{q}$. 
We are specifically interested in the covering entropy $\log \cN$, which is a standard complexity measure of the function class, and it will directly affect the regret of our algorithm. The standard approach to computing the covering entropy suggests that $\log \cN = l\log \cN_q$. In this case, if we update the value function at each episode $k\in[K]$ and minimize with its predecessor, then there will be an extra $K$ factor in the covering number, which is unacceptable.
Inspired by the ``rare-switching'' technique \citep{AbbasiYadkori2011ImprovedAF,wang2021provably}, it is not necessary or efficient to update the value functions $V_{k,h+1}$ and $\check V_{k,h+1}$ at each episode. Instead, we only need to update the value function when the determinant of the covariance matrix grows much larger than before (Line \ref{algorithm:det} in Algorithm \ref{algorithm1}), which requires at most $\tilde O(dH)$ updates.
Such an update strategy reduces the covering entropy from $\log \cN = K\log \cN_q$ to $\log \cN = \tilde O(dH)\log \cN_q$, which makes the regret of Algorithm \ref{algorithm1} tight. In addition, the ``rare-switching'' nature also reduces the computational complexity from $\Omega(K^2)$ to $\Omega(K)$, which makes $\algname$ more efficient.

%\CC{covering number or covering entropy}

%The final issue in our algorithm design is how to deal with a slightly complex value function class containing $V_{k,h+1}$, which can be regarded as a minimum over number $K$ of quadratic functions. Denote the covering number of that function class as $\cN$, and the covering number of a quadratic function class as $\cN_{q}$. We are specifically interested in the term $\log \cN$, which turns out to be a standard measure of function class complexity, and it will directly affect the final regret of our algorithm. The standard approach to compute covering number suggests $\log \cN = K\log \cN_q$, and the $K$ factor makes the final regret unacceptable. A key observation is that we do not need to update $V_{k,h+1}$ and $\check V_{k,h+1}$ each episode. Instead, only an $\tilde O(dH)$ number of updates are required by using a `rare-switching' technique from previous works \citep{AbbasiYadkori2011ImprovedAF, gao2021provably, wang2021provably}, which requires the algorithm to stay at the current value function estimate unless the uncertainty of the value function grows much larger than before (line 6 in Algorithm \ref{algorithm1}). Such an update strategy addresses the covering number issue, which makes the final regret of Algorithm \ref{algorithm1} satisfying. 

% \begin{align*}
%     V_{k,h}(s)=V_{k_{\text{last}},h}(s), Q_{k,h}(s,a)=Q_{k_{\text{last}},h}(s,a).
% \end{align*}

\section{Conclusions and Future Work}

In this paper, we propose a near-optimal algorithm $\algname$ for linear MDPs. $\algname$ is based on weighted ridge regression, where the weights are constructed from a novel variance estimator that comes from a direct estimation of the variance of the true value function, and a ``rare-switching'' updating rule to update the value function estimator. We prove that with high probability, $\algname$ obtains an $\tilde{O}(d\sqrt{H^3K})$ regret, which matches the lower bound in \citet{zhou2021nearly} up to logarithmic factors. Our algorithm is also computationally efficient. In the future, we will study how to design computationally efficient near-optimal RL algorithms with general nonlinear function approximation with misspecification.

\section*{Acknowledgements}

We thank the anonymous reviewers for their helpful comments. JH, HZ, DZ and QG are supported in part by the National Science Foundation CAREER Award 1906169 and the Sloan Research Fellowship.  The views and conclusions contained in this paper are those of the authors and should not be interpreted as representing any funding agencies.

\bibliography{ref.bib}
\bibliographystyle{ims}
%%%%%%%%%%%%%%%%%%%%%%%%%%%%%%%%%%%%%%%%%%%%%%%%%%%%%%%%%%%%%%%%%%%%%%%%%%%%%%%
%%%%%%%%%%%%%%%%%%%%%%%%%%%%%%%%%%%%%%%%%%%%%%%%%%%%%%%%%%%%%%%%%%%%%%%%%%%%%%%
% APPENDIX
%%%%%%%%%%%%%%%%%%%%%%%%%%%%%%%%%%%%%%%%%%%%%%%%%%%%%%%%%%%%%%%%%%%%%%%%%%%%%%%
%%%%%%%%%%%%%%%%%%%%%%%%%%%%%%%%%%%%%%%%%%%%%%%%%%%%%%%%%%%%%%%%%%%%%%%%%%%%%%%
\newpage
\appendix
\onecolumn
\section{Comparison with \citet{hu2022nearly}}\label{section:wrong}
In this section, we give a detailed comparison with \citet{hu2022nearly}. We first elaborate on the importance of the monotonic property, then discuss the issue on the over-optimistic value function $\dot{\widehat{V}}_{k,h}(s)$ proposed in \citet{hu2022nearly} and finally illustrate the difference in the algorithm design between the algorithm in \citet{hu2022nearly} and our algorithm. 

As we discuss in the Section \ref{section:key-tech}, both our $\algname$ algorithm and LSVI-UCB+ algorithm in \citet{hu2022nearly} get rid of the covering number issue by decomposing the value function $V_{k,h+1}(s)$ to $V_{h+1}^*(s)$ and $V_{k,h+1}(s)-V_{h+1}^*(s)$. In detail, from the proof of Lemma \ref{lemma:002}, we have shown in \eqref{eq:0003} that the estimation error can be decomposed as
%\CC{change Big to bigg}
\begin{align}
    &\bigg\|\bSigma_{k,h}^{-1}\sum_{i=1}^{k-1}\bar\sigma_{i,h}^{-2}\bphi(s_h^i,a_h^i)\big(\vvalue_{k,h+1}(s_{h+1}^{i})-\PP_h\vvalue_{k,h+1}(s_h^i,a_h^i)\big)\bigg\|_{\bSigma_{k,h}}\notag\\
    &=\bigg\|\sum_{i=1}^{k-1}\bar\sigma_{i,h}^{-2}\bphi(s_h^i,a_h^i)\big(\vvalue_{k,h+1}(s_{h+1}^{i})-\PP_h\vvalue_{k,h+1}(s_h^i,a_h^i)\big)\bigg\|_{\bSigma_{k,h}^{-1}}\notag\\
    &\leq \underbrace{\bigg\|\sum_{i=1}^{k-1}\bar\sigma_{i,h}^{-2}\bphi(s_h^i,a_h^i)\big(\vvalue^*_{h+1}(s_{h+1}^{i})-\PP_h\vvalue^*_{h+1}(s_h^i,a_h^i)\big)\bigg\|_{\bSigma_{k,h}^{-1}}}_{J_1}\notag\\
    &\qquad +\underbrace{\bigg\|\sum_{i=1}^{k-1}\bar\sigma_{i,h}^{-2}\bphi(s_h^i,a_h^i)\big(\Delta{\vvalue}_{k,h+1}(s_{h+1}^{i})-[\PP_h(\Delta{\vvalue}_{k,h+1})](s_h^i,a_h^i)\big)\bigg\|_{\bSigma_{k,h}^{-1}}}_{J_2}.\notag
\end{align}
To control the concentration error on the term $J_2$, we use Lemma \ref{lemma:concentration_variance} and only need to estimate the variance $\VV_h[\Delta V_{k,h+1}](s,a)=\VV_h\big[ (V_{k,h+1}-V_{h+1}^*)\big](s,a)$, which is trivially upper bounded by $2H\cdot [\PP_h \Delta V_{k,h+1}](s,a)$. In order to make the upper bound of variance at episode $i$ hold for all subsequent episode $k>i$, we need to guarantee that the trivial upper bound $2H\cdot [\PP_h \Delta V_{k,h+1}](s,a)$ is decreasing in $k$, which requires the optimistic value function $V_{k,h+1}(s)$ to be monotonically decreasing.

To satisfy this requirement, \citet{hu2022nearly} constructs an over-optimistic value function $\dot{\widehat{V}}_{k,h}(s)$, which has the following monotonicity property.
\begin{lemma}[Lemma D.2, \citealt{hu2022nearly}]\label{lemma:hu}
    For any stage $h\in[H]$ and episodes $i< j$, the over-optimistic value function $\dot{\widehat{V}}_{j,h}(s)$ satisfies:
    \begin{align*}
        \dot{\widehat{V}}_{i,h}(s)\ge \hat{V}_{j,h}(s),
    \end{align*}
    where $\hat{V}_{j,h}(s)$ is the optimistic value function.
\end{lemma}
Based on this monotonically decreasing property, the estimation error
$2H\cdot \big[\PP_h(\dot{\widehat{V}}_{i,h+1}-V_{h+1}^*)\big](s,a)$ is a uniform variance upper bound for all subsequent episodes.
Unfortunately, the last inequality in the proof of Lemma \ref{lemma:hu} claims that $[\PP_h V_{i,h+1}](s,a)-[\PP_h V_{j,h+1}](s,a)$ holds due to $\dot{\widehat{V}}_{i,h}(s)\ge \hat{V}_{j,h}(s)$, which is not true. Thus, the monotonic property of over-optimistic value function $\dot{\widehat{V}}_{k,h}(s)$ does not hold and the estimated variance $\sigma_{i,h}$ may no longer be a variance upper bound for the subsequent episodes.

In comparison, our $\algname$ ensures the monotonic property by choosing the minimum of the optimistic value functions in the first $k$ episodes (See Line \ref{algorithm1:min}). As the cost of ensuring monotonic property, the resulting value function class can be regarded as a minimum over $K$ quadratic function classes, and the covering number grows exponentially in $K$. To overcome this problem, we utilize the 
`rare-switching' technique from previous works \citep{AbbasiYadkori2011ImprovedAF, wang2021provably}, which reduces the number of updates to 
$\tilde O(dH)$ and thus controls the complexity growth of the resulting value function class.

\section{Proof Sketch}\label{section:sketch}

This section is devoted to provide a proof sketch of Theorem \ref{theorem-1}. 

\subsection{High-Probability Events}\label{section-events}

We first introduce the following high-probability events: 
\begin{enumerate}[leftmargin = *]
    \item We define $\cE$ as the event that the following inequalities hold for all $s, a, k, h \in \cS \times \cA \times [K] \times [H]$. 
    \begin{align} 
            \big|\hat{\wb}_{k,h}^{\top}\bphi(s,a)-[\PP_h\vvalue_{k,h+1}](s,a)\big| \leq \bar{\beta} \sqrt{\bphi(s,a)^{\top}\bSigma_{k,h}^{-1}\bphi(s,a)},\notag\\
            \big|\tilde{\wb}_{k,h}^{\top}\bphi(s,a)-[\PP_h\vvalue^2_{k,h+1}](s,a)\big| \leq \tilde{\beta} \sqrt{\bphi(s,a)^{\top}\bSigma_{k,h}^{-1}\bphi(s,a)},\notag\\
            \big|\check{\wb}_{k,h}^{\top}\bphi(s,a)-[\PP_h\check{\vvalue}_{k,h+1}](s,a)\big| \leq \bar{\beta} \sqrt{\bphi(s,a)^{\top}\bSigma_{k,h}^{-1}\bphi(s,a)},\notag
    \end{align}
    where $\tilde{\beta}= O \Big(H^2\sqrt{d\lambda} +\sqrt{d^3H^4\log^2\big(dHK/(\delta\lambda)\big)}\Big)$ and $\bar{\beta}=O\Big(H\sqrt{d\lambda} +\sqrt{d^3H^2\log^2\big(dHK/(\delta\lambda)\big)}\Big)$.
    
    \item We define $\tilde \cE_h$ as the event such that for all episode $k\in[K]$, stage $h\leq h'\leq H$ and state-action pair $(s, a) \in \cS \times \cA$, the weight vector $\hat{\wb}_{k,h}$ satisfies that
    \begin{align}
    \big|\hat{\wb}_{k,h'}^{\top}\bphi(s,a)-[\PP_h\vvalue_{k,h'+1}](s,a)\big| \leq \beta \sqrt{\bphi(s,a)^{\top}\bSigma_{k,h'}^{-1}\bphi(s,a)},\label{eq:00000} 
\end{align}
where $\beta=O\Big(H\sqrt{d\lambda}+\sqrt{d \log^2\big(1+dKH/(\delta\lambda)\big)}\Big)$.
For simplicity, we further define events $\tilde {\cE}=\tilde {\cE_1}$ that \eqref{eq:00000} holds for all stage $h\in[H]$.
\end{enumerate}
Our ultimate goal is to show that  $\tilde\cE$ holds with high probability. Intuitively speaking, $\cE$ serves as a `coarse' event where the concentration results hold with a larger confidence radius $\tilde\beta$ and $\bar\beta$, and $\tilde\cE$ serves as a `refined' event where the confidence radius $\beta$ is tighter than $\tilde\beta$ and $\bar\beta$. To start with, the following lemma shows that $\cE$ holds with high probability. 

%\CC{for all vs for each}

\begin{lemma}\label{lemma:hoeffding-type}
    Event $\cE$ holds with probability at least $1-7\delta$.
\end{lemma}
Next, we prove $\tilde\cE = \tilde\cE_1$ holds with high probability. Since $\hat\wb_{k,h}$'s are obtained from weighted linear regression whose weights depend on the variances of $V_{k,h+1}$,  the key technical challenge is to show that our adapted weights $\sigma_{k,h}$'s are indeed upper bounds of these variances for all $h\in[H]$. We use backward induction to prove such a statement. In detail, the following two lemmas provide estimation error bounds at stage $h$ conditioned on $\tilde\cE_{h+1}$. 

%\CC{estimation variance}

\begin{lemma}\label{lemma:varaince}
      On the event $\cE$ and $\tilde{\cE}_{h+1}$, for each episode $k\in[K]$ and stage $h$, the estimated variance satisfies
      \begin{align}
          &\big|[\bar\VV_h\vvalue_{k,h+1}](s_h^k,a_h^k)         -[\VV_h\vvalue_{k,h+1}](s_h^k,a_h^k)\big|\leq E_{k,h},\notag\\
          &\big|[\bar\VV_h\vvalue_{k,h+1}](s_h^k,a_h^k)         -[\VV_h\vvalue_{h+1}^*](s_h^k,a_h^k)\big|\leq E_{k,h}+D_{k,h}.\notag
      \end{align}
\end{lemma}
%\CC{equation number}

\begin{lemma}\label{lemma:variance-dec}
    On the event $\cE$ and $\tilde{\cE}_{h+1}$, for any episode $k$ and $i>k$, we have
    \begin{align*}
        [\VV_h(\vvalue_{i,h+1}-\vvalue_{h+1}^*)](s_h^k,a_h^k)\leq D_{k,h}/(d^3H). 
    \end{align*}
\end{lemma}

We also have the following lemma, which shows that our constructed value functions $ \qvalue,  \vvalue$, and $\check \qvalue, \check\vvalue$ are optimistic and pessimistic estimators of the true value functions under the events we defined before. 
\begin{lemma}\label{lemma:optimistic}
    On the event $\cE$ and $\tilde{\cE}_h$, for all episode $k\in[K]$ and stage $h\leq h'\leq H$, we have $\qvalue_{k,h}(s,a)\ge \qvalue_{h}^*(s,a) \ge \check{\qvalue}_{k,h}(s,a)$. In addition, we have $\vvalue_{k,h}(s)\ge \vvalue_{h}^*(s) \ge \check{\vvalue}_{k,h}(s)$.
\end{lemma}

Equipped with Lemmas \ref{lemma:varaince}, \ref{lemma:variance-dec} and \ref{lemma:optimistic}, one can easily prove $\sigma_{k,h}$ indeed serves as an upper bound of the true variance of $V_{k,h+1}$ at stage $h$. Therefore, by the backward induction, we can prove the following lemma.

% For simplicity, we define $\cE$ as the event that Lemma \ref{lemma:hoeffding-type} holds. 
% Recall that we set the estimated variance term $\bar\VV_h\vvalue_{k,h+1}(s_h^k,a_h^k)$ as
% \begin{align}
%     \bar\VV_h\vvalue_{k,h+1}(s_h^k,a_h^k)
%     &= \big[\tilde{\wb}^{\top}_{k,h}\bphi(s_h^k,a_h^k)\big]_{[0,H^2]}-\big[\hat{\wb}^{\top}_{k,h}\bphi(s_h^k,a_h^k)\big]_{[0,H]}^2,\label{eq:variance2}
% \end{align}
% and denote the confidence bonus for the estimation variance as
% \begin{align}
%     E_{k,h}&=\min \Big\{\tilde{\beta}\big\|\bSigma_{k,h}^{-1/2}\bphi(s_h^k,a_h^k)\big\|_2,H^2\Big\}+\min \Big\{2H\bar{\beta}\big\|\bSigma_{k,h}^{-1/2}\bphi(s_h^k,a_h^k)\big\|_2,H^2\Big\},\notag\\
%     D_{k,h}&=\min\bigg\{4d^2H^2\Big(\hat{\wb}_{k,h}^{\top}\bphi(s,a)-\check{\wb}_{k,h}^{\top}\bphi(s,a)+2\bar{\beta}\sqrt{\bphi(s,a)^{\top}\bSigma_{k,h}^{-1}\bphi(s,a)}\Big),d^2H^3\bigg\}.\notag
% \end{align}
% Then we have the following lemma:

\begin{lemma}\label{lemma:002}
    On the events $\cE$, event $\tilde{\cE}$ holds with probability at least $1-\delta$.
\end{lemma}
% For simplicity, we define the events $\tilde{\cE}_h$ as the events that
% Lemma \ref{lemma:002} holds for all stage $h\leq h'\leq H$ and $\tilde{\cE}$ as the events that
% Lemma \ref{lemma:002} holds for all stage $h\in [H]$.

\subsection{Regret Decomposition}
Now, we prove the regret bound based on the high-probability events defined before. Based on Lemma \ref{lemma:optimistic}, for all stage $h\in[H]$ and episode $k\in[K]$, we have $Q_{k, h}(s_h^k, a_h^k) = V_{k, h}(s_h^k) \ge V^*_h(s_h^k)$. 
%From Lemma \ref{lemma:optimistic}, Lemma \ref{lemma:hoeffding-type} and Lemma \ref{lemma:002}, $Q_{k, h}(s_h^k, a_h^k) = V_{k, h}(s_h^k) \ge V^*_h(s_h^k)$ for all $k, h \in [K] \times [H]$. 
Thus, we can bound the regret as follows, \begin{align} 
\textbf{Regret}(K)&=\sum_{k=1}^K\big(\vvalue_{1}^*(s_1^k)-\vvalue_{k,1}^{\pi^k}(s_1^k)\big) \notag\\
&\leq\sum_{k=1}^K\big(\vvalue_{k,1}(s_1^k)-\vvalue_{k,1}^{\pi^k}(s_1^k)\big)\notag\\
        &\leq \sum_{k = 1}^K \sum_{h = 1}^H \left\{\big[\PP_h(\vvalue_{k,h+1}-\vvalue_{k,h+1}^{\pi^k})\big](s_h^k,a_h^k)-\big(\vvalue_{k,h+1}(s_{h+1}^{k})-\vvalue_{k,h+1}^{\pi^k}(s_{h+1}^{k})\big)\right\}\notag \\&\quad+O\left(\sum_{k = 1}^K \sum_{h = 1}^H \min\Big(\beta\sqrt{\bphi(s_h^k,a_h^k)^{\top}\bSigma_{k,h}^{-1}\bphi(s_h^k,a_h^k)},H\Big)\right),\notag
\end{align}
where the last inequality holds due to the decomposition of the difference of value functions and the high-probability events defined in Section \ref{section-events}. 
%\CC{explain the inequality}
Using standard regret analysis, we can bound the first term as the sum of a martingale difference sequence. Then it remains to bound the summation of bonus terms, 
    $\sum_{k=1}^K\sum_{h=1}^H \beta \|\bphi(s_h^k, a_h^k)\|_{\bSigma_{k, h}^{-1}}$. 
By Cauchy-Schwartz inequality, this summation can be bounded by \begin{align}
\sum_{k=1}^K\sum_{h=1}^H \beta \|\bphi(s_h^k, a_h^k)\|_{\bSigma_{k, h}^{-1}} \leq \tilde{O}\left(d^4H^8 + \beta d^7H^5 + \beta \sqrt{d HK + d H \sum_{k=1}^K\sum_{h=1}^H\sigma_{k, h}^2}\right), \notag 
\end{align} 
where the calculation details are deferred to Lemma \ref{lemma:sum-bonus}. According to the definition of $\sigma_{k, h}^2$, we have $\sigma_{k, h}^2 \le O\left([\bar{\VV}_{k,h}\vvalue_{k,h+1}](s_h^k,a_h^k)+E_{k,h}+D_{k,h}+H\right)$. By carefully bounding the summation of $[\bar{\VV}_{k,h}\vvalue_{k,h+1}](s_h^k,a_h^k)$ by relating them to the summation of $[{\VV}_{h}\vvalue^{\pi^k}_{k,h+1}](s_h^k,a_h^k)$ and using the total variance lemma (Lemma C.5, \citealt{jin2018q}), we have
\begin{align} 
    \sum_{k=1}^K\sum_{h=1}^H \sigma_{k, h}^2 \le \tilde{O}\left(H^2 K + d^{10.5}H^{16}\right). \notag
\end{align}
Putting all pieces together, we can obtain the high-probability regret bound $$\regret(K) \le \tilde{O}\left( d\sqrt{H^3K}+d^7H^8\right).$$ 

% \subsection{Upper bounding the variance of value functions}

% We apply the following two lemmas to show that our variance estimators $\sigma_{i, h}^2$ are large enough to upper bound all the underlying variance of value functions. 

% As discussed in Section \ref{section:keyidea-1}, we decompose the transition noise term $V_{k, h + 1}(s_{h + 1}^i) - \PP_h V_{k, h + 1}(s_h^i, a_h^i)$ into two terms: \begin{align} 
%     V_{h + 1}^*(s_{h + 1}^i) - \PP_h V_{ h + 1}^*(s_h^i, a_h^i),\ \Delta V_{k, h + 1}(s_{h + 1}^i) - [\PP_h \Delta V_{k, h + 1}](s_h^i, a_h^i), \notag
% \end{align}
% where the variance of the two terms is bounded under some high-probability events in Lemma \ref{lemma:varaince} and Lemma \ref{lemma:variance-dec} respectively. The simplified versions of the two lemmas are given as follows:

% Therefore, we can 

\section{Detailed Proof of Theorem \ref{theorem-1}}\label{appendix-1}
In this section, we provide the proof of Theorem \ref{theorem-1}.\cite{} Firstly, for the stochastic transition noises, we define the following high-probability events:
\begin{align}
    \cE_1=\bigg\{\forall h\in[H], &\sum_{k=1}^K\sum_{h'=h}^{H}  [\PP_h(\vvalue_{k,h+1}-\vvalue_{k,h+1}^{\pi^k})\big](s_h^k,a_h^k)\notag\\
    &\qquad - \sum_{k=1}^K\sum_{h'=h}^{H}\big(\vvalue_{k,h+1}(s_{h+1}^{k})-\vvalue_{k,h+1}^{\pi^k}(s_{h+1}^{k})\big) \leq 2\sqrt{2H^3K\log(H/\delta)} \bigg\},\notag\\
    \cE_2=\bigg\{\forall h\in[H], &\sum_{k=1}^K\sum_{h'=h}^{H}[\PP_h(\vvalue_{k,h+1}-\check{\vvalue}_{k,h+1})\big](s_h^k,a_h^k)\notag\\
    &\qquad -\sum_{k=1}^K\sum_{h'=h}^{H}\big(\vvalue_{k,h+1}(s_{h+1}^{k})-\check{\vvalue}_{k,h+1}(s_{h+1}^{k})\big)\leq 2\sqrt{2H^3K\log(H/\delta)}\bigg\}.\notag
\end{align}
Then according to the Azuma–Hoeffding inequality (Lemma \ref{lemma:azuma}), we have $\Pr(\cE_1)\ge 1-\delta$ and $\Pr(\cE_2)\ge 1-\delta$. Based on the definition of events $\cE_1,\cE_2$ and events $\cE,\tilde{\cE}$ in Section \ref{section-events}, the regret in the first $K$ episodes can be upper bounded by the summation of estimated variance $\sum_{k=1}^K \sum_{h=1}^H \sigma_{k,h}^2$ and we have the following lemma. 
\begin{lemma}\label{lemma:transition}
    On the events $\tilde{\cE}$, $\cE$ and $\cE_1$, for all stage $h\in[H]$, the regret in the first $K$ episodes is upper bounded by:
    \begin{align}
        &\sum_{k=1}^K\big(\vvalue_{k,h}(s_h^k)-\vvalue_{k,h}^{\pi^k}(s_h^k)\big)\leq 16d^4H^8\iota +40\beta d^7H^5\iota+ 8\beta\sqrt{2d H\iota\sum_{h=1}^H\sum_{k=1}^K(\sigma_{k,h}^2 + H)}+4\sqrt{H^3K\log(H/\delta)},\notag
    \end{align}
    and for all stage $h\in[H]$, we further have
\begin{align}
    &\sum_{k=1}^K\sum_{h=1}^H \big[\PP_h(\vvalue_{k,h+1}-\vvalue_{k,h+1}^{\pi^k})\big](s_h^k,a_h^k)\notag\\
    &\leq 16d^4H^9\iota +40\beta d^7H^6\iota+ 8H\beta\sqrt{2d H\iota\sum_{h=1}^H\sum_{k=1}^K(\sigma_{k,h}^2 + H)}+4\sqrt{H^5K\log(H/\delta)},\notag
\end{align}
where $\iota=\log\big(1+K/(d\lambda)\big)$.
\end{lemma}
In addition, for the sub-optimality gap between the optimistic value function $V_{k,h}(s)$ and pessimistic value function $\check{V}_{k,h}(s)$, we have the following lemma.
\begin{lemma}\label{lemma:transition1}
    On the events $\tilde{\cE}$, $\cE$ and $\cE_2$, the difference between the optimistic value function $V_{k,h}$ and the pessimistic value function $\check{V}_{k,h}$ is upper bounded by:
    \begin{align*}
        &\sum_{k=1}^K\sum_{h=1}^H \big[\PP_h(\vvalue_{k,h+1}-\check{\vvalue}_{k,h+1})\big](s_h^k,a_h^k)\notag\\
        &\leq 32d^4H^9\iota +40(\beta+\bar{\beta}) d^7H^6\iota+ 8H(\beta+\bar{\beta})\sqrt{2d H\iota\sum_{h=1}^H\sum_{k=1}^K(\sigma_{k,h}^2 + H)}+4\sqrt{H^5K\log(H/\delta)},
    \end{align*}
    where $\iota=\log\big(1+K/(d\lambda)\big)$.
\end{lemma}

For the  summation of variance  $\sum_{k=1}^K\sum_{h=1}^H [\VV_h \vvalue_{k,h+1}^{\pi^k}](s_h^k,a_h^k)$, we denote the following high probability events $\cE_3$:
\begin{align*}
    \cE_3=\bigg\{\sum_{k=1}^K\sum_{h=1}^H [\VV_h \vvalue_{k,h+1}^{\pi^k}](s_h^k,a_h^k)\leq 3H^2K+3H^3\log(1/\delta)\bigg\}.
\end{align*}
Then Lemma C.5 in \citet{jin2018q} shows that the probability of events $\cE_3$ is lower bounded by $\Pr(\cE_3)\ge 1-\delta$. Furthermore, on the event $\cE\cap\tilde{\cE}\cap\cE_1\cap\cE_2\cap\cE_3$, the following lemma gives an upper bound of the total estimated variance $\sum_{h=1}^H \sum_{k=1}^K\sigma_{k,h}^2$.
\begin{lemma}\label{LEMMA: TOTAL-ESTIMATE-VARIANCE}
On the event $\cE\cap\tilde{\cE}\cap\cE_1\cap\cE_2\cap\cE_3$, the total estimated variance is upper bounded by:
\begin{align}
    \sum_{k=1}^K\sum_{h=1}^H\sigma_{k,h}^{2}&\leq O\big(H^2K+d^{10.5}H^{16}\log^{1.5}(1+dKH/\delta)\big).\notag
\end{align}
\end{lemma}
With all previous lemma, we start to prove our main Theorem \ref{theorem-1}.
\begin{proof}[Proof of Theorem \ref{theorem-1}]
On the event $\cE\cap\tilde{\cE}\cap\cE_1\cap\cE_2\cap\cE_3$,, the regret is upper bounded by:
\begin{align}
\textbf{Regret}(K)&=\sum_{k=1}^K\big(\vvalue_{1}^*(s_1^k)-\vvalue_{k,1}^{\pi^k}(s_1^k)\big)\notag\\
        &\leq\sum_{k=1}^K\big(\vvalue_{k,1}(s_1^k)-\vvalue_{k,1}^{\pi^k}(s_1^k)\big)\notag\\
        &\leq 16d^4H^8\iota +40\beta d^7H^5\iota+ 8\beta\sqrt{2d H\iota\sum_{h=1}^H\sum_{k=1}^K(\sigma_{k,h}^2 + H)}+4\sqrt{H^3K\log(H/\delta)}\notag\\
        &= \tilde{O}\Big(d^7H^8+d\sqrt{H^3K\log^2(1+dKH/\delta)}\Big),\label{eq:thm}
    \end{align}
    where $\iota=\log\big(1+K/(d\lambda)\big)$, the first inequality holds due to Lemma \ref{lemma:optimistic},  the second inequality holds due to Lemma \ref{lemma:transition1} and the last inequality holds due to Lemma \ref{LEMMA: TOTAL-ESTIMATE-VARIANCE}.
    Since the event $\cE\cap\tilde{\cE}\cap\cE_1\cap\cE_2\cap\cE_3$ holds with probability at least $1-7\delta$, \eqref{eq:thm} holds. In addition, according to Lemma \ref{lemma:update-number}, the number of updates for $Q_{k,h}, \check Q_{k,h}$ is upper bounded by $O(dH \log(1+K/\lambda))$. Thus, we complete the proof of Theorem \ref{theorem-1}.
 updates for $Q_{k,h}, \check Q_{k,h}$ is upper bounded by $O(dH \log(1+K/\lambda))$.
\end{proof}

%\section{Variance-Aware Bounds on Bellman Errors}
\section{Proof of Lemmas in Section \ref{section:sketch}}\label{appendix-2}
In this section, we provide the proof of Lemmas in Section \ref{section:sketch} and we need the following lemma, which extends Lemma D.4 in \citet{jin2020provably} to weighted ridge regression.
\begin{lemma}(Lemma D.4, \citealt{jin2020provably} with weighted linear regression)\label{lemma:uni-coverge}
Let $\{x_k\}_{k=1}^{\infty}$ be a real-valued stochastic process on state space $\cS$ with corresponding filtration $ \{\mathcal{F}_k\}_{k=1}^{\infty}$. Let $\{\bphi_k\}_{k=1}^{\infty}$ be an $\RR^d$-valued stochastic process, where $\bphi_k \in \mathcal{F}_{k-1}$ and $\|\bphi_k\|_2\leq 1$. Let $\{w_k\}_{k=1}^{\infty}$ be an real-valued stochastic process where $w_k \in \mathcal{F}_{k-1}$ and $0\leq w_k\leq C$
. For any $k\ge 0$, we define $\bSigma_k= \lambda \Ib+\sum_{i=1}^kw_k^2\bphi_i\bphi_i^{\top}$.
Then with probability at least $1-\delta$, for all $k \in \NN$ and all function $V\in \mathcal{V}$ with $\max_{s}|V(x)|\leq H$, we have
\begin{align}
    \bigg\| \sum_{i=1}^k w_i^2\bphi_i \Big\{ V(x_i)-\EE\big[V(x_i)|\mathcal{F}_{i-1}\big]\Big\}\bigg\|_{\bSigma_k^{-1}}^2\leq 4C^2H^2\bigg[\frac d2\log(1+kC^2/\lambda) +\log \frac{\mathcal{N}_\epsilon}{\delta}\bigg]+{8k^2C^4\epsilon^2}/\lambda,\notag
\end{align}
where $\mathcal{N}_\epsilon$ is the $\epsilon$-covering number of the function class $\mathcal{V}$ with respect to the distance function $\text{dist}(V_1,V_2)=\max_{s}|V_1(s)-V_2(s)|$. 
\end{lemma}
\begin{proof}[Proof of Lemma \ref{lemma:uni-coverge}]
   For any function $V\in \cV$, based on the definition of $\epsilon$-covering number, there exists a function $\tilde{V}$ in the $\epsilon$-net, such that
   \begin{align}
       \text{dist}(V,\tilde{V})\leq \epsilon.\label{eq:05}
   \end{align}
   Therefore, the concentration error for the value function can be decomposed as 
   \begin{align}
        &\bigg\| \sum_{i=1}^k w_i^2\bphi_i \Big\{ V(x_i)-\EE\big[V(x_i)|\mathcal{F}_{i-1}\big]\Big\}\bigg\|_{\bSigma_k^{-1}}^2\notag\\
        &\leq   \underbrace{2\bigg\| \sum_{i=1}^k w_i^2\bphi_i \Big\{ \tilde{V}(x_i)-\EE\big[\tilde{V}(x_i)|\mathcal{F}_{i-1}\big]\Big\}\bigg\|_{\bSigma_k^{-1}}^2}_{I_1}+ \underbrace{2\bigg\| \sum_{i=1}^k w_i^2\bphi_i \Big\{ \Delta_V(x_i)-\EE\big[\Delta_V(x_i)|\mathcal{F}_{i-1}\big]\Big\}\bigg\|_{\bSigma_k^{-1}}^2}_{I_2},\label{eq:06}
   \end{align}
   where $\Delta_V=V-\tilde{V}$ and the inequality holds due to $\|\ab+\bbb \|^2_{\bSigma}\leq 2\|\ab\|^2_{\bSigma}+2\|\bbb\|^2_{\bSigma}$.
   For any fixed value function $\tilde V$, we apply Lemma \ref{lemma:hoeffding} with $\xb_i=w_i\bphi_i, \eta_t=w_i\tilde{V}(x_i)-w_i\EE[\tilde{V}(x_i)]$. According to the definition of $\xb_i,\eta_i$, we have following property 
   \begin{align*}
\|\xb_i\|_2&=w_i\bphi_i\leq C,\\
\EE[\eta_i|\mathcal{F}_i]&=0,|\eta_i|=\big|w_i\tilde{V}(x_i)-w_i\EE[\tilde{V}(x_i)]\big|\leq HC.
   \end{align*}
   Therefore, according to Lemma \ref{lemma:hoeffding}, after taking an union bound over the $\epsilon$-net of the function class $\mathcal{V}$, with probability at least $1-\delta/H$, the first term $I_1$ is upper bounded by:
   \begin{align}
       I_1&=\bigg\| 2\sum_{i=1}^k w_i^2\bphi_i \Big\{ \tilde{V}(x_i)-\EE\big[\tilde{V}(x_i)|\mathcal{F}_{i-1}\big]\Big\}\bigg\|_{\bSigma_k^{-1}}^2\notag\\
       &\leq 4C^2H^2\bigg[\frac d2\log(1+kC^2/\lambda) +\log \frac{\mathcal{N}_\epsilon}{\delta}\bigg].\label{eq:08}
   \end{align}
For the second term, it can be upper bounded by
\begin{align}
    I_2&=2\bigg\| \sum_{i=1}^k w_i^2\bphi_i \Big\{ \Delta_V(x_i)-\EE\big[\Delta_V(x_i)|\mathcal{F}_{i-1}\big]\Big\}\bigg\|_{\bSigma_k^{-1}}^2\notag\\
    &\leq 2k\sum_{i=1}^k  \bigg\|  w_i^2\bphi_i \Big\{ \Delta_V(x_i)-\EE\big[\Delta_V(x_i)|\mathcal{F}_{i-1}\big]\Big\}\bigg\|_{\bSigma_k^{-1}}^2\notag\\
    &\leq 8k^2C^4\epsilon^2/\lambda,\label{eq:09}
\end{align}
where the first inequality holds due to the Cauchy-Schwartz inequality and the last inequality holds due to the facts that $|\Delta V(x_i)|\leq \epsilon, 0\leq w_i\leq C, \|\bphi_i\|_2\leq 1, \bSigma_k \succeq \lambda \Ib$.
Substituting the results in~\eqref{eq:08} and \eqref{eq:09} into \eqref{eq:06}, we finish the proof of Lemma \ref{lemma:uni-coverge}.
\end{proof}

\subsection{Proof of Lemma \ref{lemma:hoeffding-type}}
In this subsection, we provide the proof of Lemma \ref{lemma:hoeffding-type}, which suggests a Hoeffding-type upper bound for the estimation error.

\begin{proof}[Proof of Lemma \ref{lemma:hoeffding-type}]
    Firstly, for any fixed stage $h\in[H]$ and the optimistic value function $\vvalue_{k,h+1}$, according to Lemma \ref{lemma:qvalue-linear}, there exists a vector $\wb_{k,h}$ such that $\PP_h\vvalue_{k,h+1}(s,a)$ can be represented by $\wb_{k,h}^{\top}\bphi(s,a)$ and $\|\wb_{k,h}\|_2\leq H\sqrt{d}$.
Therefore, the estimation error can be decomposed as
    \begin{align}
        &\|\hat{\wb}_{k,h}-{\wb}_{k,h}\|_{\bSigma_{k,h}}\notag\\
        &=\bigg\|\bSigma_{k,h}^{-1}\sum_{i=1}^{k-1}\bar\sigma_{i,h}^{-2}\bphi(s_h^i,a_h^i)\vvalue_{k,h+1}(s_{h+1}^{i})-\bSigma_{k,h}^{-1}\Big(\lambda \Ib+\sum_{i=1}^{k-1}\bar\sigma_{i,h}^{-2}\bphi(s_h^i,a_h^i)\bphi(s_h^i,a_h^i)^{\top}\Big){\wb}_{k,h}\bigg\|_{\bSigma_{k,h}}\notag\\
        &=\bigg\|\bSigma_{k,h}^{-1}\sum_{i=1}^{k-1}\bar\sigma_{i,h}^{-2}\bphi(s_h^i,a_h^i)\big(\vvalue_{k,h+1}(s_{h+1}^{i})-[\PP_h\vvalue_{k,h+1}](s_h^i,a_h^i)\big)-\lambda \bSigma_{k,h}^{-1}\wb_{k,h}\bigg\|_{\bSigma_{k,h}}
        \notag\\
        &\leq \underbrace{\|\lambda \bSigma_{k,h}^{-1}\wb_{k,h}\|_{\bSigma_{k,h}^{}}}_{I_1}+\underbrace{\bigg\|\bSigma_{k,h}^{-1}\sum_{i=1}^{k-1}\bar\sigma_{i,h}^{-2}\bphi(s_h^i,a_h^i)\big(\vvalue_{k,h+1}(s_{h+1}^{i})-[\PP_h\vvalue_{k,h+1}](s_h^i,a_h^i)\big)\bigg\|_{\bSigma_{k,h}}}_{I_2},\label{eq:11}
    \end{align}
    where the first inequality holds due to the fact that $\|\ab+\bbb\|_{\bSigma}\leq \|\ab\|_{\bSigma}+\|\bbb\|_{\bSigma}$.
For the first term $I_1$, since $\bSigma_{k,h}\succeq \lambda \Ib$ and $\|\wb_{k,h}\|_2\leq H\sqrt{d}$,
it is upper bounded by
\begin{align}
    I_1=\|\lambda\wb_{k,h}\|_{\bSigma_{k,h}^{-1}}\leq \sqrt{\lambda}\cdot \|\wb_{k,h}\|_2\leq H\sqrt{d\lambda}.\label{eq:12}
\end{align}
For the second term $I_2$, we apply Lemma \ref{lemma:uni-coverge} with the optimistic value function class $\mathcal{V}_h$ and $\epsilon =H\sqrt{\lambda}/K$, then for any fixed stage $h\in[H]$, with probability at least $1-\delta/H$, for all episode $k\in[K]$, we have 
\begin{align}
    I_2&=\bigg\|\bSigma_{k,h}^{-1}\sum_{i=1}^{k-1}\bar\sigma_{i,h}^{-2}\bphi(s_h^i,a_h^i)\big(\vvalue_{k,h+1}(s_{h+1}^{i})-[\PP_h\vvalue_{k,h+1}](s_h^i,a_h^i)\big)\bigg\|_{\bSigma_{k,h}}\notag\\
    &\leq \sqrt{4C^2H^2\bigg[\frac d2\log(1+kC^2/\lambda) +\log \frac{H\mathcal{N}_\epsilon}{\delta}\bigg]+{8k^2C^4\epsilon^2}/\lambda}\notag\\
    &\leq \sqrt{4H\bigg[\frac d2\log\big(1+k/(\lambda H)\big) +\log \frac{H\mathcal{N}_\epsilon}{\delta}\bigg]+{8k^2\epsilon^2}/(\lambda H^2)}\notag\\
    &\leq  \sqrt{4H\bigg[\frac d2\log\big(1+k/(\lambda H)\big) +\log \frac{H\mathcal{N}_\epsilon}{\delta}\bigg]+8}\notag\\
    &=O\bigg(\sqrt{d^3H^2\log^2\big(dHK/(\delta\lambda)\big)}\bigg),\label{eq:13}
    \end{align}
where the first inequality holds due to Lemma \ref{lemma:uni-coverge}, the second inequality holds due to $0\leq \bar\sigma_{i,h}^{-1}\leq 1/\sqrt{H}$, the third inequality holds due to Lemma \ref{lemma:covering-number} and $\epsilon=H\sqrt{\lambda}/K$. Substituting \eqref{eq:12} and \eqref{eq:13} into \eqref{eq:11}, we have
\begin{align}
    \|\hat{\wb}_{k,h}-{\wb}_{k,h}\|_{\bSigma_{k,h}}\leq I_1+I_2= O\Big(H\sqrt{d\lambda} +\sqrt{d^3H^2\log^2\big(dHK/(\delta\lambda)\big)}\Big) =\bar{\beta}.\label{eq:14}
\end{align}
Therefore, the estimation error is upper bounded by 
\begin{align}
   \big|\hat{\wb}_{k,h}^{\top}\bphi(s,a)-[\PP_h\vvalue_{k,h+1}](s,a)\big|&=|\hat{\wb}_{k,h}^{\top}\bphi(s,a)-{\wb}_{k,h}^{\top}\bphi(s,a)|\notag\\
   &\leq \|\hat{\wb}_{k,h}-{\wb}_{k,h}\|_{\bSigma_{k,h}} \cdot \|\bphi(s,a)\|_{\bSigma^{-1}_{k,h}}\notag\\
   &\leq \bar{\beta} \sqrt{\bphi(s,a)^{\top}\bSigma_{k,h}^{-1}\bphi(s,a)},\label{eq:15}
\end{align}
where the first inequality holds due to Cauchy-Schwartz inequality and the last inequality holds due to \eqref{eq:14}.
Replacing the value function class by the pessimistic value function class $\check{\mathcal{V}}_h$ (or squared value function class $\mathcal{V}^2_h$) and following the same proof of \eqref{eq:15}, we can derive the following upper bound for the estimation errors:
\begin{align*}
    \big|\tilde{\wb}_{k,h}^{\top}\bphi(s,a)-[\PP_h\vvalue^2_{k,h+1}](s,a)\big| \leq \tilde{\beta} \sqrt{\bphi(s,a)^{\top}\bSigma_{k,h}^{-1}\bphi(s,a)},\notag\\
    \big|\check{\wb}_{k,h}^{\top}\bphi(s,a)-[\PP_h\check{\vvalue}_{k,h+1}](s,a)\big| \leq \bar{\beta} \sqrt{\bphi(s,a)^{\top}\bSigma_{k,h}^{-1}\bphi(s,a)},
\end{align*}
where $\tilde{\beta}= O \Big(H\sqrt{d\lambda} +\sqrt{d^3H^4\log^2\big(dHK/(\delta\lambda)\big)}\Big)$ and $\bar{\beta}=O\Big(H\sqrt{d\lambda} +\sqrt{d^3H^2\log^2\big(dHK/(\delta\lambda)\big)}\Big)$
Thus, we finish the proof of Lemma \ref{lemma:hoeffding-type}.    
\end{proof}

\subsection{Proof of Lemma \ref{lemma:varaince}}
In this subsection, we provide the proof of Lemma \ref{lemma:varaince} for the variance estimator.

\begin{proof}[Proof of Lemma \ref{lemma:varaince}]
    Firstly, according to Lemma \ref{lemma:hoeffding-type}, we have
    \begin{align}
        &\big|[\bar\VV_h\vvalue_{k,h+1}](s_h^k,a_h^k)         -[\VV_h\vvalue_{k,h+1}](s_h^k,a_h^k)\big| \notag\\
        &=\Big| \big[\tilde{\wb}^{\top}_{k,h}\bphi(s_h^k,a_h^k)\big]_{[0,H^2]}-\big[\hat{\wb}^{\top}_{k,h}\bphi(s_h^k,a_h^k)\big]_{[0,H]}^2- [\PP_h\vvalue_{k,h+1}^2](s_h^k,a_h^k)+\big([\PP_h\vvalue_{k,h+1}](s_h^k,a_h^k)\big)^2\Big|\notag\\
        &\leq \Big|\big[\tilde{\wb}^{\top}_{k,h}\bphi(s_h^k,a_h^k)\big]_{[0,H^2]}-[\PP_h\vvalue_{k,h+1}^2](s_h^k,a_h^k)\Big|+\Big|\big[\hat{\wb}^{\top}_{k,h}\bphi(s_h^k,a_h^k)\big]_{[0,H]}^2-\big([\PP_h\vvalue_{k,h+1}](s_h^k,a_h^k)\big)^2\Big|\notag\\
        &=\Big|\big[\tilde{\wb}^{\top}_{k,h}\bphi(s_h^k,a_h^k)\big]_{[0,H^2]}-[\PP_h\vvalue_{k,h+1}^2](s_h^k,a_h^k)\Big| \notag\\
        &\qquad +\Big|\big[\hat{\wb}^{\top}_{k,h}\bphi(s_h^k,a_h^k)\big]_{[0,H]}+[\PP_h\vvalue_{k,h+1}](s_h^k,a_h^k)\Big|\cdot \Big|\big[\hat{\wb}^{\top}_{k,h}\bphi(s_h^k,a_h^k)\big]_{[0,H]}-[\PP_h\vvalue_{k,h+1}](s_h^k,a_h^k)\Big|\notag\\
        &\leq \min \Big\{\tilde{\beta}_k\big\|\bSigma_{k,h}^{-1/2}\bphi(s_h^k,a_h^k)\big\|_2,H^2\Big\}+\min \Big\{2H\bar{\beta}_k\big\|\bSigma_{k,h}^{-1/2}\bphi(s_h^k,a_h^k)\big\|_2,H^2\Big\}\notag\\
        &=E_{k,h},\label{eq:20}
    \end{align}
    where the first inequality holds due to $|a+b|\leq |a|+|b|$ and the last inequality holds due to Lemma~\ref{lemma:hoeffding-type} with the fact that $0\leq \big[\hat{\wb}^{\top}_{k,h}\bphi(s_h^k,a_h^k)\big]_{[0,H]}+[\PP_h\vvalue_{k,h+1}](s_h^k,a_h^k)\leq 2H$.
In addition, for the variance $\big[\VV_h\vvalue_{h+1}^*\big](s_h^k,a_h^k)$, we have
\begin{align}
     &\big|[\VV_h\vvalue_{k,h+1}](s_h^k,a_h^k)         -[\VV_h\vvalue_{h+1}^*](s_h^k,a_h^k)\big|\notag\\
     &=\Big|  [\PP_h\vvalue_{k,h+1}^2](s_h^k,a_h^k)-\big([\PP_h\vvalue_{k,h+1}](s_h^k,a_h^k)\big)^2- [\PP_h(\vvalue_{h+1}^*)^2](s_h^k,a_h^k)+\big([\PP_h\vvalue_{h+1}^*](s_h^k,a_h^k)\big)^2\Big|\notag\\
     &\leq \big|[\PP_h\vvalue_{k,h+1}^2](s_h^k,a_h^k)-[\PP_h(\vvalue_{h+1}^*)^2](s_h^k,a_h^k)
     \big|+\Big|\big([\PP_h\vvalue_{k,h+1}](s_h^k,a_h^k)\big)^2-\big([\PP_h\vvalue_{h+1}^*](s_h^k,a_h^k)\big)^2\Big|\notag\\
     &= \big|[\PP_h(\vvalue_{k,h+1}-\vvalue_{h+1}^*)(\vvalue_{k,h+1}+\vvalue_{h+1}^*)](s_h^k,a_h^k)
     \big|\notag\\
     &\qquad+\Big|\big([\PP_h\vvalue_{k,h+1}](s_h^k,a_h^k)-[\PP_h\vvalue_{h+1}^*](s_h^k,a_h^k)\big)\cdot \big([\PP_h\vvalue_{k,h+1}](s_h^k,a_h^k)+[\PP_h\vvalue_{h+1}^*](s_h^k,a_h^k)\big)\Big|\notag\\
     &\leq 4H \big([\PP_h\vvalue_{k,h+1}](s_h^k,a_h^k)-[\PP_h\vvalue_{h+1}^*](s_h^k,a_h^k)\big),\label{eq:21}
\end{align}
where the first inequality holds due to $|a+b|\leq |a|+|b|$ and the last inequality holds due to Lemma~\ref{lemma:optimistic} $(V_{k,h+1}(s')\ge V_{h+1}^*(s'))$ with the fact that $0\leq V_{h+1}^*(s'),V_{k,h+1}(s')\leq H$. Based on the event $\cE$ and $\tilde{\cE}_{h+1}$, \eqref{eq:21} can be further bounded by
\begin{align}
    &\big([\PP_h\vvalue_{k,h+1}](s_h^k,a_h^k)-[\PP_h\vvalue_{h+1}^*](s_h^k,a_h^k)\big)\notag\\
    &\leq \big([\PP_h\vvalue_{k,h+1}](s_h^k,a_h^k)-[\PP_h\check{\vvalue}_{k,h+1}](s_h^k,a_h^k)\big)\notag\\
    &\leq \hat{\wb}_{k,h}^{\top}\bphi(s,a)+\bar{\beta}\sqrt{\bphi(s,a)^{\top}\bSigma_{k,h}^{-1}\bphi(s,a)}-\check{\wb}_{k,h}^{\top}\bphi(s,a)+\bar{\beta}\sqrt{\bphi(s,a)^{\top}\bSigma_{k,h}^{-1}\bphi(s,a)},\label{eq:22}
\end{align}
where the first inequality holds due to Lemma \ref{lemma:optimistic} ($V_{h+1}^*(s')\ge \check{V}_{k,h+1}(s')$) and the last inequality holds due to the definition of events $\cE$. Combining the results in \eqref{eq:20}, \eqref{eq:21} and \eqref{eq:20}, we have
\begin{align*}
    &\big|[\bar\VV_h\vvalue_{k,h+1}](s_h^k,a_h^k)         -[\VV_h\vvalue_{h+1}^*](s_h^k,a_h^k)\big|\notag\\
    &\leq \big|[\bar\VV_h\vvalue_{k,h+1}](s_h^k,a_h^k)         -[\VV_h\vvalue_{k,h+1}](s_h^k,a_h^k)\big|+\big|[\VV_h\vvalue_{k,h+1}](s_h^k,a_h^k)         -[\VV_h\vvalue_{h+1}^*](s_h^k,a_h^k)\big|\notag\\
    &\leq E_{k,h}+4H\Big(\hat{\wb}_{k,h}^{\top}\bphi(s,a)+\bar{\beta}\sqrt{\bphi(s,a)^{\top}\bSigma_{k,h}^{-1}\bphi(s,a)}-\check{\wb}_{k,h}^{\top}\bphi(s,a)+\bar{\beta}\sqrt{\bphi(s,a)^{\top}\bSigma_{k,h}^{-1}\bphi(s,a)}\Big)\notag.
\end{align*}
In addition, since the value functions $V_{k,h}(s)$ and $V_{h+1}^*(s)$ is upper bounded by $H$, we have $\big|[\VV_h\vvalue_{k,h+1}](s_h^k,a_h^k)         -[\VV_h\vvalue_{h+1}^*](s_h^k,a_h^k)\big|$ which implies that
\begin{align}
        &\big|[\bar\VV_h\vvalue_{k,h+1}](s_h^k,a_h^k)         -[\VV_h\vvalue_{h+1}^*](s_h^k,a_h^k)\big|\notag\\
    &\leq \big|[\bar\VV_h\vvalue_{k,h+1}](s_h^k,a_h^k)         -[\VV_h\vvalue_{k,h+1}](s_h^k,a_h^k)\big|+\big|[\VV_h\vvalue_{k,h+1}](s_h^k,a_h^k)         -[\VV_h\vvalue_{h+1}^*](s_h^k,a_h^k)\big|\notag\\
    &\leq E_{k,h}+H^2.\notag
\end{align}
Thus, we finish the proof of Lemma \ref{lemma:varaince}.
\end{proof}

\subsection{Proof of Lemma \ref{lemma:variance-dec}}
In this subsection, we provide the proof of Lemma \ref{lemma:variance-dec} for the variance estimator. 
\begin{proof}[Proof of Lemma \ref{lemma:variance-dec}]
%(where monotonicity holds)

    On the event $\cE$ and $\tilde{\cE}_{h+1}$, we have
    \begin{align}
        &\big[\VV_h(\vvalue_{i,h+1}-\vvalue_{h+1}^*)\big](s_h^k,a_h^k)\notag\\
        &\leq [\PP_h(\vvalue_{i,h+1}-\vvalue_{h+1}^*)^2](s_h^k,a_h^k)\notag\\
        &\leq 2H\big[\PP_h(\vvalue_{i,h+1}-\vvalue_{h+1}^*)\big](s_h^k,a_h^k)\notag\\
        &\leq 2H\big([\PP_h\vvalue_{i,h+1}](s_h^k,a_h^k)-[\PP_h\check{\vvalue}_{k,h+1}](s_h^k,a_h^k)\big)\notag\\
        &\leq 2H\big([\PP_h\vvalue_{k,h+1}](s_h^k,a_h^k)-[\PP_h\check{\vvalue}_{k,h+1}](s_h^k,a_h^k)\big)\notag\\
    &\leq 2H \Big(\hat{\wb}_{k,h}^{\top}\bphi(s_h^k,a_h^k)+\bar{\beta}\sqrt{\bphi(s_h^k,a_h^k)^{\top}\bSigma_{k,h}^{-1}\bphi(s_h^k,a_h^k)}-\check{\wb}_{k,h}^{\top}\bphi(s_h^k,a_h^k)+\bar{\beta}\sqrt{\bphi(s_h^k,a_h^k)^{\top}\bSigma_{k,h}^{-1}\bphi(s_h^k,a_h^k)}\Big),\notag
    \end{align}
    where the first inequality holds due to $\Var(x)\leq \EE[x^2]$, the second and third inequality holds due to Lemma \ref{lemma:optimistic} with the fact that $0\leq V_{i,h+1}(s'),\vvalue_{h+1}^*(s')\leq H$, the fourth inequality the fact $\vvalue_{k,h+1}\ge \vvalue_{i,h+1}$ from the update-rule in Algorithm \ref{algorithm1} and the fifth inequality holds due to Lemma \ref{lemma:hoeffding-type}. On the other hand, since value function $0\leq V_{i,h+1}(s'),\vvalue_{h+1}^*(s')\leq H$, we have
    \begin{align*}
        \big[\VV_h(\vvalue_{i,h+1}-\vvalue_{h+1}^*)\big](s_h^k,a_h^k)\leq H^2 =(d^3H^3)/(d^3H).
    \end{align*}
    Thus, we finish the proof of Lemma \ref{lemma:variance-dec}.
\end{proof}
\subsection{Proof of Lemma \ref{lemma:optimistic}}
In this subsection, we provide proof of optimistic property.

\begin{proof}[Proof of Lemma \ref{lemma:optimistic}]
    We prove this lemma by induction. First, we prove the base case for the last stage $H+1$. Under this situation, for all state $s\in \cS$ and action $a\in \cA$, we have $Q_{k,H+1}(s,a)=\qvalue_{h}^*(s,a)=\check{\qvalue}_{k,h}(s,a)=0$ and $\vvalue_{k,h}(s)\ge \vvalue_{h}^*(s) \ge \check{\vvalue}_{k,h}(s)=0$. Thus, the results in Lemma \ref{lemma:optimistic} holds for stage $H+1$.
    
    Now, we focus on stage $h+1$.
    Since events $\tilde{\cE}_h$ directly implies the events $\tilde{\cE}_{h+1}$, according to the reduction assumption, we have
    \begin{align}
         \vvalue_{k,h+1}(s)\ge \vvalue_{h+1}^*(s) \ge \check{\vvalue}_{k,h}(s).\label{eq:001}
    \end{align}
    Thus, for all episode $k\in[K]$, we have
    \begin{align}
\reward_h(s,a)+\hat{\wb}_{k,h}^{\top}\bphi(s,a)+\beta\sqrt{\bphi(s,a)^{\top}\bSigma_{k,h}^{-1}\bphi(s,a)}-Q_h^*(s,a)\ge \big[\PP_h (V_{k,h+1}-\vvalue_{h+1}^*)\big](s,a)\ge 0,\notag
    \end{align}
where the first inequality holds due to the definition of events $\tilde{\cE}_h$ and the second inequality holds due to \eqref{eq:001}.
    Furthermore, the optimal value function is upper bounded by $\qvalue_{h}^{*}(s,a)\leq H$ and it implies that
    \begin{align}
        Q_h^*(s,a)\leq \min\Big\{\min_{1\leq i\leq k} \reward_h(s,a)+\hat{\wb}_{i,h}^{\top}\bphi(s,a)+\beta\sqrt{\bphi(s,a)^{\top}\bSigma_{i,h}^{-1}\bphi(s,a)} ,H\Big\}\leq Q_{k,h}(s,a).\label{eq:002}
    \end{align}
    With a similar argument, for the pessimistic action-value function $\check{Q}_{k,h}(s,a)$, we have
    \begin{align}
\reward_h(s,a)+\check{\wb}_{k,h}^{\top}\bphi(s,a)-\bar{\beta}\sqrt{\bphi(s,a)^{\top}\bSigma_{k,h}^{-1}\bphi(s,a)}-Q_h^*(s,a)\leq \big[\PP_h (\check{V}_{k,h+1}-\vvalue_{h+1}^*)\big](s,a)\leq 0.\notag 
    \end{align}
    Since the optimal value function is lower bounded by $\qvalue_{h}^{*}(s,a)\ge 0$, the result further implies that
    \begin{align}
        Q_h^*(s,a)\ge \max\Big\{ \max_{1\leq i\leq k}\reward_h(s,a)+\hat{\wb}_{k_{\text{last}},h}^{\top}\bphi(s,a)+\beta\sqrt{\bphi(s,a)^{\top}\bSigma_{k_{\text{last}},h}^{-1}\bphi(s,a)} ,0\Big\}\ge\check{Q}_{k,h}(s,a).\label{eq:003}
    \end{align}
In addition, for the value function $V$, we have
\begin{align*}
    V_{k,h}(s)&=\max_a Q_{i,h}(s,a)\ge \min_{1\leq i\leq k}  \max_a Q_h^*(s,a)=V_h^*(s),\notag\\
    \check{V}_{k,h}(s)&=  \max_a Q_{i,h}(s,a)\leq \max_{1\leq i\leq k}  \max_a Q_h^*(s,a)=V_h^*(s),\notag\\
\end{align*}
where the first inequality holds due to \eqref{eq:002} and the second inequality holds due to \eqref{eq:003}. Thus, by induction, we finish the proof of Lemma \ref{lemma:optimistic}.
\end{proof}

\subsection{Proof of Lemma \ref{lemma:002}}
In this subsection, we provide the proof of Lemma \ref{lemma:002}, which suggests a Bernstein-type upper bound for the estimation error.
\begin{proof}[Proof of Lemma \ref{lemma:002}]
We prove Lemma \ref{lemma:002} by induction.
First, we prove the base case for the last stage $H$. Under this situation, the weight vector $\hat{\wb}_{k,h}=0$ and $\vvalue_{k,h+1}(s,a)=0$. Thus, the result in Lemma \ref{lemma:002} holds for stage $H$.  

For stage $h\in[H]$ and $k\in[K]$, according to Lemma \ref{lemma:qvalue-linear}, there exists a vector $\wb_{k,h}$ such that $\PP_h\vvalue_{k,h+1}(s,a)$ can be represented by $\wb_{k,h}^{\top}\bphi(s,a)$ and $\|\wb_{k,h}\|_2\leq H\sqrt{d}$.  Conditioned on the event $\tilde{\cE}_{h+1}$, the estimation error can be decomposed as
    \begin{align}
        &\|\hat{\wb}_{k,h}-{\wb}_{k,h}\|_{\bSigma_{k,h}}\notag\\
        &=\bigg\|\bSigma_{k,h}^{-1}\sum_{i=1}^{k-1}\bar\sigma_{i,h}^{-2}\bphi(s_h^i,a_h^i)\vvalue_{k,h+1}(s_{h+1}^{i})-\bSigma_{k,h}^{-1}\Big(\lambda \Ib+\sum_{i=1}^{k-1}\bar\sigma_{i,h}^{-2}\bphi(s_h^i,a_h^i)\bphi(s_h^i,a_h^i)^{\top}\Big){\wb}_{k,h}\bigg\|_{\bSigma_{k,h}}\notag\\
        &=\bigg\|\bSigma_{k,h}^{-1}\sum_{i=1}^{k-1}\bar\sigma_{i,h}^{-2}\bphi(s_h^i,a_h^i)\big(\vvalue_{k,h+1}(s_{h+1}^{i})-[\PP_h\vvalue_{k,h+1}](s_h^i,a_h^i)\big)-\lambda \bSigma_{k,h}^{-1}\wb_{k,h}\bigg\|_{\bSigma_{k,h}}
        \notag\\
        &\leq \underbrace{\|\lambda \bSigma_{k,h}^{-1}\wb_{k,h}\|_{\bSigma_{k,h}^{}}}_{I_1}+\underbrace{\bigg\|\bSigma_{k,h}^{-1}\sum_{i=1}^{k-1}\bar\sigma_{i,h}^{-2}\bphi(s_h^i,a_h^i)\big(\vvalue_{k,h+1}(s_{h+1}^{i})-[\PP_h\vvalue_{k,h+1}](s_h^i,a_h^i)\big)\bigg\|_{\bSigma_{k,h}}}_{I_2},\label{eq:0001}
    \end{align}
    where the first inequality holds due to the fact that $\|\ab+\bbb\|_{\bSigma}\leq \|\ab\|_{\bSigma}+\|\bbb\|_{\bSigma}$.
For the first term $I_1$, since $\bSigma_{k,h}\succeq \lambda \Ib$ and $\|\wb_{k,h}\|_2\leq H\sqrt{d}$,
it is upper bounded by
\begin{align}
    I_1=\|\lambda\wb_{k,h}\|_{\bSigma_{k,h}^{-1}}\leq \sqrt{\lambda}\cdot \|\wb_{k,h}\|_2\leq H\sqrt{d\lambda}.\label{eq:0002}
\end{align}
For the second term $I_2$, we have
\begin{align}
    I_2&=\bigg\|\bSigma_{k,h}^{-1}\sum_{i=1}^{k-1}\bar\sigma_{i,h}^{-2}\bphi(s_h^i,a_h^i)\big(\vvalue_{k,h+1}(s_{h+1}^{i})-[\PP_h\vvalue_{k,h+1}](s_h^i,a_h^i)\big)\bigg\|_{\bSigma_{k,h}}\notag\\
    &=\bigg\|\sum_{i=1}^{k-1}\bar\sigma_{i,h}^{-2}\bphi(s_h^i,a_h^i)\big(\vvalue_{k,h+1}(s_{h+1}^{i})-[\PP_h\vvalue_{k,h+1}](s_h^i,a_h^i)\big)\bigg\|_{\bSigma_{k,h}^{-1}}\notag\\
    &\leq \underbrace{\bigg\|\sum_{i=1}^{k-1}\bar\sigma_{i,h}^{-2}\bphi(s_h^i,a_h^i)\big(\vvalue^*_{h+1}(s_{h+1}^{i})-[\PP_h\vvalue^*_{h+1}](s_h^i,a_h^i)\big)\bigg\|_{\bSigma_{k,h}^{-1}}}_{J_1}\notag\\
    &\qquad +\underbrace{\bigg\|\sum_{i=1}^{k-1}\bar\sigma_{i,h}^{-2}\bphi(s_h^i,a_h^i)\Big(\Delta{\vvalue}_{k,h+1}(s_{h+1}^{i})-\big[\PP_h(\Delta{\vvalue}_{k,h+1})\big](s_h^i,a_h^i)\Big)\bigg\|_{\bSigma_{k,h}^{-1}}}_{J_2},\label{eq:0003}
\end{align}
where $\Delta{\vvalue}_{k,h+1}=V_{k,h+1}-V_{h+1}^*$.

For the term $J_1$, we apply Lemma \ref{lemma:concentration_variance} with $\xb_i=\bar{\sigma}^{-1}_{i,h}\bphi(s_h^i,a_h^i) $ and $\eta_i=\ind\big\{[\VV_{h}\vvalue_{h+1}^*](s_h^i,a_h^i)\leq \bar{\sigma}_{i,h}^2\big\}\cdot \bar{\sigma}^{-1}_{i,h}\big(\vvalue^*_{h+1}(s_{h+1}^{i})-[\PP_h\vvalue^*_{h+1}](s_h^i,a_h^i)\big)$. For $\xb_t,\eta_t$, we have the following property:
\begin{align}
    \|\xb_i\|_2&=\big\|\bar{\sigma}^{-1}_{i,h}\bphi(s_h^i,a_h^i)\big\|_2\leq \big\|\bphi(s_h^i,a_h^i)\big\|_2/\sqrt{H}\leq 1/\sqrt{H},\notag\\
    \EE[\eta_i|\mathcal{F}_i]&=0,|\eta_t|\leq\Big| \bar{\sigma}^{-1}_{i,h}\big(\vvalue^*_{h+1}(s_{h+1}^{i})-[\PP_h\vvalue^*_{h+1}](s_h^i,a_h^i)\big)\Big|\leq \sqrt{H},\notag\\
    \EE[\eta^2_i|\mathcal{F}_i]&=\EE\Big[\ind\big\{[\VV_{h}\vvalue_{h+1}^*](s_h^i,a_h^i)\leq \bar{\sigma}_{i,h}^2\big\}\cdot\bar{\sigma}^{-2}_{i,h}[\VV_h \vvalue_{h+1}^*](s_h^i,a_h^i)\Big]\leq 1,\notag\\
    \max_{i}&\big\{|\eta_i|\cdot \min\{1,\|\xb_i\|_{\bSigma_{i,h}^{-1}}\}\big\}\leq 2H\bar{\sigma}_{i,h}^{-1}\|\xb_i\|_{\bSigma_{i,h}^{-1}}\leq \sqrt{d}.\notag
\end{align}
Thus, with probability at least $1-\delta/H$, for all $k\in[K]$, we have
\begin{align*}
    \bigg\|\sum_{i=1}^{k-1}\xb_i\eta_i\bigg\|_{\bSigma_{k,h}^{-1}}\leq O\Big(\sqrt{d \log^2\big(1+dKH/(\delta\lambda)\big)}\Big).
\end{align*}
In addition, on the event $\tilde{\cE}_{h+1}$ and $\cE$, according to Lemma \ref{lemma:varaince}, we have 
\begin{align*}
    \bar{\sigma}_{k,h}^2\ge [\bar{\VV}_{k,h}\vvalue_{k,h+1}](s_h^k,a_h^k)+E_{k,h}+D_{k,h} \ge [\VV_h\vvalue_{h+1}^*](s_h^k,a_h^k),
\end{align*}
which further implies that
\begin{align}
    J_1&=\bigg\|\sum_{i=1}^{k-1}\bar\sigma_{i,h}^{-2}\bphi(s_h^i,a_h^i)\big(\vvalue^*_{h+1}(s_{h+1}^{i})-[\PP_h\vvalue^*_{h+1}](s_h^i,a_h^i)\big)\bigg\|_{\bSigma_{k,h}^{-1}}\notag\\
    &= \bigg\|\sum_{i=1}^{k-1}\xb_i\eta_i\bigg\|_{\bSigma_{k,h}^{-1}}\notag\\
    &\leq O\Big(\sqrt{d \log^2\big(1+dKH/(\delta\lambda)\big)}\Big).\label{eq:0004}
\end{align}

For the term $J_2$, we can not directly use Lemma \ref{lemma:concentration_variance}, Since the stochastic noise $\big(\Delta{\vvalue}_{k,h+1}(s_{h+1}^{i})-[\PP_h(\Delta{\vvalue}_{k,h+1})](s_h^i,a_h^i)\big)$ is not $\mathcal{F}_{i+1}$ measurable. Thus, we need to use the $\epsilon$-net covering argument. In detail, for each episode, $i$, the value function $\vvalue_{i,h}$ belongs to the optimistic value function class $\mathcal{V}$. If we set $\epsilon=\sqrt{\lambda}/(4H^2d^2K)$, then according to Lemma \ref{lemma:covering-number}, the covering entropy for function class  $\mathcal{V}-\vvalue_{h+1}^*$ is upper bounded by
\begin{align}
    \log \mathcal{N}_\epsilon \leq O\big(d^3H^2\log^2(dHK/\lambda)\big).\label{eq:0005}
\end{align}
Then for function $\vvalue_{k,h}$, there must exists a function $\tilde{V}$ in the $\epsilon$-net, such that
   \begin{align}
       \text{dist}(\Delta \vvalue_{k,h},\tilde{V})\leq \epsilon.\label{eq:0006}
   \end{align}
   Therefore, the variance of function $\tilde{V}$ is upper bounded by
   \begin{align}
       &[\VV_h{\tilde{V}}](s_h^i,a_h^i)-\big[\VV_h{(\Delta \vvalue_{k,h+1})}\big](s_h^i,a_h^i)\notag\\
       &=[\PP_h{\tilde{V}^2}](s_h^i,a_h^i)-\big[\PP_h {(\Delta \vvalue_{k,h+1})^2}\big](s_h^i,a_h^i)+\Big(\big[\PP_h{(\Delta \vvalue_{k,h+1})\big]}(s_h^i,a_h^i)\Big)^2-\big(\PP_h{\tilde{V}}(s_h^i,a_h^i)\big)^2\notag\\
       &\leq 2\text{dist}(\Delta \vvalue_{k,h},\tilde{V}) \cdot \max_{s'}|\Delta \vvalue_{k,h+1}+\tilde{V}|(s')
       \notag\\
       &\leq 4H\cdot \text{dist}(\Delta \vvalue_{k,h},\tilde{V})\notag\\
       &\leq 1/d^2,
   \end{align}
   where the first inequality holds due to the definition of distance between different functions, the third inequality holds since
   $|\Delta V_{k,h+1}(s')+\tilde{V}(s')|\leq 2H$ and the last inequality holds due to the definition of $\epsilon$-net.
Thus, we  
apply Lemma \ref{lemma:concentration_variance} with $\xb_i=\bar{\sigma}^{-1}_{i,h}\bphi(s_h^i,a_h^i) $ and $\eta_i=\ind\big\{[\VV_{h}\tilde{\vvalue}](s_h^i,a_h^i)\leq \bar{\sigma}_{i,h}^2/(d^3H)\big\}\cdot \bar{\sigma}^{-1}_{i,h}\big(\tilde{\vvalue}(s_{h+1}^{i})-[\PP_h\tilde{\vvalue}](s_h^i,a_h^i)\big)$. Therefore, For $\xb_t,\eta_t$, we have the following property:
\begin{align}
\|\xb_i\|_2&=\big\|\bar{\sigma}^{-1}_{i,h}\bphi(s_h^i,a_h^i)\big\|_2\leq \big\|\bphi(s_h^i,a_h^i)\big\|_2/\sqrt{H}\leq 1/\sqrt{H},\notag\\
    \EE[\eta_i|\mathcal{F}_i]&=0,|\eta_t|\leq\Big| \bar{\sigma}^{-1}_{i,h}\big(\vvalue^*_{h+1}(s_{h+1}^{i})-[\PP_h\tilde{\vvalue}_{h+1}](s_h^i,a_h^i)\big)\Big|\leq \sqrt{H},\notag\\
    \EE[\eta^2_i|\mathcal{F}_i]&=\EE\Big[\ind\big\{[\VV_{h}\tilde{\vvalue}](s_h^i,a_h^i)\leq \bar{\sigma}_{i,h}^2/(d^3H)\big\}\cdot\bar{\sigma}^{-2}_{i,h}[\VV_h \tilde{\vvalue}](s_h^i,a_h^i)\Big]\leq 1/(d^3H),\notag\\
    \max_{i}&\big\{|\eta_i|\cdot \min\{1,\|\xb_i\|_{\bSigma_{i,h}^{-1}}\}\big\}\leq 2H\bar{\sigma}_{i,h}^{-1}\|\xb_i\|_{\bSigma_{i,h}^{-1}}\leq 1/(d^3H).\notag
\end{align}
After taking a union bound over the $\epsilon$-net, with probability at least $1-\delta$,  we have
\begin{align}
    \bigg\|\sum_{i=1}^{k-1}\xb_i\eta_i\bigg\|_{\bSigma_{k,h}^{-1}}\leq O\Big(\sqrt{d \log^2\big(1+dKH/(\delta\lambda)\big)}\Big).\label{eq:0008}
\end{align}
In addition, on the event $\tilde{\cE}_{h+1}$ and $\cE$, according to Lemmas \ref{lemma:varaince} and \ref{lemma:variance-dec}, we have 
\begin{align*}
    \bar{\sigma}_{i,h}^2&\ge [\bar{\VV}_{i,h}\vvalue_{i,h+1}](s_h^k,a_h^k)+E_{i,h}+D_{i,h}+H\notag\\
    &\ge D_{i,h}+H \notag\\
    &\ge d^3H  \big[\VV_h{(\Delta \vvalue_{k,h+1})}\big](s_h^i,a_h^i)+H\\
    &\ge d^3H[\VV_h{\tilde{V}}](s_h^i,a_h^i),
\end{align*}
For simplicity, we denote $\bar{\vvalue}=\Delta V_{k,h+1}-\tilde{V}$ and it further implies that
\begin{align}
    J_2&=\bigg\|\sum_{i=1}^{k-1}\bar\sigma_{i,h}^{-2}\bphi(s_h^i,a_h^i)\big(\Delta{\vvalue}(s_{h+1}^{i})-[\PP_h(\Delta{\vvalue}_{k,h+1})](s_h^i,a_h^i)\big)\bigg\|_{\bSigma_{k,h}^{-1}}\notag\\
    &\leq 2\bigg\|\sum_{i=1}^{k-1}\bar\sigma_{i,h}^{-2}\bphi(s_h^i,a_h^i)\big(\tilde{V}(s_{h+1}^{i})-[\PP_h\tilde{V}](s_h^i,a_h^i)\big)\bigg\|_{\bSigma_{k,h}^{-1}}\notag\\
    &\qquad + 2\bigg\|\sum_{i=1}^{k-1}\bar\sigma_{i,h}^{-2}\bphi(s_h^i,a_h^i)\big(\bar{\vvalue}(s_{h+1}^{i})-[\PP_h\bar{\vvalue}](s_h^i,a_h^i)\big)\bigg\|_{\bSigma_{k,h}^{-1}}\notag\\
    &\leq 2\bigg\|\sum_{i=1}^{k-1}\xb_i\eta_i\bigg\|_{\bSigma_{k,h}^{-1}}+8\epsilon^2k^2/\lambda\notag\\
    &\leq O\Big(\sqrt{d \log^2\big(1+dKH/(\delta\lambda)\big)}\Big).\label{eq:0009}
\end{align}
where the first inequality holds due to $\|\ab+\bbb \|^2_{\bSigma}\leq 2\|\ab\|^2_{\bSigma}+2\|\bbb\|^2_{\bSigma}$, the second inequality holds due to the fact that $|\bar{V}(s')|\leq \epsilon,\|\bphi(s,a)\|_2\leq 1, \bSigma_{k,h}\succeq \lambda \Ib,\bar{\bSigma}_{k,h}^{-1}\leq 1$ and the last inequality holds due to~\eqref{eq:0008} with $\epsilon=\sqrt{\lambda}/(4H^2d^2K)$.
   Substituting the results in \eqref{eq:0002}, \eqref{eq:0003}, \eqref{eq:0004} and \eqref{eq:0008} into \eqref{eq:0001}, we obtain
   \begin{align}
       \|\hat{\wb}_{k,h}-{\wb}_{k,h}\|_{\bSigma_{k,h}}\leq I_1+J_1+J_2\leq O\Big(H\sqrt{d\lambda}+\sqrt{d \log^2\big(1+dKH/(\delta\lambda)\big)}\Big)=\beta,\label{eq:0010}
   \end{align}
Therefore, the estimation error is upper bounded by 
\begin{align}
   |\hat{\wb}_{k,h}^{\top}\bphi(s,a)-\PP_h\vvalue_{k,h+1}(s,a)|&=|\hat{\wb}_{k,h}^{\top}\bphi(s,a)-{\wb}_{k,h}^{\top}\bphi(s,a)|\notag\\
   &\leq \|\hat{\wb}_{k,h}-{\wb}_{k,h}\|_{\bSigma_{k,h}} \cdot \|\bphi(s,a)\|_{\bSigma^{-1}_{k,h}}\notag\\
   &\leq {\beta} \sqrt{\bphi(s,a)^{\top}(\bSigma_{k,h})^{-1}\bphi(s,a)},\notag\end{align}
where the first inequality holds due to Cauchy-Schwartz inequality and the last inequality holds due to \eqref{eq:0010}, which implies the results in Lemma \ref{lemma:002} holds for stage $h$. Therefore, by induction, we finish the proof of Lemma \ref{lemma:002}.
\end{proof}

\section{Proof of Lemmas in Appendix \ref{appendix-1} }
In this section, we provide the proof of Lemmas in Appendix \ref{appendix-1} and we need the following auxiliary Lemma, which is modified from Lemma 4.4 in \citet{zhou2022computationally}
\begin{lemma}\label{lemma:sum-bonus}
    For any parameters $\beta'\ge 1$ and $C\ge 1$, the summation of bonuses is upper bounded by
    \begin{align*}
\sum_{k=1}^K\min\Big(\beta'\sqrt{\bphi(s_h^k,a_h^k)^{\top}\bSigma_{k,h}^{-1}\bphi(s_h^k,a_h^k)},C\Big)\leq 4d^4H^6C\iota +10\beta' d^5H^4\iota+ 2\beta'\sqrt{2d \iota\sum_{k=1}^K(\sigma_{k,h}^2 + H)},
    \end{align*}
    where $\iota=\log\big(1+K/(d\lambda)\big)$.
\end{lemma}

\begin{proof}[Proof of Lemma \ref{lemma:sum-bonus}]
For each stage $h\in[H]$, the summation of bonuses is upper bounded by
\begin{align}
&\sum_{k=1}^K\min\Big(\beta'\sqrt{\bphi(s_h^k,a_h^k)^{\top}\bSigma_{k,h}^{-1}\bphi(s_h^k,a_h^k)},C\Big)\notag\\
&\leq \sum_{k=1}^K\beta'\min\Big(\sqrt{\bphi(s_h^k,a_h^k)^{\top}\bSigma_{k,h}^{-1}\bphi(s_h^k,a_h^k)},1\Big)+C\sum_{k=1}^K\ind \Big\{\sqrt{\bphi(s_h^k,a_h^k)^{\top}\bSigma_{k,h}^{-1}\bphi(s_h^k,a_h^k)} \ge 1\Big\}\notag\\
&\leq C\sum_{k=1}^K\ind \Big\{\sqrt{\bphi(s_h^k,a_h^k)^{\top}\bSigma_{k,h}^{-1}\bphi(s_h^k,a_h^k)} \ge 1\Big\} +10\beta' d^5H^4\iota+ 2\beta'\sqrt{2d \iota\sum_{k=1}^K(\sigma_{k,h}^2 + H)},
\label{eq:006}
    \end{align}
    where $\iota=\log\big(1+K/(d\lambda)\big)$ and the last inequality holds due to Lemma \ref{lemma:keysum:temp}.
Now, we only need to estimate the number of episodes where the bonus is larger than $1$ and we denote these episodes as $\{k_1,..,k_m\}$. For simplicity, we denote 
\begin{align}
\bSigma'_i=\lambda \Ib+\sum_{j=1}^i \bar{\sigma}_{k_j,h}^2 \bphi(s_h^{k_j},a_h^{k_j})\bphi(s_h^{k_j},a_h^{k_j})^{\top}
,\notag
\end{align}
and we have
\begin{align}
\sum_{i=1}^m\bphi(s_h^{k_i},a_h^{k_i})^{\top}\bSigma'_{i-1}\bphi(s_h^{k_i},a_h^{k_i})\ge \sum_{i=1}^m\bphi(s_h^{k_i},a_h^{k_i})^{\top}\bSigma_{{k_i},h}^{-1}\bphi(s_h^{k_i},a_h^{k_i})\ge m.\label{eq:007}
\end{align}
On the other hand, notice that the estimated variance $\bar{\sigma}_{k,h}^2$ is upper bounded by $4d^4H^4/\lambda$, we have
\begin{align}
\sum_{i=1}^m\bphi(s_h^{k_i},a_h^{k_i})^{\top}\bSigma'_{i-1}\bphi(s_h^{k_i},a_h^{k_i})\leq 4d^4H^4/\lambda \cdot \sum_{i=1}^m \bar{\sigma}_{k,h}^{-2}\bphi(s_h^{k_i},a_h^{k_i})^{\top}\bSigma'_{i-1}\bphi(s_h^{k_i},a_h^{k_i})\leq 4d^4H^6 \iota,\label{eq:008}
\end{align}
where $\iota=\log\big(1+K/(d\lambda)\big)$ and the last inequality holds due to Lemma \ref{Lemma:abba}. Combining the results in \eqref{eq:007} and \eqref{eq:008}, we have $m\leq 4d^4H^6 \beta^2\iota$, and it further implies that
\begin{align*}
&\sum_{k=1}^K\min\Big(\beta'\sqrt{\bphi(s_h^k,a_h^k)^{\top}\bSigma_{k,h}^{-1}\bphi(s_h^k,a_h^k)},C\Big)\notag\\
&\leq  C\sum_{k=1}^K\ind \Big\{\sqrt{\bphi(s_h^k,a_h^k)^{\top}\bSigma_{k,h}^{-1}\bphi(s_h^k,a_h^k)} \ge 1\Big\} +10\beta' d^5H^4\iota+ 2\beta'\sqrt{2d \iota\sum_{k=1}^K(\sigma_{k,h}^2 + H)}\notag\\
&\leq 4d^4H^6C\iota +10\beta' d^5H^4\iota+ 2\beta'\sqrt{2d \iota\sum_{k=1}^K(\sigma_{k,h}^2 + H)}.
\end{align*}
Thus, we finish the proof of Lemma \ref{lemma:sum-bonus}.
    \end{proof}

\subsection{Proof of Lemma \ref{lemma:transition}}

\begin{proof}[Proof of Lemma \ref{lemma:transition}]
For all stage $h\in[H]$ and episode $k\in[K]$, we have
    \begin{align}
    &\vvalue_{k,h}(s_h^k)-\vvalue_{k,h}^{\pi^k}(s_h^k) \notag\\
   &= \qvalue_{k,h}(s_h^k,a_h^k)-\qvalue^{\pi^k}_{k,h}(s_h^k,a_h^k)\notag\\
&\leq\min\Big(\hat{\wb}_{k_{\text{last}},h}^{\top}\bphi(s,a)+\beta\sqrt{\bphi(s_h^k,a_h^k)^{\top}\bSigma_{k_{\text{last}},h}^{-1}\bphi(s_h^k,a_h^k)},H\Big)-[\PP_h\vvalue_{k,h+1}](s_h^k,a_h^k)\notag\\
&\qquad + \big[\PP_h(\vvalue_{k,h+1}-\vvalue_{k,h+1}^{\pi^k})\big](s_h^k,a_h^k)\notag\\
   &\leq \big[\PP_h(\vvalue_{k,h+1}-\vvalue_{k,h+1}^{\pi^k})\big](s_h^k,a_h^k)+2\min\Big(\beta\sqrt{\bphi(s_h^k,a_h^k)^{\top}\bSigma_{k_{\text{last}},h}^{-1}\bphi(s_h^k,a_h^k)},H\Big)
  \notag\\
  &\leq \big[\PP_h(\vvalue_{k,h+1}-\vvalue_{k,h+1}^{\pi^k})\big](s_h^k,a_h^k)+4\min\Big(\beta\sqrt{\bphi(s_h^k,a_h^k)^{\top}\bSigma_{k,h}^{-1}\bphi(s_h^k,a_h^k)},H\Big)\notag\\
   &=\vvalue_{k,h+1}(s_{h+1}^{k})-\vvalue_{k,h+1}^{\pi^k}(s_{h+1}^{k})+\big[\PP_h(\vvalue_{k,h+1}-\vvalue_{k,h+1}^{\pi^k})\big](s_h^k,a_h^k)-\big(\vvalue_{k,h+1}(s_{h+1}^{k})-\vvalue_{k,h+1}^{\pi^k}(s_{h+1}^{k})\big) \notag\\
   &\qquad +4\min\Big(\beta\sqrt{\bphi(s_h^k,a_h^k)^{\top}\bSigma_{k,h}^{-1}\bphi(s_h^k,a_h^k)},H\Big),\label{eq:004}
\end{align}
where the first inequality holds due to the definition of value function $Q_{k,h}(s_h^k,a_h^k)$, the second inequality holds due to Lemma \ref{lemma:002} and the last inequality holds due to Lemma \ref{lemma:det} with the updating rule (Line \ref{algorithm:det}). Furthermore, for all stage $h\in[H]$, we have
\begin{align}
   &\sum_{k=1}^K\big(\vvalue_{k,h}(s_h^k)-\vvalue_{k,h}^{\pi^k}(s_h^k)\big)\notag\notag\\
   &\leq \sum_{k=1}^K\sum_{h'=h}^{H}4\min\Big(\beta\sqrt{\bphi(s_h^k,a_h^k)^{\top}\bSigma_{k,h}^{-1}\bphi(s_h^k,a_h^k)},H\Big)\notag\\
   &\qquad +\sum_{k=1}^K\sum_{h'=h}^{H}\Big(\big[\PP_h(\vvalue_{k,h+1}-\vvalue_{k,h+1}^{\pi^k})\big](s_h^k,a_h^k)-\big(\vvalue_{k,h+1}(s_{h+1}^{k})-\vvalue_{k,h+1}^{\pi^k}(s_{h+1}^{k})\big)\Big)\notag\\
   &\leq \sum_{k=1}^K\sum_{h'=h}^{H}4\min\Big(\beta\sqrt{\bphi(s_h^k,a_h^k)^{\top}\bSigma_{k,h}^{-1}\bphi(s_h^k,a_h^k)},H\Big)+4\sqrt{H^3K\log(H/\delta)}\notag\\
   &\leq 16d^4H^8\iota +40\beta d^7H^5\iota+ 8\beta\sum_{h'=h}^H\sqrt{2d \iota\sum_{k=1}^K(\sigma_{k,h'}^2 + H)}+4\sqrt{H^3K\log(H/\delta)}\notag\\
   &\leq 16d^4H^8\iota +40\beta d^7H^5\iota+ 8\beta\sqrt{2d H\iota\sum_{h=1}^H\sum_{k=1}^K(\sigma_{k,h}^2 + H)}+4\sqrt{H^3K\log(H/\delta)},\label{eq:005}
\end{align}
where the first inequality holds by taking the summation of \eqref{eq:004} for $k\in[K]$ and $h\leq h'\leq H$, the second inequality holds due to the definition of events $\cE_1$, the third inequality holds due to Lemma \ref{lemma:sum-bonus} and the last inequality holds due to Cauchy-Schwartz inequality.
 Furthermore, taking the summation of \eqref{eq:005}, we have
\begin{align}
    &\sum_{k=1}^K\sum_{h=1}^H \big[\PP_h(\vvalue_{k,h+1}-\vvalue_{k,h+1}^{\pi^k})\big](s_h^k,a_h^k)\notag\\
    &=\sum_{k=1}^K\sum_{h=1}^H\big(\vvalue_{k,h+1}(s_{h+1}^{k})-\vvalue_{k,h+1}^{\pi^k}(s_{h+1}^{k})\big)\notag\\
    &\qquad+\sum_{k=1}^K\sum_{h=1}^{H}\Big(\big[\PP_h(\vvalue_{k,h+1}-\vvalue_{k,h+1}^{\pi^k})\big](s_h^k,a_h^k)-\big(\vvalue_{k,h+1}(s_{h+1}^{k})-\vvalue_{k,h+1}^{\pi^k}(s_{h+1}^{k})\big)\Big)\notag\\
    &\leq \sum_{k=1}^K\sum_{h=1}^H\big(\vvalue_{k,h+1}(s_{h+1}^{k})-\vvalue_{k,h+1}^{\pi^k}(s_{h+1}^{k})\big)+2\sqrt{2H^3K\log(H/\delta)}\notag\\
    &\leq 16d^4H^9\iota +40\beta d^7H^6\iota+ 8H\beta\sqrt{2d H\iota\sum_{h=1}^H\sum_{k=1}^K(\sigma_{k,h}^2 + H)}+4\sqrt{H^5K\log(H/\delta)}
    ,\notag
\end{align}
where the first inequality holds due to Lemma \ref{lemma:azuma} and the last inequality holds due \eqref{eq:005}. Therefore, we finish the proof of Lemma \ref{lemma:transition}.
\end{proof}

\subsection{Proof of Lemma \ref{lemma:transition1}}
\begin{proof}[Proof of Lemma \ref{lemma:transition1}]
For each stage $h\in[H]$ and episode $k\in[K]$, we have
    \begin{align}
    &\vvalue_{k,h}(s_h^k)-\check{\vvalue}_{k,h}(s_h^k) \notag\\
   &\leq \qvalue_{k,h}(s_h^k,a_h^k)-\check{\qvalue}_{k,h}(s_h^k,a_h^k)\notag\\
&\leq\min\Big(\check{\wb}_{k_{\text{last}},h}^{\top}\bphi(s,a)+\beta\sqrt{\bphi(s_h^k,a_h^k)^{\top}\bSigma_{k_{\text{last}},h}^{-1}\bphi(s_h^k,a_h^k)},H\Big)-[\PP_h\vvalue_{k,h+1}](s_h^k,a_h^k)\notag\\
&\qquad -\max\Big(\hat{\wb}_{k_{\text{last}},h}^{\top}\bphi(s,a)-\bar{\beta}\sqrt{\bphi(s_h^k,a_h^k)^{\top}\bSigma_{k_{\text{last}},h}^{-1}\bphi(s_h^k,a_h^k)},0\Big)+[\PP_h\check{\vvalue}_{k,h+1}](s_h^k,a_h^k)\notag \notag\\
&\qquad + \big[\PP_h(\vvalue_{k,h+1}-\check{\vvalue}_{k,h+1})\big](s_h^k,a_h^k)\notag\\
   &\leq \big[\PP_h(\vvalue_{k,h+1}-\check{\vvalue}_{k,h+1})\big](s_h^k,a_h^k)+2\min\Big(\beta\sqrt{\bphi(s_h^k,a_h^k)^{\top}\bSigma_{k_{\text{last}},h}^{-1}\bphi(s_h^k,a_h^k)},H\Big)
  \notag\\
  &\qquad + 2\min\Big(\bar{\beta}\sqrt{\bphi(s_h^k,a_h^k)^{\top}\bSigma_{k_{\text{last}},h}^{-1}\bphi(s_h^k,a_h^k)},H\Big) \notag\\
  &\leq \big[\PP_h(\vvalue_{k,h+1}-\check{\vvalue}_{k,h+1})\big](s_h^k,a_h^k)+4\min\Big(\beta\sqrt{\bphi(s_h^k,a_h^k)^{\top}\bSigma_{k,h}^{-1}\bphi(s_h^k,a_h^k)},H\Big)\notag\\
  &\qquad + 4\min\Big(\bar{\beta}\sqrt{\bphi(s_h^k,a_h^k)^{\top}\bSigma_{k,h}^{-1}\bphi(s_h^k,a_h^k)},H\Big) \notag\\
   &=\vvalue_{k,h+1}(s_{h+1}^{k})-\check{\vvalue}_{k,h+1}(s_{h+1}^{k})+\big[\PP_h(\vvalue_{k,h+1}-\check{\vvalue}_{k,h+1})\big](s_h^k,a_h^k)-\big(\vvalue_{k,h+1}(s_{h+1}^{k})-\check{\vvalue}_{k,h+1}(s_{h+1}^{k})\big) \notag\\
   &\qquad +4\min\Big(\beta\sqrt{\bphi(s_h^k,a_h^k)^{\top}\bSigma_{k,h}^{-1}\bphi(s_h^k,a_h^k)},H\Big)+4\min\Big(\bar{\beta}\sqrt{\bphi(s_h^k,a_h^k)^{\top}\bSigma_{k,h}^{-1}\bphi(s_h^k,a_h^k)},H\Big),\label{eq:009}
\end{align}
where the first inequality holds due to the fact that $\check{V}_{k,h}(s_h^k)=\max_a \check{Q}_{k,h}(s_h^k,a)\ge \check{Q}_{k,h}(s_h^k,a_h^k) $, the second inequality holds due to the definition of value functions $Q_{k,h}$ and $\check{Q}_{k,h}$, the third inequality holds due to Lemma \ref{lemma:002} and Lemma \ref{lemma:hoeffding-type}, and the last inequality holds due to Lemma \ref{lemma:det} with the updating rule (Line \ref{algorithm:det}). Furthermore, for all stage $h\in[H]$, we have 
\begin{align}
   &\sum_{k=1}^K\big(\vvalue_{k,h}(s_h^k)-\check{\vvalue}_{k,h}(s_h^k)\big)\notag\notag\\
   &\leq \sum_{k=1}^K\sum_{h'=h}^{H}4\min\Big(\beta\sqrt{\bphi(s_h^k,a_h^k)^{\top}\bSigma_{k,h}^{-1}\bphi(s_h^k,a_h^k)},H\Big)+4\min\Big(\bar{\beta}\sqrt{\bphi(s_h^k,a_h^k)^{\top}\bSigma_{k,h}^{-1}\bphi(s_h^k,a_h^k)},H\Big)\notag\\
   &\qquad +\sum_{k=1}^K\sum_{h'=h}^{H}\Big(\big[\PP_h(\vvalue_{k,h+1}-\check{\vvalue}_{k,h+1})\big](s_h^k,a_h^k)-\big(\vvalue_{k,h+1}(s_{h+1}^{k})-\
   \check{\vvalue}_{k,h+1}(s_{h+1}^{k})\big)\Big)\notag\\
   &\leq \sum_{k=1}^K\sum_{h'=h}^{H}4\min\Big(\beta\sqrt{\bphi(s_h^k,a_h^k)^{\top}\bSigma_{k,h}^{-1}\bphi(s_h^k,a_h^k)},H\Big)+4\min\Big(\bar{\beta}\sqrt{\bphi(s_h^k,a_h^k)^{\top}\bSigma_{k,h}^{-1}\bphi(s_h^k,a_h^k)},H\Big)\notag\\
   &\qquad + 4\sqrt{H^3K\log(H/\delta)}\notag\\
   &\leq 32d^4H^8\iota +40(\beta+\bar\beta) d^7H^5\iota+ 8(\beta+\bar\beta)\sum_{h'=h}^H\sqrt{2d \iota\sum_{k=1}^K(\sigma_{k,h'}^2 + H)}+4\sqrt{H^3K\log(H/\delta)}\notag\\
   &\leq 32d^4H^8\iota +40(\beta+\bar\beta)d^7H^5\iota+ 8(\beta+\bar\beta)\sqrt{2d H\iota\sum_{h=1}^H\sum_{k=1}^K(\sigma_{k,h}^2 + H)}+4\sqrt{H^3K\log(H/\delta)},\label{eq:010}
\end{align}
where the first inequality holds by taking the summation of \eqref{eq:009} for $k\in[K]$ and $h\leq h'\leq H$, the second inequality holds due to the definition of event $\cE_2$, the third inequality holds due to Lemma \ref{lemma:sum-bonus} and the last inequality holds due to Cauchy-Schwartz inequality.
 Furthermore, taking the summation of \eqref{eq:010}, we have
\begin{align}
    &\sum_{k=1}^K\sum_{h=1}^H \big[\PP_h(\vvalue_{k,h+1}-\check{\vvalue}_{k,h+1})\big](s_h^k,a_h^k)\notag\\
    &=\sum_{k=1}^K\sum_{h=1}^H\big(\vvalue_{k,h+1}(s_{h+1}^{k})-\check{\vvalue}_{k,h+1}(s_{h+1}^{k})\big)\notag\\
    &\qquad+\sum_{k=1}^K\sum_{h=1}^{H}\Big(\big[\PP_h(\vvalue_{k,h+1}-\check{\vvalue}_{k,h+1})\big](s_h^k,a_h^k)-\big(\vvalue_{k,h+1}(s_{h+1}^{k})-\check{\vvalue}_{k,h+1}(s_{h+1}^{k})\big)\Big)\notag\\
    &\leq \sum_{k=1}^K\sum_{h=1}^H\big(\vvalue_{k,h+1}(s_{h+1}^{k})-\check{\vvalue}_{k,h+1}(s_{h+1}^{k})\big)+2\sqrt{2H^3K\log(H/\delta)}\notag\\
    &\leq 32d^4H^9\iota +40(\beta+\bar{\beta}) d^7H^6\iota+ 8H(\beta+\bar{\beta})\sqrt{2d H\iota\sum_{h=1}^H\sum_{k=1}^K(\sigma_{k,h}^2 + H)}+4\sqrt{H^5K\log(H/\delta)}
    ,\notag
\end{align}
where the first inequality holds due to Lemma \ref{lemma:azuma} and the last inequality holds due \eqref{eq:010}. Therefore, we finish the proof of Lemma \ref{lemma:transition1}.
\end{proof}

\subsection{Proof of Lemma \ref{LEMMA: TOTAL-ESTIMATE-VARIANCE}}

\begin{proof}[Proof of Lemma \ref{LEMMA: TOTAL-ESTIMATE-VARIANCE}]
    According to the definition of estimated variance $\sigma_{k,h}$, we have 
    \begin{align}
        \sum_{k=1}^K\sum_{h=1}^H\sigma_{k,h}^{2} &=\sum_{k=1}^K\sum_{h=1}^H[\bar{\VV}_{k,h}\vvalue_{k,h+1}](s_h^k,a_h^k)+E_{k,h}+D_{k,h}+H\notag\\
    &=H^2K+\underbrace{\sum_{k=1}^K\sum_{h=1}^H \big([\bar{\VV}_{k,h}\vvalue_{k,h+1}](s_h^k,a_h^k)-[{\VV}_{h}\vvalue_{k,h+1}](s_h^k,a_h^k)\big)}_{I_1}+\underbrace{\sum_{k=1}^K\sum_{h=1}^HE_{k,h}}_{I_2}+\underbrace{\sum_{k=1}^K\sum_{h=1}^HD_{k,h}}_{I_3}\notag\\
    &\qquad +\underbrace{\sum_{k=1}^K\sum_{h=1}^H \big([{\VV}_{h}\vvalue_{k,h+1}](s_h^k,a_h^k)-[{\VV}_{h}\vvalue^{\pi^k}_{k,h+1}](s_h^k,a_h^k)\big)}_{I_4}+\underbrace{\sum_{k=1}^K\sum_{h=1}^H [{\VV}_{h}\vvalue^{\pi^k}_{k,h+1}](s_h^k,a_h^k)}_{I_5}.\label{eq:011}
    \end{align}
For the term $I_1$, according to Lemma \ref{lemma:varaince}, it is upper bounded by:
\begin{align}
    I_1&=\sum_{k=1}^K\sum_{h=1}^H \big([\bar{\VV}_{k,h}\vvalue_{k,h+1}](s_h^k,a_h^k)-[{\VV}_{h}\vvalue_{k,h+1}](s_h^k,a_h^k)\big)\leq \sum_{k=1}^K\sum_{h=1}^H E_{k,h}= I_2.\label{eq:012}
\end{align}
For the term $I_2$, it is upper bounded by
\begin{align}
    I_2&=\sum_{k=1}^K\sum_{h=1}^HE_{k,h}\notag\\
    &=\sum_{k=1}^K\sum_{h=1}^H\min \Big\{\tilde{\beta}\big\|\bSigma_{k,h}^{-1/2}\bphi(s_h^k,a_h^k)\big\|_2,H^2\Big\}+\min \Big\{2H\bar{\beta}\big\|\bSigma_{k,h}^{-1/2}\bphi(s_h^k,a_h^k)\big\|_2,H^2\Big\}\notag\\
    &\leq 8d^4H^9\iota +(10\tilde{\beta}+20\bar{\beta}) d^5H^5\iota+ (2\tilde{\beta}+4\bar{\beta})H\sqrt{2d \iota\sum_{h=1}^H\sum_{k=1}^K(\sigma_{k,h}^2 + H)},\label{eq:013}
\end{align}
where $\iota=\log\big(1+K/(d\lambda)\big)$ and the inequality holds due to Lemma \ref{lemma:sum-bonus}.

For the term $I_3$, it is upper bounded by
\begin{align}
    I_3&=\sum_{k=1}^K\sum_{h=1}^H D_{k,h}\notag\\
    &=\sum_{k=1}^K\sum_{h=1}^H\min\bigg\{4d^3H^2\Big(\hat{\wb}_{k,h}^{\top}\bphi(s,a)-\check{\wb}_{k,h}^{\top}\bphi(s,a)+2\bar{\beta}\sqrt{\bphi(s,a)^{\top}\bSigma_{k,h}^{-1}\bphi(s,a)}\Big),d^3H^3\bigg\}\notag\\
    &\leq \sum_{k=1}^K\sum_{h=1}^H \min\bigg\{4d^3H^2\Big(\big[\PP_h(\vvalue_{k,h+1}-\check{\vvalue}_{k,h+1})\big](s_h^k,a_h^k)+4\bar{\beta}\sqrt{\bphi(s,a)^{\top}\bSigma_{k,h}^{-1}\bphi(s,a)}\Big),d^3H^3\bigg\}\notag\\
    &\leq  \sum_{k=1}^K\sum_{h=1}^H 4d^3H^2\big[\PP_h(\vvalue_{k,h+1}-\check{\vvalue}_{k,h+1})\big](s_h^k,a_h^k) +\sum_{k=1}^K\sum_{h=1}^H  \min\Big\{16d^3H^2\bar{\beta}\sqrt{\bphi(s,a)^{\top}\bSigma_{k,h}^{-1}\bphi(s,a)},d^3H^3\Big\}\notag\\
    &\leq  4d^7H^9\iota +160\bar{\beta} d^8H^7\iota+ 32d^3H^3\bar{\beta}\sqrt{2d \iota\sum_{h=1}^H\sum_{k=1}^K(\sigma_{k,h}^2 + H)}\notag\\
    &\qquad + \sum_{k=1}^K\sum_{h=1}^H 4d^3H^2\big[\PP_h(\vvalue_{k,h+1}-\check{\vvalue}_{k,h+1})\big](s_h^k,a_h^k)\notag\\
    &\leq 132d^7H^{11}\iota +320(\beta+\bar{\beta}) d^{10}H^8\iota+ 64d^3H^3(\beta+\bar{\beta})\sqrt{2d H\iota\sum_{h=1}^H\sum_{k=1}^K(\sigma_{k,h}^2 + H)}+4d^3\sqrt{H^9K\log(H/\delta)},\label{eq:014}
\end{align}
where $\iota=\log\big(1+K/(d\lambda)\big)$, the first inequality holds due to Lemma \ref{lemma:hoeffding-type}, the second inequality holds due to the fact that $V_{k,h+1}(s) \ge V_{h+1}^*(s)\ge \check{V}_{k,h+1}(s) $, the third inequality holds due to Lemma \ref{lemma:sum-bonus}  and the last inequality holds due to Lemma \ref{lemma:transition1}.

For the term $I_4$, it is upper bounded by
\begin{align}
    I_4&=\sum_{k=1}^K\sum_{h=1}^H \big([{\VV}_{h}\vvalue_{k,h+1}](s_h^k,a_h^k)-[{\VV}_{h}\vvalue^{\pi^k}_{k,h+1}](s_h^k,a_h^k)\big)\notag\\
    &=\sum_{k=1}^K\sum_{h=1}^H \Big([\PP_h \vvalue_{k,h+1}^2](s_h^k,a_h^k)-\big([\PP_h \vvalue_{k,h+1}](s_h^k,a_h^k)\big)^2
    -[\PP_h (\vvalue^{\pi^k}_{k,h+1})^2](s_h^k,a_h^k)+\big([\PP_h \vvalue^{\pi^k}_{k,h+1}](s_h^k,a_h^k)\big)^2\Big)\notag\\
    &\leq \sum_{k=1}^K\sum_{h=1}^H\big([\PP_h \vvalue_{k,h+1}^2](s_h^k,a_h^k)-[\PP_h (\vvalue^{\pi^k}_{k,h+1})^2](s_h^k,a_h^k)\big)\notag\\
    &\leq 2H \sum_{k=1}^K\sum_{h=1}^H\big([\PP_h \vvalue_{k,h+1}](s_h^k,a_h^k)-[\PP_h \vvalue^{\pi^k}_{k,h+1}](s_h^k,a_h^k)\big)\notag\\
    &\leq 32d^4H^{10}\iota +80\beta d^7H^7\iota+ 16H^2\beta\sqrt{2d H\iota\sum_{h=1}^H\sum_{k=1}^K(\sigma_{k,h}^2 + H)}+8\sqrt{H^7K\log(H/\delta)},\label{eq:015}
\end{align}
where the first inequality holds due to the fact that $\vvalue^{\pi^k}_{k,h+1}(s')\leq \vvalue_{k,h+1}(s')$, the second inequality holds due to $0\leq \vvalue_{k,h+1}(s'),\vvalue^{\pi^k}_{k,h+1}(s') \leq H$ and the last inequality holds due to Lemma \ref{lemma:transition}.

Based on the definition of events $\cE_3$, for the term $I_5$, we have
\begin{align}
    I_5=\sum_{k=1}^K\sum_{h=1}^H [{\VV}_{h}\vvalue^{\pi^k}_{k,h+1}](s_h^k,a_h^k)\leq 3\big(H^2K+H^3\log(1/\delta)\big).\label{eq:016}
\end{align}
Substituting the results in \eqref{eq:012}, \eqref{eq:013}, \eqref{eq:014}, \eqref{eq:015} and \eqref{eq:016} into \eqref{eq:011}, we have 
\begin{align}
&\sum_{k=1}^K\sum_{h=1}^H\sigma_{k,h}^{2}\notag\\
&=I_1+I_2+I_3+I_4+I_5\notag\\
&\leq 3H^2K+ 183d^7H^{11}\iota +460(\beta+\tilde{\beta}+\bar{\beta}) d^{10}H^8\iota + 12d^3\sqrt{H^9K\log(H/\delta)}\notag\\
&\qquad + 92d^3H^3(\beta+\tilde{\beta}+\bar{\beta})\sqrt{2dH \iota\sum_{h=1}^H\sum_{k=1}^K(\sigma_{k,h}^2 + H)} \notag\\
&\leq 3H^2K+ 183d^7H^{11}\iota +460(\beta+\tilde{\beta}+\bar{\beta}) d^{10}H^8\iota + 12d^3\sqrt{H^9K\log(H/\delta)}\notag\\
&\qquad +92d^3H^3 (\beta+\tilde{\beta}+\bar{\beta}) \sqrt{2dH^2K\iota}   + 92d^3H^3(\beta+\tilde{\beta}+\bar{\beta})\sqrt{2d H\iota\sum_{h=1}^H\sum_{k=1}^K\sigma_{k,h}^2} ,\notag
\end{align}
where 
$\iota=\log\big(1+K/(d\lambda)\big)$,
\begin{align*}
    \beta&=O\Big(H\sqrt{d\lambda}+\sqrt{d \log^2\big(1+dKH/(\delta\lambda)\big)}\Big)\notag\\
    \tilde{\beta}&= O \Big(H^2\sqrt{d\lambda} +\sqrt{d^3H^4\log^2\big(dHK/(\delta\lambda)\big)}\Big)\notag\\
    \bar{\beta}&=O\Big(H\sqrt{d\lambda} +\sqrt{d^3H^2\log^2\big(dHK/(\delta\lambda)\big)}\Big),
\end{align*}
and the last inequality holds due to the fact that $\sqrt{a+b}\leq \sqrt{a}+\sqrt{b}$. 
Therefore, by the fact that $x\leq a\sqrt{x}+b$ implies $x\leq a^2+2b$ and $\lambda=1/H^2$, we have
\begin{align*}
    \sum_{k=1}^K\sum_{h=1}^H \sigma_{k,h}^{2} \leq O\big(H^2K+d^{10.5}H^{16}\log^{1.5}(1+dKH/\delta)\big).
\end{align*}
Thus, we complete the proof of Lemma \ref{LEMMA: TOTAL-ESTIMATE-VARIANCE}.
\end{proof}

 % define high probability evenets for hoffding concentration

\section{Covering Number Arguments}
\subsection{Number of Value Function Updating}
According to the determinant-based criterion in Algorithm \ref{algorithm1} (Line \ref{algorithm:det}), the number of episodes where the algorithm updates the value function is upper bounded by:
\begin{lemma}\label{lemma:update-number}
    The number of episodes where the algorithm updates the value function in Algorithm \ref{algorithm1} is upper bounded by $dH\log(1+K/\lambda)$.
\end{lemma}
\begin{proof}
   We denote $k_0=0$ and suppose that $\{k_1,..,k_m\}$ be the episodes where the algorithm updates the value function. Then according to the determinant-based criterion (Line \ref{algorithm:det}), for each episodes $k_i$, there exists a stage $h\in[H]$ such that 
   \begin{align}
       \det (\bSigma_{k_i,h})\ge 2 \det (\bSigma_{k_{i-1},h}).\notag
   \end{align}
   According to the update rule of $\bSigma_{k,h}$ (Line \ref{algorithm:line2}), for other stage $h'\ne h$, 
   we have $\bSigma_{k_i,h'}\succeq\bSigma_{k_i,h'}$, which implies $\det(\bSigma_{k_i,h'})\ge \det(\bSigma_{k_i,h'})$. Thus, we have
 \begin{align}
     \prod_{h=1}^H \det (\bSigma_{k_i,h})\ge 2  \prod_{h=1}^H\det (\bSigma_{k_{i-1},h}).\label{eq:01}
 \end{align}
Applying the result \eqref{eq:01} overall episodes in $\{k_1,..,k_m\}$, we have
\begin{align}
    \prod_{h=1}^H \det (\bSigma_{k_m,h})\ge 2^m  \prod_{h=1}^H\det (\bSigma_{k_0,h})=2^m \prod_{h=1}^H\det(\lambda \Ib)=2^m\lambda^{dH}.\label{eq:02}
\end{align}
On the other hand, the determinant $\det (\bSigma_{k_m,h})$ is upper bounded by:
\begin{align}
    \prod_{h=1}^H \det (\bSigma_{k_m,h})\leq 
    \prod_{h=1}^H \det (\bSigma_{K,h})\leq (\lambda + K)^{dH}
    ,\label{eq:03}
\end{align}
where the first inequality holds due to $\bSigma_{K,h}\succeq \bSigma_{k_m,h}$, the last inequality holds due to $ \bar{\sigma}_{k,h}^{-1}\leq 1$ and $\|\bphi(s,a)\|_2\leq 1$.
Combining the results in \eqref{eq:02} and \eqref{eq:03}, we have
\begin{align*}
    m\leq dH\log(1+K/\lambda).
\end{align*}
Thus, we finish the proof of Lemma \ref{lemma:update-number}.
\end{proof}
\subsection{Norm of the Weight Vectors
}
In this section, we provide the following upper bounds for the norm of the weight vectors. 
\begin{lemma}\label{lemma:vector-range}
For all stage $h\in[H]$ and all episode $k\in \NN$, the norm of the weight vector $\hat{\wb}_{k,h}$ can be upper bounded by
\begin{align}
    \|\hat{\wb}_{k,h}\|_2\leq H\sqrt{dK/\lambda}.\notag
\end{align}
\end{lemma}
\begin{proof}[Proof of Lemma \ref{lemma:vector-range}]
 According to the definition of weight vector $\hat{\wb}_{k,h}$ in Algorithm \ref{algorithm1}, we have \begin{align}
  \bSigma_{k,h}&=\lambda \Ib+\sum_{i=1}^{k-1}\bar{\sigma}_{k,i}^{-2} \bphi(s_h^{i},a_h^{i})\bphi(s_h^{i},a_h^{i})^\top,\notag\\
 \hat{\wb}_{k,h}&=\bSigma_{k,h}^{-1}  \sum_{i=1}^{k-1}  \bar{\sigma}_{k,i}^{-2}\bphi(s_h^{i},a_h^{i})\vvalue_{k,h+1}(s_{h+1}^{i}) .\notag
\end{align}
Then for the norm $\|\hat{\wb}_{k,h}\|_2$, we have the following inequality
\begin{align}
    \|\hat{\wb}_{k,h}\|^2_2&=\bigg\|\notag \bSigma_{k,h}^{-1}\sum_{i=1}^{k-1}  \bar{\sigma}_{k,i}^{-2}\bphi(s_h^{i},a_h^{i})\vvalue_{k,h+1}(s_{h+1}^{i})\bigg\|^2_2\\
    &\leq k\sum_{i=1}^{k-1} \big\|\notag \bSigma_{k,h}^{-1} \bar{\sigma}_{k,i}^{-2}\bphi(s_h^{i},a_h^{i})\vvalue_{k,h+1}(s_{h+1}^{i})\big\|^2_2\notag\\
    &\leq  kH^2\sum_{i=1}^{k-1}\bar{\sigma}_{k,i}^{-2} \big\|\notag \bSigma_{k,h}^{-1} \bphi(s_h^{i},a_h^{i})\big\|^2_2\notag\\
    &\leq \frac{kH^2}{\lambda}\sum_{i=1}^{k-1}\bar{\sigma}_{k,i}^{-2}\bphi(s_h^{i},a_h^{i})^{\top}\bSigma_{k,h}^{-1} \bphi(s_h^{i},a_h^{i})\notag\\
    &= \frac{kH^2}{\lambda}\text{tr}\bigg(\bSigma_{k,h}^{-1}\sum_{i=1}^{k-1}\bar{\sigma}_{k,i}^{-2}\bphi(s_h^{i},a_h^{i})^{\top} \bphi(s_h^{i},a_h^{i})\bigg),\label{eq:616}
\end{align}
where the first inequality holds due to Cauchy-Schwartz inequality, the second inequality holds due to $ 0\leq V_{k,h+1}(s,a)\leq H$ and the last inequality holds due to $\bSigma_{k,h}\succeq \lambda \Ib$. Now, we assume the eigen-decomposition of matrix $\sum_{i=1}^{k-1}\bar{\sigma}_{k,i}^{-2}\bphi(s_h^{i},a_h^{i})^{\top} \bphi(s_h^{i},a_h^{i})$ is $Q^{\top}\Lambda Q$ and we have
\begin{align}
\text{tr}\bigg(\bSigma_{k,h}^{-1}\sum_{i=1}^{k-1}\bar{\sigma}_{k,i}^{-2}\bphi(s_h^{i},a_h^{i})^{\top} \bphi(s_h^{i},a_h^{i})\bigg)&=\text{tr}\big((Q^{\top}\Lambda Q+\lambda \Ib)^{-1}Q^{\top}\Lambda Q\big)\notag\\
    &=\text{tr}\big((\Lambda+\lambda \Ib)^{-1}\Lambda\big)\notag\\
    &=\sum_{i=1}^d\frac{\Lambda_i}{\Lambda_i+\lambda}\notag\\
    &\leq d.\label{eq:17}
\end{align}
Substituting \eqref{eq:17} into \eqref{eq:616}, we have
\begin{align}
    \|\hat{\wb}_{k,h}\|^2_2&\leq  \frac{kH^2}{\lambda}\text{tr}\bigg(\bSigma_{k,h}^{-1}\sum_{i=1}^{k-1}\bar{\sigma}_{k,i}^{-2}\bphi(s_h^{i},a_h^{i})^{\top} \bphi(s_h^{i},a_h^{i})\bigg)\notag\\
    &\leq \frac{kH^2d}{\lambda},
\end{align}
where the first inequality holds due to \eqref{eq:616} and the last inequality holds due to \eqref{eq:17}.
Thus, we finish the proof of Lemma \ref{lemma:vector-range}
\end{proof}
In addition, for the pessimistic weight vector $\check{\wb}_{k,h}$ and weight vector $\tilde {\wb}_{k,h}$, we have the following lemmas:
\begin{lemma}\label{lemma:vector-range1}
For all stage $h\in[H]$ and all episode $k\in \NN$, the norm of the weight vector $\check{\wb}_{k,h}$ can be upper bounded by
\begin{align}
    \|\check{\wb}_{k,h}\|_2\leq H\sqrt{dK/\lambda}.\notag
\end{align}
\end{lemma}
\begin{proof}[Proof of Lemma \ref{lemma:vector-range1}]
    The proof is almost the same as Lemma \ref{lemma:vector-range} and we only need to replace the optimistic value function $V_{k,h}(s,a)$ by the pessimistic value function $\check{V}_{k,h}(s,a)$.
\end{proof}

\begin{lemma}\label{lemma:vector-range2}
For each stage $h\in[H]$ and each episode $k\in \NN$, the norm of the weight vector $\tilde{\wb}_{k,h}$ can be upper bounded by
\begin{align}
    \|\check{\wb}_{k,h}\|_2\leq H^2\sqrt{dK/\lambda}.\notag
\end{align}
\end{lemma}
\begin{proof}[Proof of Lemma \ref{lemma:vector-range2}]
    The proof is almost the same as Lemma \ref{lemma:vector-range} and we only need to replace the optimistic value function $V_{k,h}(s,a)$ with the squared value function $V_{k,h}^2(s,a)$.
\end{proof}

\subsection{Function Class and Covering Number}
Combining the update rule (Line \ref{algorithm:det}) with Lemma \ref{lemma:update-number} and Lemma \ref{lemma:vector-range}, for each episodes $k\in[K]$ and $h\in[H]$, the optimistic value function $V_{k,h}=\min_{i\leq k}\max_{a}Q_{i,h}(s,a)$ belong to the following function class 
\begin{align}
    \mathcal{V}_h=\Bigg\{V\bigg|V(\cdot)&=\max_{a}\min_{1\leq i\leq l}\min\bigg(H,r_h(\cdot,a)+\wb_{i}^\top\bphi(\cdot,a)+\beta\sqrt{\bphi(\cdot,a)^\top \bSigma_{i}^{-1}\bphi(\cdot,a)}\bigg),\|\wb_i\|\leq L,\bSigma_i \succeq \lambda \Ib\Bigg\},\label{eq:optimistic-set}
\end{align}
where $l= dH \log(1+K/\lambda)$ and $L=H\sqrt{dK/\lambda}$.
Similarly, for each episodes $k\in[K]$ and $h\in[H]$, the pessimistic value function $\check{V}_{k,h}=\max_{i\leq k}\max_{a}\check{Q}_{i,h}(s,a)$ belongs to the following function class 
\begin{align}
    \check{\mathcal{V}}_h=\Bigg\{V\bigg|V(\cdot)&=\max_{a}\max_{1\leq i\leq l}\max\bigg(0,r_h(\cdot,a)+\wb_{i}^\top\bphi(\cdot,a)-\beta\sqrt{\bphi(\cdot,a)^\top \bSigma_{i}^{-1}\bphi(\cdot,a)}\bigg),\|\wb_i\|\leq L,\bSigma_i \succeq \lambda \Ib\Bigg\},\label{eq:pessimistic-set}
\end{align}
where $l= dH \log(1+K/\lambda)$ and $L=H\sqrt{dK/\lambda}$.
To compute the covering number of function classes $\mathcal{V}_h$, $\mathcal{V}_h^2$ and $\check{\mathcal{V}}_h$, we need the following result on the Euclidean ball. 
\begin{lemma}[Lemma D.5, \citealt{jin2020provably}]\label{lemma:ball-cover}
For a Euclidean ball with radius $R$ in $\RR^d$, the  $\epsilon$-covering number of this ball is upper bounded by $(1+2R/\epsilon)^d$.
\end{lemma}
With the help of Lemma \ref{lemma:ball-cover}, the covering number $\mathcal{N}_\epsilon$ of optimistic function class $\mathcal{V}$ can be upper bounded by the following lemma:
\begin{lemma}\label{lemma:covering-number}
For optimistic function class $\mathcal{V}_h$, we define the distance between two function $V_1$ and $V_2$ as $V_1,V_2\in \mathcal{V}_h$ as $dist(V_1,V_2)=\max_{s}|V_1(s)-V_2(s)|$. With respect to this distance function, the $\epsilon$-covering number $\mathcal{N}_\epsilon$ of the function class $\mathcal{V}_h$ can be upper bounded by
\begin{align}
   \log \mathcal{N}_\epsilon \leq dl \log (1+4L/\epsilon) + d^2l \log (1+8\sqrt{d}\beta^2/\epsilon^2)\notag
\end{align}
\end{lemma}

\begin{proof}[Proof of Lemma \ref{lemma:covering-number}]
For any two function $\vvalue_1, \vvalue_2 \in \mathcal{V}_h$, according to the definition of function class $\mathcal{V}_h$, we have
\begin{align}
  \vvalue_1(\cdot)= \max_{a}\min_{1\leq i\leq l} \min\bigg(H,r_h(\cdot,a)+\wb_{1,i}^\top\bphi(\cdot,a)+\beta\sqrt{\bphi(\cdot,a)^\top \bGamma_{1,i}\bphi(\cdot,a)}\bigg) ,\notag\\
    \vvalue_2(\cdot)=\max_{a}\min_{1\leq i\leq l} \min\bigg(H,r_h(\cdot,a)+\wb_{2,i}^\top\bphi(\cdot,a)+\beta\sqrt{\bphi(\cdot,a)^\top \bGamma_{2,i}\bphi(\cdot,a)}\bigg),\notag 
\end{align}
where $\|\wb_{1,i}\|_2,\|\wb_{2,i}\|_2\leq L$ and $\bGamma_{1,i}, \bGamma_{2,i} \preceq  \Ib$. Since all of the functions $\min_{1\leq i\leq l}$, $\max_a$  and $\min(H,\cdot)$ are contraction functions, we have
\begin{align}
    \text{dist}(\vvalue_1,\vvalue_2)&=\max_{s\in \cS} \big|\vvalue_1(s)-\vvalue_2(s)\big|\notag\\
    &\leq \max_{1\leq i\leq l, s\in \cS, a\in \cA}  
    \Big|  \wb_{1,i}^\top\bphi(s,a)+\beta\sqrt{\bphi(s,a)^\top \bGamma_{1,i}\bphi(s,a)}\notag\\
    &\qquad - \wb_{2,i}^\top\bphi(s,a)-\beta\sqrt{\bphi(s,a)^\top \bGamma_{2,i}\bphi(s,a)}\Big|\notag\\
    &\leq \beta \max_{1\leq i\leq l, s\in \cS, a\in \cA} \Big|\sqrt{\bphi(s,a)^\top \bGamma_{1,i}\bphi(s,a)}- \sqrt{\bphi(s,a)^\top \bGamma_{2,i}\bphi(s,a)} \Big|\notag\\
    &\qquad + \max_{1\leq i\leq l, s\in \cS, a\in \cA} \big |(\wb_{1,i}-\wb_{2,i})^{\top}\bphi(s,a)\big|\notag\\
    &\leq \beta \max_{1\leq i\leq l, s\in \cS, a\in \cA} \Big|\sqrt{\bphi(s,a)^\top (\bGamma_{1,i}-\bGamma_{2,i})\bphi(s,a)}\Big|\notag\\
    &\qquad +\max_{1\leq i\leq l, s\in \cS, a\in \cA} \big |(\wb_{1,i}-\wb_{2,i})^{\top}\bphi(s,a)\big|\notag \notag\\
    &\leq \beta \max_{1\leq i\leq l}\sqrt{\|\bGamma_{1,i}-\bGamma_{2,i}\|_{F}}+ \max_{1\leq i\leq l}\|\wb_{1,i}-\wb_{2,i}\|_2,\label{eq:27}
\end{align}
where the first inequality holds due to the contraction property, the second inequality holds due to the fact that $\max_{x}|f(x)+g(x)|\leq \max_{x}|f(x)|+\max_{x}|g(x)|$, the third inequality holds due to $|\sqrt{x}-\sqrt{y}|\ge |\sqrt{x}-\sqrt{y}|$ and the last inequality holds due to the fact that $\|\bphi(s,a)\|_2\leq 1$.
Now, we denote $\mathcal{C}_{\wb}$ as a $\epsilon/2$-cover of the set $\big\{\wb \in \RR^d\big| \|\wb\|_2\leq L\big\}$ and $\mathcal{C}_{\bGamma}$ as a $\epsilon^2/(4\beta^2)$-cover of the set $\{\bGamma\in \RR^{d\cdot d}\big|\|\bGamma\|_{F}\leq \sqrt{d} \}$ with respect to the Frobenius norms. Thus, according to Lemma \ref{lemma:ball-cover}, we have following property:
\begin{align}
    |\mathcal{C}_{\wb}|\leq \big(1+4L/\epsilon\big)^d, |\mathcal{C}_{\bGamma}|\leq \big(1+8\sqrt{d}\beta^2/\epsilon^2\big)^{d^2}.\label{eq:28}
\end{align}
By the definition of covering number, for any function $V_1\in \mathcal{V}$ with parameters $\wb_{1,i}, \bGamma_{1,i}(1\leq i\leq l)$, s other parameters $\wb_{2,i}, \bGamma_{2,i}(1\leq i\leq l)$ such that $\wb_{2,i} \in \mathcal{C}_{\wb}, \bGamma_{2,i} \in \mathcal{C}_{\bGamma}$ and $\|\wb_{2,i}-\wb_{1,i}\|_2\leq \epsilon/2, \|\bGamma_{2,i}-\bGamma_{1,i}\|_{F}\leq \epsilon^2/(4\beta^2)$. Thus, we have
\begin{align}
     \text{dist}(\vvalue_1,\vvalue_2)\leq \beta \max_{1\leq i\leq l}\sqrt{\|\bGamma_{1,i}-\bGamma_{2,i}\|_{F}}+ \max_{1\leq i\leq l}\|\wb_{1,i}-\wb_{2,i}\|_2 \leq \epsilon,\notag
\end{align}
where the inequality holds due to \eqref{eq:27}.
Therefore, the $\epsilon$-covering number of optimistic function class $\mathcal{V}_h$ is bounded by $\mathcal{N}_\epsilon\leq |\mathcal{C}_{\wb}|^l\cdot |\mathcal{C}_{\bGamma}|^l$ and it implies
\begin{align}
    \log \mathcal{N}_\epsilon &\leq dl \log (1+4L/\epsilon) + d^2l \log (1+8\sqrt{d}\beta^2/\epsilon^2),\notag
\end{align}
where the first inequality holds due to \eqref{eq:28}. Thus, we finish the proof of Lemma \ref{lemma:covering-number}.
\end{proof}

With a similar argument as Lemma \ref{lemma:covering-number}, the covering number $\mathcal{N}_\epsilon$ of pessimistic value function class $\check{\mathcal{V}}_h$ can be upper bounded by the following lemma:
\begin{lemma}\label{lemma:covering-number1}
For pessimistic function class $\check{\mathcal{V}}_h$, we define the distance between two function $V_1$ and $V_2$ as $V_1,V_2\in \check{\mathcal{V}}_h$ as $dist(V_1,V_2)=\max_{s}|V_1(s)-V_2(s)|$. With respect to this distance function, the $\epsilon$-covering number $\mathcal{N}_\epsilon$ of the function class $\check{\mathcal{V}}_h$ can be upper bounded by
\begin{align}
   \log \mathcal{N}_\epsilon \leq dl \log (1+4L/\epsilon) + d^2l \log (1+8\sqrt{d}\beta^2/\epsilon^2)\notag
\end{align}
\end{lemma}
\begin{proof}[Proof of Lemma \ref{lemma:covering-number1}]
For any two function $\vvalue_1, \vvalue_2 \in \check{\mathcal{V}}_h$, according to the definition of function class $\check{\mathcal{V}}_h$, we have
\begin{align}
  \vvalue_1(\cdot)= \max_{a}\max_{1\leq i\leq l} \max\bigg(0,r_h(\cdot,a)+\wb_{1,i}^\top\bphi(\cdot,a)-\beta\sqrt{\bphi(\cdot,a)^\top \bGamma_{1,i}\bphi(\cdot,a)}\bigg) ,\notag\\
    \vvalue_2(\cdot)=\max_{a}\max_{1\leq i\leq l} \max\bigg(0,r_h(\cdot,a)+\wb_{2,i}^\top\bphi(\cdot,a)-\beta\sqrt{\bphi(\cdot,a)^\top \bGamma_{2,i}\bphi(\cdot,a)}\bigg),\notag 
\end{align}
where $\|\wb_{1,i}\|_2,\|\wb_{2,i}\|_2\leq L$ and $\bGamma_{1,i}, \bGamma_{2,i} \preceq  \Ib$. Since all of the functions $\max_{1\leq i\leq l}$, $\max_a$  and $\max(0,\cdot)$ are contraction functions, we have
\begin{align}
    \text{dist}(\vvalue_1,\vvalue_2)&=\max_{s\in \cS} \big|\vvalue_1(s)-\vvalue_2(s)\big|\notag\\
    &\leq \max_{1\leq i\leq l, s\in \cS, a\in \cA}  
    \Big|  \wb_{1,i}^\top\bphi(s,a)-\beta\sqrt{\bphi(s,a)^\top \bGamma_{1,i}\bphi(s,a)}\notag\\
    &\qquad - \wb_{2,i}^\top\bphi(s,a)+\beta\sqrt{\bphi(s,a)^\top \bGamma_{2,i}\bphi(s,a)}\Big|\notag\\
    &\leq \beta \max_{1\leq i\leq l, s\in \cS, a\in \cA} \Big|\sqrt{\bphi(s,a)^\top \bGamma_{1,i}\bphi(s,a)}- \sqrt{\bphi(s,a)^\top \bGamma_{2,i}\bphi(s,a)} \Big|\notag\\
    &\qquad + \max_{1\leq i\leq l, s\in \cS, a\in \cA} \big |(\wb_{1,i}-\wb_{2,i})^{\top}\bphi(s,a)\big|\notag\\
    &\leq \beta \max_{1\leq i\leq l, s\in \cS, a\in \cA} \Big|\sqrt{\bphi(s,a)^\top (\bGamma_{1,i}-\bGamma_{2,i})\bphi(s,a)}\Big|\notag\\
    &\qquad +\max_{1\leq i\leq l, s\in \cS, a\in \cA} \big |(\wb_{1,i}-\wb_{2,i})^{\top}\bphi(s,a)\big|\notag \notag\\
    &\leq \beta \max_{1\leq i\leq l}\sqrt{\|\bGamma_{1,i}-\bGamma_{2,i}\|_{F}}+ \max_{1\leq i\leq l}\|\wb_{1,i}-\wb_{2,i}\|_2,\label{eq:027}
\end{align}
where the first inequality holds due to the contraction property, the second inequality holds due to the fact that $\max_{x}|f(x)+g(x)|\leq \max_{x}|f(x)|+\max_{x}|g(x)|$, the third inequality holds due to $|\sqrt{x}-\sqrt{y}|\ge |\sqrt{x}-\sqrt{y}|$ and the last inequality holds due to the fact that $\|\bphi(s,a)\|_2\leq 1$.
Now, we denote $\mathcal{C}_{\wb}$ as a $\epsilon/2$-cover of the set $\big\{\wb \in \RR^d\big| \|\wb\|_2\leq L\big\}$ and $\mathcal{C}_{\bGamma}$ as a $\epsilon^2/(4\beta^2)$-cover of the set $\{\bGamma\in \RR^{d\cdot d}\big|\|\bGamma\|_{F}\leq \sqrt{d} \}$ with respect to the Frobenius norms. Thus, according to Lemma \ref{lemma:ball-cover}, we have following property:
\begin{align}
    |\mathcal{C}_{\wb}|\leq \big(1+4L/\epsilon\big)^d, |\mathcal{C}_{\bGamma}|\leq \big(1+8\sqrt{d}\beta^2/\epsilon^2\big)^{d^2}.\label{eq:028}
\end{align}
By the definition of covering number, for any function $V_1\in \check{\mathcal{V}}$ with parameters $\wb_{1,i}, \bGamma_{1,i}(1\leq i\leq l)$, s other parameters $\wb_{2,i}, \bGamma_{2,i}(1\leq i\leq l)$ such that $\wb_{2,i} \in \mathcal{C}_{\wb}, \bGamma_{2,i} \in \mathcal{C}_{\bGamma}$ and $\|\wb_{2,i}-\wb_{1,i}\|_2\leq \epsilon/2, \|\bGamma_{2,i}-\bGamma_{1,i}\|_{F}\leq \epsilon^2/(4\beta^2)$. Thus, we have
\begin{align}
     \text{dist}(\vvalue_1,\vvalue_2)\leq \beta \max_{1\leq i\leq l}\sqrt{\|\bGamma_{1,i}-\bGamma_{2,i}\|_{F}}+ \max_{1\leq i\leq l}\|\wb_{1,i}-\wb_{2,i}\|_2 \leq \epsilon,\notag
\end{align}
where the inequality holds due to \eqref{eq:027}.
Therefore, the $\epsilon$-covering number of function class $\check{\mathcal{V}}_h$ is bounded by $\mathcal{N}_\epsilon\leq |\mathcal{C}_{\wb}|^l\cdot |\mathcal{C}_{\bGamma}|^l$ and it implies
\begin{align}
    \log \mathcal{N}_\epsilon &\leq dl \log (1+4L/\epsilon) + d^2l \log (1+8\sqrt{d}\beta^2/\epsilon^2),\notag
\end{align}
where the first inequality holds due to \eqref{eq:028}. Thus, we finish the proof of Lemma \ref{lemma:covering-number1}.
\end{proof}
In addition, according to the result in Lemma \ref{lemma:covering-number}, the covering number $\mathcal{N}_\epsilon$ of squared function class $\mathcal{V}^2_h$ can be upper bounded by:
\begin{lemma}\label{lemma:covering-number2}
For squared function class $\mathcal{V}^2_h$, we define the distance between two function $V^2_1$ and $V^2_2$ as $V^2_1,V^2_2\in\mathcal{V}^2_h$ as $dist(V^2_1,V^2_2)=\max_{s}|V^2_1(s)-V^2_2(s)|$. With respect to this distance function, the $\epsilon$-covering number $\mathcal{N}_\epsilon$ of the function class $\mathcal{V}^2_h$ can be upper bounded by
\begin{align}
   \log \mathcal{N}_\epsilon \leq dl \log (1+8HL/\epsilon) + d^2l \log (1+32\sqrt{d}H^2\beta^2/\epsilon^2)\notag
\end{align}
\end{lemma}
\begin{proof}[Proof of Lemma \ref{lemma:covering-number2}]
    For any function $V^2_1,V^2_2\in\mathcal{V}_h^2$, the distance can be upper bounded by:
    \begin{align}
    \text{dist}(\vvalue^2_1,\vvalue^2_2)&=\max_{s\in \cS} \big|\vvalue^2_1(s)-\vvalue^2_2(s)\big|\notag\\
    &=\max_{s\in \cS} \big|\vvalue_1(s)-\vvalue_2(s)\big|\cdot  \big|\vvalue_1(s)+\vvalue_2(s)\big|\notag\\
    &\leq 2H\max_{s\in \cS} \big|\vvalue_1(s)-\vvalue_2(s)\big|\notag\\
    &=2H \text{dist}(\vvalue_1,\vvalue_2)
    ,\label{eq:04}
    \end{align}
    where the inequality holds due to the fact that $0\leq \vvalue_1(s),\vvalue_2(s)\leq H$. Thus, any $(\epsilon/2H)$-net for optimistic function class $\mathcal{V}_h$ is also a $(\epsilon/2H)$-net for the squared function class $\mathcal{V}^2$. According to Lemma \ref{lemma:covering-number}, the covering number of the squared function class is upper bounded by:
    \begin{align}
   \log \mathcal{N}_\epsilon \leq dl \log (1+4L/\epsilon) + d^2l \log (1+8\sqrt{d}\beta^2/\epsilon^2)\notag.
\end{align}
Thus, we finish the proof of Lemma \ref{lemma:covering-number2}.
\end{proof}

    \section{Auxiliary Lemmas}
\begin{lemma}\label{lemma:qvalue-linear}
 For any stage $h\in[h]$ in a linear MDP and any bounded-function $\vvalue:\cS \rightarrow [0,B]$, there always exists a vector $\wb \in \RR^d$ such that for all state-action pair $(s,a)\in \cS\times \cA$, we have
\begin{align}
    [\PP_h V](s,a)=\wb^{\top}\bphi(s,a), \text{ where } \|\wb\|_2\leq B\sqrt{d}.\notag
\end{align}
\end{lemma}
\begin{proof}[Proof of Lemma \ref{lemma:qvalue-linear}]
    According to the definition of linear MDP (Assumption \ref{assumption:linear-MDP}), we have
    \begin{align}
        [\PP_h V](s,a)&=\int \PP_h(s'|s,a)V(s') d s'\notag\\
        &= \int  \bphi(s,a)^{\top}V(s') d \btheta_h(s')\notag\\
        &=\bphi(s,a)^{\top}\int V(s') d \btheta_h(s')\notag\\
        &= \bphi(s,a)^{\top} \wb\notag,
    \end{align}
   where we set $\wb= \int V(s') d \btheta_h(s')$. In addition, the norm of $\wb$ is upper bounded by:
\begin{align*}
    \Big\|\int V(s') d \btheta_h(s')\Big\|\leq \max_{s'}V(s')\cdot \sqrt{d}=B\sqrt{d}.
\end{align*}
Thus, we finish the proof of Lemma \ref{lemma:qvalue-linear}.
\end{proof}

\begin{lemma}[Azuma–Hoeffding inequality, \citealt{cesa2006prediction}]\label{lemma:azuma}
Let $\{x_i\}_{i=1}^n$ be a martingale difference sequence with respect to a filtration $\{\cG_{i}\}$ satisfying $|x_i| \leq M$ for some constant $M$, $x_i$ is $\cG_{i+1}$-measurable, $\EE[x_i|\cG_i] = 0$. Then for any $0<\delta<1$, with probability at least $1-\delta$, we have 
\begin{align}
    \sum_{i=1}^n x_i\leq M\sqrt{2n \log (1/\delta)}.\notag
\end{align} 
\end{lemma}
        
    \begin{lemma}[Lemma 11,  \citealt{AbbasiYadkori2011ImprovedAF}]\label{Lemma:abba}
        Let $\{\xb_k\}_{k=1}^{K}$ be a sequence of vectors in $\RR^d$, matrix $\bSigma_0$ a $d \times d$ positive definite matrix and define $\bSigma_k=\bSigma_0+\sum_{i=1}^{k} \xb_i\xb_i^{\top}$, then we have
        \begin{align}
            \sum_{i=1}^{k} \min\Big\{1,\xb_i^{\top} \bSigma_{i-1}^{-1} \xb_i\Big\}\leq 2 \log \bigg(\frac{\det{\bSigma_k}}{\det{\bSigma_0}}\bigg).\notag
        \end{align}
        In addition, if $\|\xb_i\|_2\leq L$ holds for all $i\in [K]$, then 
        \begin{align}
            \sum_{i=1}^{k} \min\Big\{1,\xb_i^{\top} \bSigma_{i-1}^{-1} \xb_i\Big\}\leq 2 \log \bigg(\frac{\det{\bSigma_k}}{\det{\bSigma_0}}\bigg)\leq 2\Big(d\log\big((\text{trace}(\bSigma_0)+kL^2)/d\big)-\log \det \bSigma_0\Big).\notag
        \end{align}
    \end{lemma}
        
    \begin{lemma}[Lemma 12,  \citealt{AbbasiYadkori2011ImprovedAF}]\label{lemma:det}
        Suppose $\Ab, \Bb\in \RR^{d \times d}$ are two positive definite matrices satisfying that $\Ab \succeq \Bb$, then for any $\xb \in \RR^d$, $\|\xb\|_{\Ab} \leq \|\xb\|_{\Bb}\cdot \sqrt{\det(\Ab)/\det(\Bb)}$.
    \end{lemma}

    \begin{lemma}[Confidence Ellipsoid, Theorem 2, \citealt{AbbasiYadkori2011ImprovedAF}] \label{lemma:hoeffding}
        Let $\{\cG_k\}_{k=1}^\infty$ be a filtration, and $\{\xb_k,\eta_k\}_{k\ge 1}$ be a stochastic process such that
        $\xb_k \in \RR^d$ is $\cG_k$-measurable and $\eta_k \in \RR$ is $\cG_{k+1}$-measurable.
        Let $L,\sigma,\bSigma, \epsilon>0$, $\bmu^*\in \RR^d$. 
        For $k\ge 1$, 
        let $y_k = \la \bmu^*, \xb_k\ra + \eta_k$ and
        suppose that $\eta_k, \xb_k$ also satisfy 
        \begin{align}
            \EE[\eta_k|\cG_k] = 0,\  |\eta_k| \leq R,\,\|\xb_k\|_2 \leq L.
        \end{align}
        For $k\ge 1$, let $\Zb_k = \lambda\Ib + \sum_{i=1}^{k} \xb_i\xb_i^\top$, $\bbb_k = \sum_{i=1}^{k}y_i\xb_i$, $\bmu_k = \Zb_k^{-1}\bbb_k$, and
        \begin{small}
        \begin{align}
            \beta_k &= R\sqrt{d \log \left(\frac{1 + kL^2 / \lambda}{\delta}\right)}.\notag
        \end{align}
        \end{small}
        Then, for any $0 <\delta<1$, we have with probability at least $1-\delta$ that, 
        \begin{align}
            \forall k\geq 1,\ \big\|\textstyle{\sum}_{i=1}^{k} \xb_i \eta_i\big\|_{\Zb_k^{-1}} \leq \beta_k,\ \|\bmu_k - \bmu^*\|_{\Zb_k} \leq \beta_k + \sqrt{\lambda}\|\bmu^*\|_2. \notag
        \end{align}
    \end{lemma}

    \begin{lemma}[Lemma 4.4, \citealt{zhou2022computationally}]\label{lemma:keysum:temp}
        Let $\{\sigma_k, \hat\beta_k\}_{k \geq 1}$ be a sequence of non-negative numbers, $\alpha, \gamma>0$, $\{\ab_k\}_{k \geq 1} \subset \RR^d$ and $\|\ab_k\|_2 \leq A$. Let $\{\bar\sigma_k\}_{k \geq 1}$ and $\{\hat\bSigma_k\}_{k \geq 1}$ be (recursively) defined as follows: $\hat\bSigma_1 = \lambda\Ib$,\ 
        \begin{align}
            \forall k \geq 1,\ \bar\sigma_k = \max\{\sigma_k, \alpha, \gamma\|\ab_k\|_{\hat\bSigma_k^{-1}}^{1/2}\},\ \hat\bSigma_{k+1} = \hat\bSigma_k + \ab_k\ab_k^\top/\bar\sigma_k^2.\notag
        \end{align}
        Let $\iota = \log(1+KA^2/(d\lambda\alpha^2))$. Then we have
        \begin{align}
            \sum_{k=1}^K\min\Big\{1, \|\ab_k\|_{\hat\bSigma_k^{-1}}\Big\} &\leq 2d \iota +2\gamma^2d\iota+ 2\sqrt{d \iota}\sqrt{\sum_{k=1}^K(\sigma_k^2 + \alpha^2)}\notag.
        \end{align}
    \end{lemma}
%\begin{lemma}(Lemma D.4, \citep{jin2020provably})\label{lemma:uni-coverge}
%Let $\{x_k\}_{k=1}^{\infty}$ be a real-valued stochastic process on state space $\cS$ with corresponding filtration $ \{\mathcal{F}_k\}_{k=1}^{\infty}$. Let $\{\bphi_k\}_{k=1}^{\infty}$ be an $\RR^d$-valued stochastic process where $\bphi_k \in \mathcal{F}_{k-1}$ and $\|\bphi_k\|_2\leq 1$. For any $k\ge 0$, we define $\bSigma_k= I+\sum_{i=1}^k\bphi_i\bphi_i^{\top}$.
%Then with probability at least $1-\delta$, for all $k \in \NN$ and all function $V\in \mathcal{V}$ with $\max_{s}|V(x)|\leq H$, we have
%\begin{align}
%    \bigg\| \sum_{i=1}^k \bphi_i \Big\{ V(x_i)-\EE\big[V(x_i)|\mathcal{F}_{i-1}\big]\Big\}\bigg\|_{\bSigma_k^{-1}}^2\leq 4H^2\bigg[\frac d2\log(k+1) +\log \frac{\mathcal{N}_\epsilon}{\delta}\bigg]+{8k^2\epsilon^2},\notag
%\end{align}
%where $\mathcal{N}_\epsilon$ is the $\epsilon$-covering number of the function class $\mathcal{V}$ with respect to the distance function $\text{dist}(V_1,V_2)=\max_{s}|V_1(s)-V_2(s)|$. 
%\end{lemma}

\end{document}